\documentclass{article}

\PassOptionsToPackage{numbers}{natbib}
\usepackage[main, final]{arxiv_2026}

\usepackage[utf8]{inputenc}
\usepackage[T1]{fontenc}
\usepackage{hyperref}
\usepackage{url}
\usepackage{booktabs}
\usepackage{amsfonts}
\usepackage{nicefrac}
\usepackage{microtype}
\usepackage{xcolor}
\usepackage{amsmath, amssymb, amsthm, mathtools}
\usepackage{mathrsfs}
\usepackage{bm}
\usepackage{algorithm}
\usepackage{algorithmic}
\usepackage{enumitem}
\usepackage{float}
\usepackage{placeins}
\usepackage{tikz}
\usepackage{wrapfig2}
\usepackage{wrapstuff}
\usepackage{needspace}
\newcommand{\coderepo}{\url{https://github.com/federicoperlino/dagdgp}}
\usepackage{multirow}
\usetikzlibrary{arrows.meta,positioning,calc,fit,backgrounds}
\tikzset{
  dagnode/.style={
    circle,
    draw=black!80,
    fill=black!2,
    line width=0.9pt,
    minimum size=9.5mm,
    inner sep=0pt
  },
  dagfocus/.style={
    circle,
    draw=black,
    fill=black!1,
    line width=1.2pt,
    minimum size=9.8mm,
    inner sep=0pt
  },
  dagdag/.style={
    -{Latex[length=2.2mm,width=1.5mm]},
    draw=black!80,
    line width=0.9pt
  },
  skipedge/.style={
    -{Latex[length=2.2mm,width=1.5mm]},
    draw=black!60,
    dashed,
    line width=0.9pt
  },
  routea/.style={
    -{Latex[length=2.2mm,width=1.5mm]},
    draw=orange!80!black,
    line width=0.95pt
  },
  routeb/.style={
    -{Latex[length=2.2mm,width=1.5mm]},
    draw=teal!70!black,
    line width=0.95pt
  },
  cont/.style={
    -{Latex[length=2.2mm,width=1.5mm]},
    draw=black!55,
    dashed,
    line width=0.85pt
  },
  slice/.style={
    draw=black!18,
    fill=black!3,
    line width=0.8pt,
    ellipse
  }
}
\usepackage{pgfplots}
\pgfplotsset{compat=1.18}
\usepackage{caption}
\newsavebox{\sidebysidebox}
\usepackage{pgf}
\usepackage{graphicx}
\usepgfplotslibrary{groupplots}

\usepackage{booktabs}
\usepackage{tabularx}
\usepackage{colortbl}
\usepackage{xcolor}
\usepackage{microtype}

\newcommand{\R}{\mathbb{R}}

\newcommand{\PA}{\mathrm{Pa}}

\newcommand{\Anc}{\mathrm{Anc}}
\newcommand{\N}{\mathbb{N}}
\newcommand{\cX}{\mathcal{X}}
\newcommand{\cS}{\mathcal S}
\newcommand{\cA}{\mathcal A}

\newcommand{\cV}{\mathcal V}
\newcommand{\cE}{\mathcal E}
\newcommand{\cR}{\mathcal R}
\newcommand{\cO}{\mathcal{O}}

\newcommand{\cD}{\mathcal D}
\newcommand{\cU}{\mathcal U}
\newcommand{\cH}{\mathcal H}
\newcommand{\cN}{\mathcal N}

\newcommand{\cG}{\mathcal G}

\DeclareMathOperator{\Var}{Var}
\DeclareMathOperator{\Cov}{Cov}

\DeclareMathOperator{\tr}{tr}
\DeclareMathOperator{\GP}{GP}
\newcommand{\bF}{\mathbf{F}}

\newcommand{\bX}{\mathbf{X}}
\newcommand{\bx}{\mathbf{x}}

\newcommand{\bY}{\mathbf{Y}}

\newcommand{\bh}{\mathbf{h}}

\newcommand{\bZ}{\mathbf{Z}}

\newcommand{\PP}{\mathbb P}
\newcommand{\EE}{\mathbb E}
\newcommand{\cF}{\mathcal F}
\newcommand{\bU}{\mathbf{U}}
\providecommand{\CH}{\operatorname{Ch}}
\providecommand{\Desc}{\operatorname{Desc}}

\usepackage{array}

\usepackage{amsthm}
\newtheorem{proposition}{Proposition}
\newtheorem{theorem}{Theorem}
\newtheorem{lemma}{Lemma}
\newtheorem{assumption}{Assumption}
\newtheorem{corollary}{Corollary}
\newtheorem{remark}{Remark}
\theoremstyle{definition}
\newtheorem{definition}[theorem]
{Definition}
\usepackage{bbm}
\usepackage{titletoc}

\usepackage[nolist,smaller]{acronym}
\usepackage{xspace}

\definecolor{black}{HTML}{000000}
\definecolor{orange}{HTML}{E69F00}
\definecolor{blue}{HTML}{56B4E9}
\definecolor{green}{HTML}{009E73}
\definecolor{yellow}{HTML}{F0E442}
\definecolor{dark_blue}{HTML}{0072B2}
\definecolor{dark_orange}{HTML}{D55E00}
\definecolor{pink}{HTML}{CC79A7}
\definecolor{white}{HTML}{FFFFFF}
\definecolor{grey}{HTML}{808080}

\title{Deep Gaussian Processes on Directed Acyclic Graphs}

\author{
Federico L. Perlino\textsuperscript{1}
\quad Oliver Hamelijnck\textsuperscript{2,3}
\quad Adam M. Johansen\textsuperscript{1}
\quad Theodoros Damoulas\textsuperscript{1,2,3}
\\[2mm]
\textsuperscript{1}Department of Statistics, University of Warwick\\
\textsuperscript{2}Unilink AI Research Center (UAIRC)\\
\textsuperscript{3}Department of Computer Science, University of Warwick
}

\begin{document}

\maketitle

\begin{abstract}
Many real-world processes can be represented as compositions of functions along a directed acyclic graph (DAG). In causal modelling, these correspond to the underlying mechanisms; in engineering, to multiple fidelity levels; and in gene-regulatory networks, to transcription factors. These functions are partially observed across the DAG, with noisy and heterogeneously sampled measurements, posing significant challenges for reconstruction, uncertainty propagation, and inference. To tackle these challenges, we place priors over functions and naturally arrive at Deep Gaussian Processes over DAGs. We theoretically study their prior-collapse behaviour and the effect of graph topology and intermediate observations on the preservation of information. We obtain almost-sure lower bounds on the asymptotic frequency of depths at which the distinction between inputs is preserved, identify broad kernel classes for which these hold, and prove an observation by \citep{dunlop2018} on the role of input connections. We offer a structured variational approximation that retains graph dependencies, preserves compositional uncertainty, and captures the explaining-away behaviour of colliders. Finally, we empirically validate our theory and methodology at scale and under DAG misspecification. We then model a latent-collider DAG, a protein signalling network, and a multi-fidelity \mbox{heavy-ion} collision emulation task, attaining state-of-the-art performance while recovering low-fidelity contributions and yielding interpretability over the simulator hierarchy.
\end{abstract}

\section{Introduction}\label{sec:intro}
A wide range of phenomena across the sciences, and beyond, can be represented as
compositions of interdependent latent functions along a Directed Acyclic Graph
(DAG). In probabilistic modelling, the DAG encodes conditional dependencies
between quantities of interest, although the interpretation of these
dependencies varies by setting. In causal models, the DAG represents mechanistic
relations \citep{spirtes2001causation,pearl2009}; in multi-fidelity modelling across
physics and engineering \citep{fernandezgodino2023}, it encodes dependencies
between information sources of different fidelity
\citep{kennedy2000,perdikaris2017,ji2024}; and in systems biology, it is
used to describe regulatory links between transcription factors in
gene-regulatory networks \citep{friedman2004}. The graph itself may be
elicited from expert knowledge \citep{delSagrado2003,ohagan2006}, derived
from mechanistic constraints and the natural directionality of the phenomenon
\citep{boudali2005}, or inferred from data through structure discovery
\citep{heckerman1995}. More broadly, DAGs are often formulated by scientists
as representations of their hypotheses \citep{greenland1999}.

In such DAG settings, we rarely have complete, noise-free observations at every node. Data may be available only at a subset of nodes, at varying sample sizes and resolutions, and are often affected by missingness, measurement error, or model discrepancy \citep{little2019,caroll2006,kennedy2001}. Inference then becomes a coupled inverse problem, in which DAG-dependent latent functions at different nodes must be jointly recovered from indirect, heterogeneous evidence. Furthermore, when an observed downstream quantity can be explained by several upstream functions, evidence for one explanation changes the posterior plausibility of the others, which is the classical \textit{explaining-away} effect \citep{pearl1988,lauritzen1996}. At the same time, latent quantities that are never directly observed are typically weakly identified or unidentified \citep{raue2009,allman2009}, so collapsing onto any single explanation misrepresents the uncertainty in the system \citep{ustyuzhaninov2020}.
\begin{wrapstuff}[r, width=0.292\textwidth, top=2, type=figure]
  \captionsetup{type=figure}\centering
  \includegraphics[width=\linewidth]{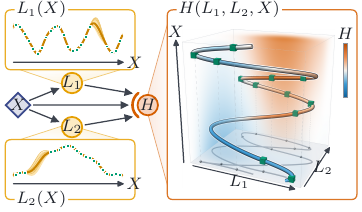}
  \caption{\textbf{DAG-DGP for multi-fidelity.} Posteriors of $L_1, L_2$ (left) and volume rendering of the posterior mean of $H$ (right), with data in green. The tube traces $(\mathbb{E}[L_1(X)],\mathbb{E}[L_2(X)],X)$.}
  \label{fig:heavy_ion_dag}
\end{wrapstuff}
A natural modelling response to these challenges is to place Gaussian Process (GP)
priors \citep[e.g.,][]{rasmussen2006gpml} over the latent functions composing
the DAG, endowing each node with a principled probabilistic representation. Deep Gaussian Processes (DGPs) are hierarchical compositions of GP mappings
\citep{lawrence2007,damianou2013}. In the standard chain setting, such compositions have been shown to improve contraction rates for compositional targets \citep{finocchio2023,giordano2022}, yet  deep GP priors may also collapse \citep{neal1995,duvenaud2014,dunlop2018,tong2021}, failing to preserve information with depth and motivating input (or skip) connections as a remedy \citep{neal1995,duvenaud2014}.
DGPs therefore provide a natural, but delicate, language for compositional
probabilistic modelling.
We take this viewpoint to define
DGPs on DAGs (DAG-DGPs), where the inductive bias of the system is directly
reflected in the architecture, and, in Sec.~\ref{sec:theory}, study their capability to preserve information along depth.
Compared with a standard chain DGP, a DAG-DGP
exposes two modelling choices that are central in scientific applications.
First, a node may have several parents, so the kernel at that node must
specify how parent contributions are fused. Second,
observations may be available at arbitrary internal nodes, so the model must
propagate uncertainty through the graph while using intermediate measurements
to anchor internal representations.  Fig.~\ref{fig:heavy_ion_dag} shows a DAG-DGP with
heterogeneous node observations; the posterior of
\(H\) is evaluated on the random curve of its parents \((L_1(X),L_2(X),X)\).

This viewpoint is related to specialised multi-fidelity and information-fusion models, where lower-fidelity or intermediate simulator outputs are propagated as uncertain inputs to improve high-fidelity emulation \citep{perdikaris2017,cutajar2019}, and to recent graphical multi-fidelity emulators that organise such dependencies over directed trees \citep{ji2024}, such as in the heavy-ion collision task (Sec.~\ref{sec:heavy-ion}). These are designed for specific DAG topologies with nested data \citep{LeGratietGarnier2014} and focus on high-fidelity prediction, leaving joint latent-function inference and uncertainty composition across a general DAG largely open.
DAGs carry a rich structure, which we exploit by modelling them directly rather than using layerisation, a non-trivial graph-drawing problem utilising dummy nodes to preserve dependencies \citep{sugiyama1981,harrigan2006}.

Posterior inference over the DAG composition of functions is challenging as it requires marginalisation over intermediate latent functions.
Even in simpler chain DGP settings this difficulty has motivated a large literature on approximate inference, including variational, expectation-propagation, sampling, and doubly stochastic methods \citep{damianou2013,hensman2014,dai2016,bui2016,salimbeni2017,havasi2018,salimbeni2019}. In a DAG the challenge goes beyond layer-wise uncertainty propagation: observed colliders and descendants induce posterior dependence between a priori independent branches, precisely what is needed to represent \textit{explaining-away}. Structured approximations for DGPs have shown that richer posterior dependence is important for calibrated uncertainty and compositional ambiguity \citep{ustyuzhaninov2020,lindinger2020,ober2021}, yet existing constructions target only chain architectures or cross-layer dependence. This motivates tractable, structured variational inference that preserves posterior dependencies across the DAG. We make the following contributions:

\begin{itemize}
\item \textbf{DAG-DGP framework.}
We formulate the first unified DGP framework on DAGs, with noisy observations at arbitrary nodes,
recovering chain DGPs, multi-fidelity DGPs, and graphical multi-fidelity
emulators as special cases.

\item \textbf{Structured variational inference for DAG-DGPs.}
We offer a structured variational approximation for DAG-DGPs retaining compositional uncertainty \citep{ustyuzhaninov2020} and explaining-away, recovering \citep[Sec.~4.1]{ustyuzhaninov2020}
and our own extension of \citep{salimbeni2017} to DAGs as special cases.

\item \textbf{Theoretical results on the DAG-DGP model class.}
We offer lower bounds on the frequency of DAG depths for avoiding prior-collapse and characterise how graph structure, kernel families, and intermediate
observations affect information preservation. We provide explicit lower bounds for bounded-curvature exponential family-observation models.

\item \textbf{Theoretical results on standard DGPs (i.e. chains).}
We prove non-collapse for the
input-connected chain DGP setting considered by \citep{dunlop2018} and offer the first theoretical account of standard DGPs with intermediate observations.

\item \textbf{Empirical validation and scientific applications.} We validate our theory, show explaining-away and behaviour under DAG misspecification in colliders, and assess accuracy, recovery, and scalability on protein signalling \citep{sachs2005} and multi-fidelity heavy-ion collision \citep{ji2024}.

\end{itemize}

\section{Deep Gaussian Processes: from Chains to Directed Acyclic Graphs}\label{sec:dag-development}
We begin by reviewing a chain DGP that organises latent variables along a total order of
\(L\) layers, i.e.,
\begin{equation}\label{eq:dgp-recursion}
    F_0(x) = x,\qquad
  F_\ell(x) = f_\ell(F_{\ell-1}(x)),
  \qquad
  f_\ell\sim\GP(0,K_\ell),
  \qquad \ell=1,\ldots,L,
\end{equation}
where $x \in \cX$ is an input case  (e.g., Fig.~\ref{fig:dag-precision-overview}a). Each latent layer receives a single latent input. The kernel \(K_\ell\) is
therefore defined on the state space of layer \(\ell-1\), and no node has multiple parents.

The data enter separately from this latent recursion. In the standard
supervised formulation, the dataset is \(\cD=(\bX,\bY)\), and observations are linked to the
final layer through an observational distribution \(p(\bY\mid\bF_L)\).
Intermediate observations with their own distributions have been included in specialised models, notably
multi-fidelity DGPs \citep{cutajar2019}, but their placement is then tied to the fidelity
ordering. Despite being flexible for modelling hierarchical structures, the construction
in Eq.~\eqref{eq:dgp-recursion} can suffer from prior collapse
\citep{neal1995,duvenaud2014}, a problem we address in Section~\ref{sec:theory}.  Scalable DGP inference typically augments each layer with inducing variables
and optimises a variational ELBO using Monte Carlo propagation through the
latent GP conditionals \citep{hensman2014,salimbeni2017}. Using Markov chain Monte Carlo can provide fully Bayesian inference, but typically at
substantially higher computational cost \citep{havasi2018,sauer2023}.
\begin{figure}[t]
\centering
\resizebox{\linewidth}{!}{\begin{tikzpicture}[font=\small]

\definecolor{cb_orange}{HTML}{E69F00}
\definecolor{cb_blue}{HTML}{56B4E9}
\definecolor{cb_green}{HTML}{009E73}
\definecolor{cb_yellow}{HTML}{F0E442}
\definecolor{cb_darkblue}{HTML}{0072B2}
\definecolor{cb_darkorange}{HTML}{D55E00}
\definecolor{cb_pink}{HTML}{CC79A7}
\definecolor{cb_grey}{HTML}{808080}
\definecolor{cb_black}{HTML}{000000}

\tikzset{
    >=Latex,
    border/.style    ={draw=black!65, line width=0.65pt},
    nzblock/.style   ={fill=black!10, draw=black!35, line width=0.22pt},
    clique/.style    ={fill opacity=1.0, draw=none},
    sepborder/.style ={draw=cb_black, line width=0.7pt},
    clqa/.style={clique, fill=cb_darkblue},
    clqb/.style={clique, fill=cb_orange},
    clqc/.style={clique, fill=cb_green},
    clqd/.style={clique, fill=cb_pink},
    clqe/.style={clique, fill=cb_darkorange},
    clqf/.style={clique, fill=cb_blue},
    clqg/.style={clique, fill=cb_yellow},
    clqh/.style={clique, fill=cb_grey},
    clqi/.style={clique, fill=cb_black},
    dagnode/.style={
        circle,
        draw=black!70,
        fill=black!3,
        minimum size=3.6mm,
        inner sep=0pt,
        line width=0.55pt
    },
    obsnode/.style={
        rectangle,
        draw=black!70,
        fill=black!3,
        minimum size=3.4mm,
        inner sep=0pt,
        line width=0.55pt
    },
    dagedge/.style={
        -{Latex[length=1.6mm,width=1.2mm]},
        draw=black!70,
        line width=0.6pt,
        shorten >=0.5pt,
        shorten <=0.5pt
    },
    moraledge/.style={
        dashed,
        draw=black!55,
        line width=0.6pt
    },
    chordedge/.style={
        densely dotted,
        draw=black,
        line width=0.9pt
    }
}

\newcommand{\cliquecells}[2]{\foreach \ci in {#2}{\foreach \cj in {#2}{\path[#1] (\cj-1,-\ci+1) rectangle ++(1,-1);}}}
\newcommand{\cliqueresidual}[3]{\foreach \ci in {#2}{\foreach \cj in {#3}{\path[#1] (\cj-1,-\ci+1) rectangle ++(1,-1);
    \path[#1] (\ci-1,-\cj+1) rectangle ++(1,-1);}}}
\newcommand{\sepcells}[1]{\foreach \i/\j in {#1}{\draw[sepborder] (\j-1,-\i+1) rectangle ++(1,-1);}}

\begin{scope}[xshift=0cm, yshift=0.8cm]
\node[dagnode] (a1) at (0.00,1.0) {};
\node[dagnode] (a2) at (0.80,1.0) {};
\node[dagnode] (a3) at (1.60,1.0) {};
\node[obsnode] (a4) at (2.40,1.0) {};
\draw[dagedge] (a1) -- (a2);
\draw[dagedge] (a2) -- (a3);
\draw[dagedge] (a3) -- (a4);
\end{scope}

\begin{scope}[xshift=3.9cm, yshift=0.8cm]
\node[dagnode] (b1) at (0.45,2.0) {};
\node[dagnode] (b2) at (1.95,2.0) {};
\node[dagnode] (b3) at (0.45,1.0) {};
\node[dagnode] (b4) at (1.95,1.0) {};
\node[obsnode] (b5) at (0.45,0.0) {};
\node[obsnode] (b6) at (1.95,0.0) {};
\draw[dagedge] (b1) -- (b3);
\draw[dagedge] (b3) -- (b5);
\draw[dagedge] (b2) -- (b4);
\draw[dagedge] (b4) -- (b6);
\draw[dagedge] (b1) -- (b6);
\draw[dagedge] (b2) -- (b5);
\draw[moraledge] (b1) -- (b4);
\draw[moraledge] (b2) -- (b3);
\draw[chordedge] (b3) -- (b4);
\end{scope}

\begin{scope}[xshift=7.8cm, yshift=0.8cm]
\node[dagnode] (c1) at (0.45,2.0) {};
\node[dagnode] (c2) at (1.95,2.0) {};
\node[dagnode] (c3) at (1.20,1.0) {};
\node[obsnode] (c4) at (0.40,0.0) {};
\node[obsnode] (c5) at (1.20,0.0) {};
\node[obsnode] (c6) at (2.00,0.0) {};
\draw[dagedge] (c1) -- (c3);
\draw[dagedge] (c2) -- (c3);
\draw[dagedge] (c3) -- (c4);
\draw[dagedge] (c3) -- (c5);
\draw[dagedge] (c3) -- (c6);
\draw[moraledge] (c1) -- (c2);
\end{scope}

\begin{scope}[xshift=12.35cm, yshift=0.8cm]
\node[dagnode] (d1)  at (0.0,2.0) {};
\node[dagnode] (d2)  at (0.8,2.0) {};
\node[dagnode] (d3)  at (1.6,2.0) {};
\node[dagnode] (d4)  at (2.4,2.0) {};
\node[dagnode] (d5)  at (0.4,1.0) {};
\node[dagnode] (d6)  at (1.2,1.0) {};
\node[dagnode] (d7)  at (2.0,1.0) {};
\node[obsnode] (d8)  at (0.0,0.0) {};
\node[obsnode] (d9)  at (0.8,0.0) {};
\node[obsnode] (d10) at (1.6,0.0) {};
\node[obsnode] (d11) at (2.4,0.0) {};
\draw[dagedge] (d1) -- (d5);
\draw[dagedge] (d2) -- (d5);
\draw[dagedge] (d2) -- (d6);
\draw[dagedge] (d3) -- (d6);
\draw[dagedge] (d3) -- (d7);
\draw[dagedge] (d4) -- (d7);
\draw[dagedge] (d5) -- (d8);
\draw[dagedge] (d5) -- (d9);
\draw[dagedge] (d6) -- (d9);
\draw[dagedge] (d6) -- (d10);
\draw[dagedge] (d7) -- (d10);
\draw[dagedge] (d7) -- (d11);
\draw[moraledge] (d1) -- (d2);
\draw[moraledge] (d2) -- (d3);
\draw[moraledge] (d3) -- (d4);
\draw[moraledge] (d5) -- (d6);
\draw[moraledge] (d6) -- (d7);
\end{scope}

\node[font=\scriptsize\bfseries, inner sep=0pt] at (1.20, 0.50) {a};
\node[font=\scriptsize\bfseries, inner sep=0pt] at (5.10, 0.50) {b};
\node[font=\scriptsize\bfseries, inner sep=0pt] at (9.00, 0.50) {c};
\node[font=\scriptsize\bfseries, inner sep=0pt] at (13.55,0.50) {d};

\begin{scope}[xshift=0cm, x=0.6cm, y=0.6cm]
\node[anchor=south, font=\scriptsize] at (0.5,0.14) {$U_1$};
\node[anchor=south, font=\scriptsize] at (3.5,0.14) {$U_4$};
\node[anchor=east, font=\scriptsize] at (-0.10,-0.5) {$U_1$};
\node[anchor=east, font=\scriptsize] at (-0.10,-3.5) {$U_4$};

\foreach \i/\j in {
    1/1,1/2, 2/1,2/2,2/3, 3/2,3/3,3/4, 4/3,4/4
}{\fill[nzblock] (\j-1,-\i+1) rectangle ++(1,-1);}

\cliquecells{clqa}{1,2}
\cliqueresidual{clqb}{2,3}{3}
\cliqueresidual{clqc}{3,4}{4}

\draw[border] (0,0) rectangle (4,-4);
\sepcells{2/2, 3/3}
\end{scope}

\begin{scope}[xshift=3.9cm, x=0.40cm, y=0.40cm]
\node[anchor=south, font=\scriptsize] at (0.5,0.14) {$U_1$};
\node[anchor=south, font=\scriptsize] at (3.0,0.14) {$\cdots$};
\node[anchor=south, font=\scriptsize] at (5.5,0.14) {$U_6$};
\node[anchor=east, font=\scriptsize] at (-0.10,-0.5) {$U_1$};
\node[anchor=east, font=\scriptsize] at (-0.10,-3.0) {$\vdots$};
\node[anchor=east, font=\scriptsize] at (-0.10,-5.5) {$U_6$};

\foreach \i/\j in {
    1/1,1/3,1/4,1/6,
    2/2,2/3,2/4,2/5,
    3/1,3/2,3/3,3/4,3/5,
    4/1,4/2,4/3,4/4,4/6,
    5/2,5/3,5/5,
    6/1,6/4,6/6
}{\fill[nzblock] (\j-1,-\i+1) rectangle ++(1,-1);}

\cliquecells{clqa}{1,3,4}
\cliqueresidual{clqb}{1,4,6}{6}
\cliqueresidual{clqc}{2,3,4}{2}
\cliqueresidual{clqd}{2,3,5}{5}

\draw[border] (0,0) rectangle (6,-6);
\sepcells{
    1/1, 1/4, 4/1, 4/4,
    3/3, 3/4, 4/3,
    2/2, 2/3, 3/2
}
\end{scope}

\begin{scope}[xshift=7.8cm, x=0.40cm, y=0.40cm]
\node[anchor=south, font=\scriptsize] at (0.5,0.14) {$U_1$};
\node[anchor=south, font=\scriptsize] at (3.0,0.14) {$\cdots$};
\node[anchor=south, font=\scriptsize] at (5.5,0.14) {$U_6$};
\node[anchor=east, font=\scriptsize] at (-0.10,-0.5) {$U_1$};
\node[anchor=east, font=\scriptsize] at (-0.10,-3.0) {$\vdots$};
\node[anchor=east, font=\scriptsize] at (-0.10,-5.5) {$U_6$};

\foreach \i/\j in {
    1/1,1/2,1/3,
    2/1,2/2,2/3,
    3/1,3/2,3/3,3/4,3/5,3/6,
    4/3,4/4,
    5/3,5/5,
    6/3,6/6
}{\fill[nzblock] (\j-1,-\i+1) rectangle ++(1,-1);}

\cliquecells{clqa}{1,2,3}
\cliqueresidual{clqb}{3,4}{4}
\cliqueresidual{clqc}{3,5}{5}
\cliqueresidual{clqd}{3,6}{6}

\draw[border] (0,0) rectangle (6,-6);
\sepcells{3/3}
\end{scope}

\begin{scope}[xshift=12.35cm, x=0.22cm, y=0.22cm]
\node[anchor=south, font=\scriptsize] at (0.5,0.14) {$U_1$};
\node[anchor=south, font=\scriptsize] at (5.5,0.14) {$\cdots$};
\node[anchor=south, font=\scriptsize] at (10.5,0.14) {$U_{11}$};
\node[anchor=east, font=\scriptsize] at (-0.12,-0.5) {$U_1$};
\node[anchor=east, font=\scriptsize] at (-0.12,-5.5) {$\vdots$};
\node[anchor=east, font=\scriptsize] at (-0.12,-10.5) {$U_{11}$};

\foreach \i/\j in {
    1/1,1/2,1/5,
    2/1,2/2,2/3,2/5,2/6,
    3/2,3/3,3/4,3/6,3/7,
    4/3,4/4,4/7,
    5/1,5/2,5/5,5/6,5/8,5/9,
    6/2,6/3,6/5,6/6,6/7,6/9,6/10,
    7/3,7/4,7/6,7/7,7/10,7/11,
    8/5,8/8,
    9/5,9/6,9/9,
    10/6,10/7,10/10,
    11/7,11/11
}{\fill[nzblock] (\j-1,-\i+1) rectangle ++(1,-1);}

\cliquecells{clqa}{1,2,5}
\cliqueresidual{clqb}{2,5,6}{6}
\cliqueresidual{clqc}{2,3,6}{3}
\cliqueresidual{clqd}{3,6,7}{7}
\cliqueresidual{clqe}{3,4,7}{4}
\cliqueresidual{clqf}{5,6,9}{9}
\cliqueresidual{clqg}{6,7,10}{10}
\cliqueresidual{clqh}{5,8}{8}
\cliqueresidual{clqi}{7,11}{11}

\draw[border] (0,0) rectangle (11,-11);
\sepcells{
    2/2, 2/5, 5/2, 5/5,
    2/6, 6/2, 6/6,
    3/3, 3/6, 6/3,
    3/7, 7/3, 7/7,
    5/6, 6/5,
    6/7, 7/6
}
\end{scope}

\end{tikzpicture}}
\caption{
DAGs discussed in Secs.~\ref{sec:dag-development},~\ref{sec:dag-dsvi},
and~\ref{sec:theory} (top row) and the corresponding block sparsity
patterns of the structured precision matrix \(\bm{\Lambda}\) (bottom row):
chain DGP, disjoint routes (Thm.~\ref{thm:stochastic-skip-main}), V-structure
with branching, and a three-layered DAG. Circles denote latent nodes and
squares nodes with partial observations. With observations placed at the
terminal nodes, each DAG is the ancestral graph of its observed nodes, so
every inducing block contributes to \(\bm{\Lambda}\). Dashed grey edges
denote moralisation (co-parents joined within each V-structure) and dotted
black edges denote chordal fill-ins added to obtain the chordal completion
\(\cH\). Colours identify maximal cliques of \(\cH\), with black outlines
marking separator blocks.
}
\label{fig:dag-precision-overview}
\end{figure}

Motivated by the need to model compositional functions whose dependencies are
specified by a given DAG, we develop a DAG-DGP architecture that propagates
information along the graph-induced partial order, accommodates
multi-parent dependencies through nodewise fusion rules, and incorporates
heterogeneous observations within a single compositional model.
\label{sec:dagdgp}
\paragraph{Setup.}
Let \(\mathcal G=(\mathcal V,\mathcal E)\) be a DAG (with vertices $\mathcal{V}$ and edges $\mathcal{E}$) with roots \(\mathcal R\), non-roots \(\mathcal U=\mathcal V\setminus\mathcal R\). Denote with \(\PA(w)\) the parent set of \(w\in\mathcal U\).
For each node \(w\), let \(\bF_w\in\mathbb R^{n\times d_w}\) collect the latent values of the \(n\) observations, with row \(\smash{F_w^{(i)}\in\mathbb R^{d_w}}\).
Roots \(r\in\cR\) are supplied with design matrices \(\bX_r\in\mathbb R^{n\times d_r}\), encoding deterministic inputs such as covariates, spatial locations, or time indices, and we set \(\bF_r=\bX_r\).
For notational simplicity, observations are indexed over the same set as the corresponding latent values, so a non-root node \(w\) may have response \(\bY_w\in\mathbb R^{n\times d_w}\) along with an observation mask \(\cO_w\subseteq[n]\times[d_w]\).
Only entries in \(\cO_w\) enter the observational distribution, while the remaining entries are treated as missing.
We let \(\cD=\smash{\bigl(\{\bX_r\}_{r\in\cR},\,\{(\bY_w,\cO_w)\}_{w\in\cU}\bigr)}\) denote the full dataset.

\paragraph{DAG-DGP prior.}
Each non-root latent node takes as input the latent values of its parents and therefore operates on the product of the parent latent spaces.
Concretely, for each \(w\in\cU\) we introduce a function \(f_w\colon\prod_{p\in\PA(w)}\R^{d_p}\to\R^{d_w}\) with prior \(f_w\sim\GP(0,K_w)\), where \(K_w\) is defined on the corresponding product space.
The DAG-DGP prior is thus given by the recursion
\begin{equation}
  F_w^{(i)} = f_w\!\bigl(
    \{F_p^{(i)}\}_{p\in\PA(w)}
  \bigr),
  \qquad w\in\cU,\; i\in[n].
  \label{eq:dagdgp-recursion}
\end{equation}
An additive Gaussian innovation at each node can equivalently be absorbed into \(K_w\)~\citep{salimbeni2017} and is omitted throughout.
Since the nodewise GP modules are mutually independent a priori, the joint
prior over the latent nodes factorises over the DAG into a product of Gaussian
conditionals, one per non-root node given its parents; these factors are
precisely those appearing in the posterior of Eq.~\eqref{eq:dagdgp-posterior}.

\paragraph{Fusion kernels.}
Some nodes of the DAG can have a high indegree, i.e.\ multiple parents (see, e.g., Fig. \ref{fig:dag-precision-overview}), so the way their
contributions are combined is a modelling choice that shapes the inductive
bias of local latent evaluations. We encode this choice through a node-wise fusion rule that directly captures parent contributions and goes beyond the typical concatenations employed in Gaussian process networks (GPNs) \cite{friedman2000, giudice2023, kiroriwal2025bayesian}. The kernel \(K_w\) at node \(w\) is \(\R^{d_w\times d_w}\)-valued positive-semidefinite
on \(\smash{\prod_{p\in\PA(w)}\R^{d_p}}\). Assigning to each
parent \(p\in\PA(w)\) a kernel
\(\smash{K_w^{(p)}\colon\R^{d_p}\times\R^{d_p}\to\R^{d_w\times d_w}}\)
capturing its isolated contribution, and setting
\begin{equation}
  K_w
  \;:=\;
  \Phi_w\!\Bigl(\bigl\{
    K_w^{(p)}
  \bigr\}_{p\in\PA(w)}\Bigr),
  \label{eq:fusion-kernel}
\end{equation}
where the fusion rule \(\Phi_w\) returns a valid
kernel on the full parent space is a convenient way to achieve this.
The fusion rule determines whether parent effects
enter independently, interact, or gate one another, and can vary
from node to node to reflect heterogeneous domain knowledge across
the graph. Natural instances include additive fusion, which sums the
per-parent contributions and treats them as independent effects \cite{duvenaud2011};
product fusion, which multiplies per-parent contributions and thereby allows the effect of each parent to be modulated by the others; and ANOVA-type fusion, which supplements the additive main effects with explicit pairwise interaction terms \citep{AlvarezRosascoLawrence2012KernelsVectorValuedReview}. No fusion rule is universally preferable. Additive fusion suits independent parent effects, whereas product fusion suits effects where one parent modulates another (App.~\ref{app:fusion-kernels}), so the choice should reflect prior knowledge and inferential objectives. More specialised mechanisms encode domain-specific structure, as in multifidelity models \citep{perdikaris2017, cutajar2019,ji2024}. In these models, the fusion rule transfers lower-fidelity information through linear and nonlinear terms modulated by the input, plus a discrepancy (App.~\ref{app:heavyion}). This lets the high fidelity borrow strength, while retaining direct input dependence.

\paragraph{Intermediate observations.}
Heterogeneous observations are incorporated locally at the nodes where they are available, as in Fig.~\ref{fig:heavy_ion_dag}, where lower fidelities are observed more than the high fidelity. For each non-root node \(w\in\cU\) with observations, we
specify a nodewise conditional distribution
\(p_w(\bY_w\mid\bF_w,\cO_w)\) for the observed entries given the latent evaluations matrix
\(\bF_w\).
For nodes without observations, the corresponding factor is set to one as they carry no information.
These nodewise factors anchor internal latent nodes wherever data exist. Thus, the
posterior is
\begin{equation}
  p\,\!\bigl(\{\bF_w\}_{w\in\cU}
    \mid \cD\bigr)
  \;\propto\; \smash{
  \prod_{w\in\cU}
  \left[
    p_0\!\bigl(\bF_w\mid\{\bF_p\}_{p\in\PA(w)}\bigr)\,
    p_w(\bY_w\mid\bF_w,\cO_w)
  \right].}
  \label{eq:dagdgp-posterior}
\end{equation}

\section{Structured Variational Inference for DAG-DGPs}\label{sec:dag-dsvi}
Marginalising over the latent hierarchies makes the posterior in
Eq.~\eqref{eq:dagdgp-posterior} computationally challenging, as in standard
DGPs but over a more complex object. We introduce two doubly stochastic
variational families for DAG-DGPs. The first, DAG-VI, is a mean-field inducing
posterior: a scalable DAG adaptation of \citep{salimbeni2017}. The
second, DAG-SVI, retains more posterior dependencies and is more expressive. The first family is a special case of the second. For each non-root node \(w\), consider $M_w$ inducing
locations \(\bZ_w\) in the input space of each latent, and the corresponding
inducing values \(\bU_w=f_w(\bZ_w)\); write \(\bF\) and \(\bU\) for the
collections of latent and inducing values. Following
\citep{salimbeni2017}, we define:
\begin{equation}
q(\bF,\bU)
:=
q(\bU)
\prod_{w\in\cU}
p_0\!\left(
\bF_w
\mid
\{\bF_p\}_{p\in\PA(w)},\bU_w
\right).
\label{eq:dag-dsvi-family}
\end{equation}
To complete the specification, we choose \(q(\bU)\) so as to retain the
posterior couplings most directly informed by the evidence. Since only nodes
\(w\) with \(\cO_w\neq\emptyset\) contribute observational terms, we restrict
attention to their ancestral graph. The exact latent-state posterior is Markov
with respect to its moralized graph \citep[e.g.,][]{lauritzen1996}, in which
co-parents of a common child become adjacent. For example, in the collider in
Fig.~\ref{fig:dag-precision-overview}c, conditioning on an observed
descendant of the child induces posterior dependence between the co-parents, a
phenomenon known as \emph{explaining away}. Inspired by this moralized graph, which captures the
posterior couplings closest to the observational distributions but need
not be decomposable, we take as support a chordal completion \(\cH\), which
admits a clique-separator Gaussian representation and sparse Cholesky
elimination \citep{lauritzen1996,rue2005}.

Let \(C_1,\dots,C_k\) denote the maximal cliques of \(\cH\), with separators
\(S_i := C_i\cap(C_1\cup\dots\cup C_{i-1})\). For a perfect clique ordering, decomposability of \(\cH\)
yields the clique factorization \citep[Eq.~(2)]{green2013}
\begin{equation}
q_\cH(\bU)
\;=\;
\prod_{i=1}^{k} q\bigl(\bU_{C_i\setminus S_i}\mid\bU_{S_i}\bigr),
\label{eq:clique-factorisation}
\end{equation}
in which each factor is a conditional Gaussian on the clique residual given its separator. Equivalently, \(q_\cH\) is jointly Gaussian \(\mathcal N(\bm m,\bm\Sigma)\) whose precision \(\bm\Lambda=\bm\Sigma^{-1}\) has non-zero blocks only between nodes adjacent in \(\cH\); the maximal cliques in Fig.~\ref{fig:dag-precision-overview} are the dense blocks of \(\bm\Lambda\). Inducing blocks lying in a common clique are freely correlated, so explaining-away between co-parents and the couplings induced along ancestral paths to observed nodes are preserved; blocks that share no clique are conditionally independent, dropping in particular dependencies between disjoint branches of the DAG with no common observed descendant. In the standard chain, \(\cH\) is a path, so DAG-SVI
recovers the block-tridiagonal precision family of
\citep[Sec.~4.1]{ustyuzhaninov2020}.

The resulting ELBO for our model is, where \(q_\cH(\bF_w)\) is the marginal induced by \(q_\cH(\bU)\),
\begin{equation}
\mathcal L_\cH
=
\sum_{w:\,\cO_w\neq\emptyset}
\mathbb E_{q_\cH(\bF_w)}
\left[
\log p_w(\bY_w\mid\bF_w,\cO_w)
\right]
-
\operatorname{KL}
\left(
q_\cH(\bU)
\,\middle\|\,
\prod_{w\in\cU}p_0(\bU_w)
\right),
\label{eq:structured-elbo}
\end{equation}
and the expectation is estimated with mini-batching by ancestral Monte Carlo over the latent DAG in topological order, while the KL term is analytic.
Imposing the stronger restriction \(\bm\Lambda_{vw}=\bm 0\) for all \(v\neq w\) gives a mean-field approximation over the inducing outputs (DAG-VI), \(q_{\mathrm{VI}}(\bU)=\prod_{w\in\cU}q_w(\bU_w)\), recovering the DAG-DGP adaptation of \citep{salimbeni2017}. This restriction still propagates uncertainty through the DAG, but removes posterior dependence between distinct inducing outputs, and therefore cannot represent explaining-away (see App.~\ref{app:explaining-away}) or compositional uncertainty, i.e., posterior uncertainty over the unobserved latent composition of functions \citep{ustyuzhaninov2020}.

\paragraph{Scaling up to larger DAGs.}
\begin{wrapfigure}[14]{r}{0.36\textwidth}
    \captionsetup{type=figure}\centering
    \includegraphics[width=\linewidth]{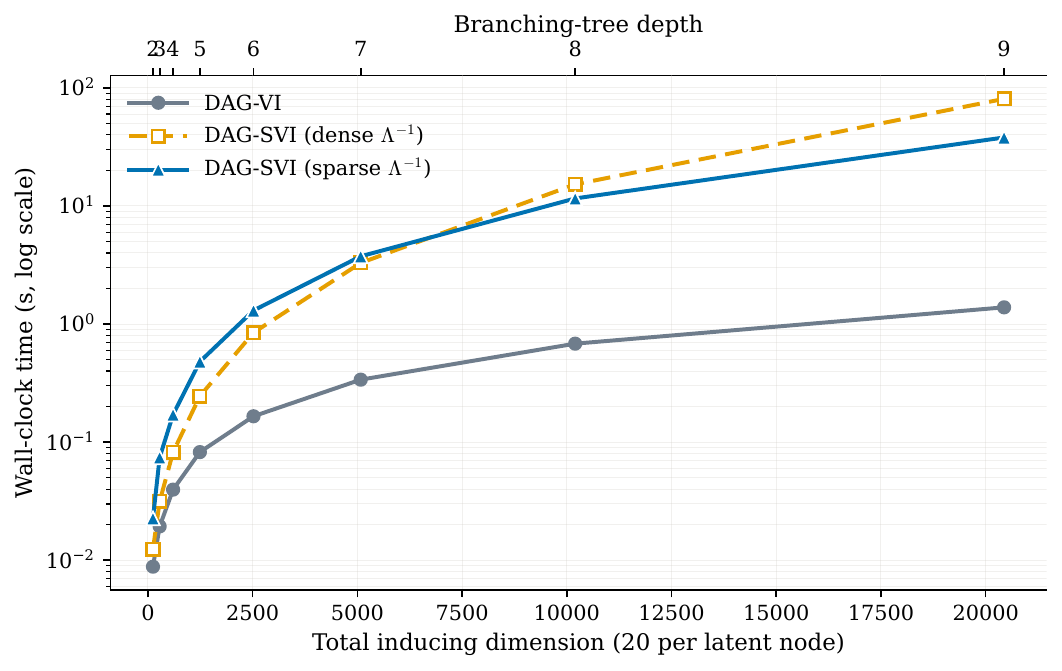}
    \caption{Wall-clock time per ELBO evaluation on increasingly deep branching trees.}
    \label{fig:branching-elbo-scaling}
\end{wrapfigure}
The expectation in Eq.~\eqref{eq:structured-elbo} is estimated by marginal
ancestral sampling (see Prop.~\ref{prop:app-marginal-ancestral-sampler}).
Since \(q_\cH\) is in canonical form, the sampler requires marginal and
conditional moments, equivalently selected applications of \(\bm \Lambda^{-1}\).
Let \(J=|\cU|\), \(\smash{M=\max_w\dim(\bU_w)}\), and \(K=SB\) for \(S\) Monte
Carlo samples and minibatch size \(B\). For large DAGs, DAG-VI is the most scalable option, achieving $\mathcal{O}(JM^3+KJM^2)$ complexity at the expense of expressivity. DAG-SVI trades higher cost for more expressivity. A dense implementation
forms \(\bm\Sigma=\bm\Lambda^{-1}\), costing \(\mathcal O((JM)^3+K(JM)^2)\) per
ELBO evaluation, and is preferable when \(JM\) is moderate or \(\cH\) is close
to dense. For larger sparse DAGs, the block sparsity of \(\bm\Lambda\)
(Fig.~\ref{fig:dag-precision-overview}) can instead be exploited directly:
sparse Cholesky provides the Gaussian solves required by the sampler, while
Takahashi selected-inversion recursions provide the diagonal covariance blocks
needed for the analytic KL term \citep{takahashi1973, erisman1975}. In favourable sparse regimes, the base sparse cost is
\(\mathcal O(Jc^2M^3+KJM^2)\), where \(c\) is the largest number of inducing
blocks in any clique of \(\cH\), up to conditioning-update factors (App.~\ref{app:computational-cost}).
Sparse algebra pays off earlier on DAGs that split into disconnected branches, but at small inducing dimensions its overhead can outweigh the gain, as also observed in similar structured GP settings \citep{durrande2019,hamelijnck2021}. Fig.~\ref{fig:branching-elbo-scaling} reports wall-clock time on branching trees (App.~\ref{app:branching-tree} for a $W$-wide DAG).
\section{On Prior and Posterior Non-collapse in DAG-DGPs}\label{sec:theory}
A central pathology of DGP priors is the loss of separation between distinct inputs under
repeated composition \citep{duvenaud2014,dunlop2018,tong2021}, hindering
information preservation across depth. Following recent usage \citep{meng2024}, we call
this \emph{prior collapse}. For DAG-DGPs this raises a richer
question: given two input cases \(a\neq b\), when does a difference at the roots,
or refreshed by internal observations, remain visible downstream? Following
\citep{dunlop2018}, we measure distinguishability at node \(w\) via the
two-case contrast of latent values
\(
  \smash{\Delta_w:=F_w^{(a)}-F_w^{(b)}},
\)
and say \(w\) carries an \(\varepsilon\)-contrast for \((a,b)\) if
\(\|\Delta_w\|_2>\varepsilon\). Input connections in chain DGPs
\citep{neal1995, duvenaud2014} can be viewed as deterministic refreshes of
this contrast; general DAGs admit further such mechanisms via topology,
intermediate observations, and root-dependent fusion kernels.
Results in this section characterise the DAG-DGP as a model class, and therefore provide guidance on basic properties that every approximate inference scheme should preserve.
\subsection{Repeated separating nodes prevent prior collapse}
\label{sec:prior-noncollapse-separating}
Unlike chains, DAGs lack a unique notion of layer. We therefore consider
progressive antichain decompositions
\(
  \cV=\bigsqcup_{\ell=0}^{h-1}\cA_\ell ,
\)
where each \(\cA_\ell\) is an antichain (no two nodes connected by a directed
path), and progressive means that every node in \(\cA_\ell\) reaches some node in \(\cA_{\ell+1}\).
Each \(\cA_\ell\) acts as a depth slice, recovering single-layer slices in
chains. Every finite DAG admits such a decomposition
(App.~\ref{app:antichain-decomposition}); asymptotic statements refer to infinite antichain sequences or their finite truncations.

We track the largest contrast on an antichain,
\(
  M_\ell:=\max_{w\in\cA_\ell}\|\Delta_w\|_2 ,
  \label{eq:main-M-ell}
\)
since a single non-collapsed node suffices to distinguish the two cases at
depth~\(\ell\). Prior collapse for \((a,b)\) means \(M_\ell\to0\) a.s.

For a non-root node \(w\), let \(\Gamma_w(a,b)\) be the conditional
covariance of \(\Delta_w\) given parent evaluations
\(\{F_p^{(a)},F_p^{(b)}\}_{p\in\PA(w)}\). We call \(w\)
\(v_\star\)-separating for \((a,b)\) if
\(
  \Gamma_w(a,b)\succeq v_\star I_{d_w}
\)
a.s.; this requires the parents to expose a contrast visible to the kernel
at \(w\). A separating node (anchor) injects at least \(v_\star\) of conditional
variance into the difference, giving a uniformly positive chance of
counteracting collapse.

\begin{theorem}[Repeated separating nodes prevent prior collapse]
\label{thm:separating-noncollapse}
Under the DAG-DGP prior, suppose that every \(\cA_\ell\) in a
progressive antichain sequence contains at least \(s\ge1\)
\(v_\star\)-separating nodes for the pair \((a,b)\). Then,
\[
\forall \varepsilon >0:\qquad
\smash{
  \liminf_{m\to\infty}
  \frac{1}{m}
  \sum_{\ell=0}^{m-1}
  \mathbf{1}\{M_\ell>\varepsilon\}
  \ge
  1-(1-p_\varepsilon)^s
  }
\qquad\text{a.s.}\qquad
\smash{
 p_\varepsilon:=2\Phi_{\mathrm N}\left(-\frac{\varepsilon}{\sqrt{v_\star}}\right)}.
  \label{eq:main-separating-noncollapse}
\]
\end{theorem}
Hence non-trivial contrasts occur on a positive fraction of depths, ruling
out prior collapse (Fig.~\ref{fig:main-theory-validation}a). In a chain
each antichain has a single latent node, and the input connection induces
separation; for the squared-exponential input-connected chain of
\citep[Remark~5(3)]{dunlop2018}, Corollary~\ref{cor:dunlop-skip-noncollapse}
shows this node is separating at every depth, so the bound applies with
\(s=1\). Anchors need not be placed everywhere: if separating nodes occur
infinitely often, permanent collapse is ruled out, and if they occur on a
positive fraction of depths the bound is multiplied by that fraction
(App.~\ref{app:separating}).

Separation is easy to certify if the kernel retains non-degenerate dependence on a root parent. Additive and ANOVA root-only components give
sufficient conditions, as other parents cannot cancel their contribution;
this covers multifidelity fusion kernels
\citep{perdikaris2017, cutajar2019, ji2024}, such as that of the
heavy-ion task (Sec.~\ref{sec:heavy-ion}). Perfectly observed internal nodes
act as separating coordinates after conditioning
(App.~\ref{app:conditional-refresh}).

\begin{figure}[t]
    \centering
    \begin{tikzpicture}
      \node[anchor=south west, inner sep=0pt] (mt) at (0,0) {\includegraphics[
          width=\textwidth,
          trim=3bp 3bp 3bp 3bp,
          clip
      ]{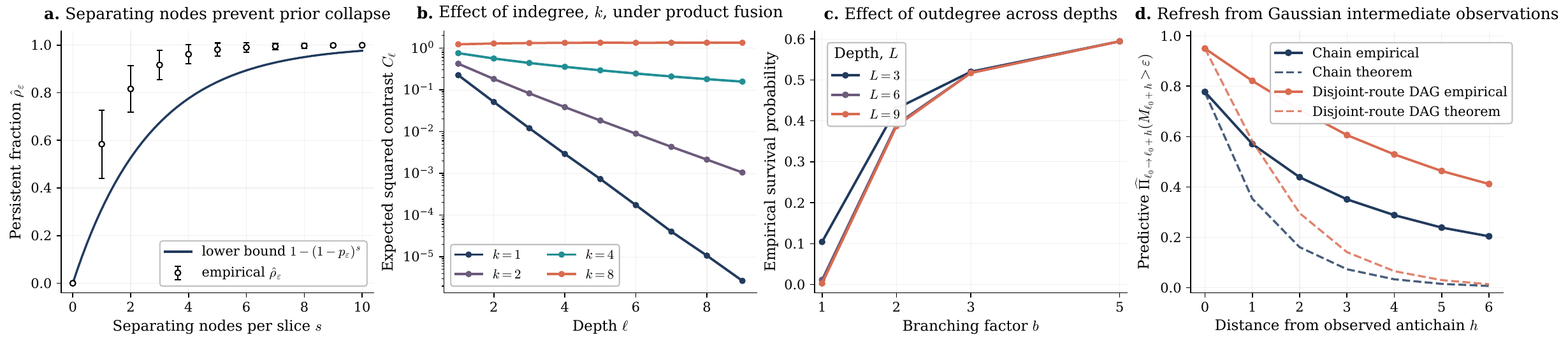}};
    \end{tikzpicture}
    \caption{Empirical validation of the main theoretical results in Sec.~\ref{sec:theory}. See also App.~\ref{app:theory-validation}.}
    \label{fig:main-theory-validation}
\end{figure}
\subsection{Effect of the DAG topology}
\label{sec:degree-effects}
Theorem~\ref{thm:separating-noncollapse} identifies anchors against
collapse; we next isolate topological effects, starting with indegree. Consider
a local layered radial block (Fig.~\ref{fig:dag-precision-overview}d shows a layered DAG),
where parents of \(\cA_\ell\) lie in \(\cA_{\ell-1}\) and each node uses a
radial kernel on the concatenated parent state. The expected contrast then
admits an explicit recursion (full assumptions in
App.~\ref{app:degree-effects}). For clarity we display product fusion of
squared-exponential parent kernels; App.~\ref{app:degree-effects} extends
the recursion to general monotonic Laplace-radial kernels, including rational-quadratic kernels, with layer-wise dimensions and hyperparameters.

Let \(k_\ell=\max_{w\in\cA_\ell}|\PA(w)|\) and
\(C_\ell:=\max_{w\in\cA_\ell}\EE[\|\Delta_w\|_2^2]\). For product
squared-exponential fusion with variance \(\tau^2\), length-scale \(\lambda\),
and output dimension \(d\), App.~\ref{app:degree-effects} gives
\[
  C_\ell
  \le
  2d\tau^2
  \left[
    1-
    \left(
      1+C_{\ell-1}/({d\lambda^2})
    \right)^{-dk_\ell/2}
  \right].
\]
For \(k_\ell=1\) this recovers the chain recurrence of \citep{tong2021}; for
\(k_\ell>1\), the kernel sees contrast accumulated across
incoming edges, so contrast can persist across depth for large \(k_\ell\) (Fig.~\ref{fig:main-theory-validation}b). Near zero the slope is
\(d\tau^2 k_\ell/\lambda^2\), yielding the sufficient contraction criterion
\(d\tau^2\bar k/\lambda^2<1\) with \(\bar k:=\sup_\ell k_\ell\). Indegree
thus enters the contraction threshold linearly, shrinking by a factor
\(\bar k\) the hyperparameter range (in \(d\tau^2/\lambda^2\)) where this contraction is
guaranteed under product fusion.
Additive root-retaining fusion is insensitive to indegree, provided its
root-retaining component is separating
(Prop.~\ref{prop:additive-monotone}, App.~\ref{app:degree-proofs}).

Outdegree provides a complementary mechanism: large outdegree yields many
parallel descendants, hence many conditionally independent chances for a
contrast to survive. In the same layered radial regime, consider a
\(b\)-ary branching subgraph from a seed node \(w_0\), a topology
arising in various probabilistic models \citep{jordan1994, liu2024}. Call a node active if its contrast is at least \(t>0\). If the seed is active with
positive probability and each child of an active node is active with
probability at least \(p_t>1/b\), then with positive probability
\(M_\ell\ge t\) at every depth
(Prop.~\ref{prop:outdegree-branching}, App.~\ref{app:outdegree-branching};
Fig.~\ref{fig:main-theory-validation}c).

\subsection{Intermediate observations as stochastic skip connections}
\label{sec:posterior-refresh}

Skip connections preserve information by reintroducing input-dependent
variation at later depths; intermediate observations have an analogous
posterior effect. For two cases \(a\neq b\) and an observed internal node
\(u\) (e.g., \(L_1\), \(L_2\) in Fig.~\ref{fig:heavy_ion_dag}), after observing \(Y_u^{(a)}, Y_u^{(b)}\) the posterior may assign
substantial mass to
\(
  \|F_u^{(a)}-F_u^{(b)}\|_2>\varepsilon ,
\)
making the observation a local source of separation at \(u\). Being noisy,
it does not guarantee a genuine latent contrast, but updates the posterior
mass on the latent \(\varepsilon\)-contrast event at \(u\). To affect later
antichains, this contrast must then propagate along directed routes.

Let \(\Pi_{\ell_0}\) be the filtering posterior after assimilating
observations on the strict ancestors of \(\cA_{\ell_0}\) and on
\(\cA_{\ell_0}\) itself, \(\mathscr H_{\ell_0}\) the sigma-field of strict
ancestral states, and \(\Pi_{\ell_0\to\ell_1}\) the forward predictive law
from \(\cA_{\ell_0}\) to \(\cA_{\ell_1}\) without assimilating later
observations. An admissible route has off-route parents that are roots or known at the source antichain, so its contrast can be tracked in isolation; an \(\varepsilon\)-retaining
route with factor \(\rho\) carries an \(\varepsilon\)-contrast from source
to target with probability at least \(\rho\); interior-disjoint routes, sharing no nodes beyond their sources, propagate independently (definitions in App.~\ref{app:posterior-refresh}).

\begin{theorem}[Intermediate observations act as stochastic skip connections]
\label{thm:stochastic-skip-main}
Consider two input cases \(a\neq b\), a threshold \(\varepsilon>0\), and
antichain levels \(\ell_1\ge\ell_0\). Assume that \(s\) distinct observed
source nodes \(u_1,\dots,u_s\in\cA_{\ell_0}\) can reach \(s\) distinct target
nodes \(v_1,\dots,v_s\in\cA_{\ell_1}\) through pairwise interior-disjoint
admissible routes \(\gamma_j:u_j\leadsto v_j\). Assume further that each
\(\gamma_j\) is \(\varepsilon\)-retaining with factor \(\rho_j\in[0,1]\).
Define
\(
  E_j
  :=
  \bigl\{
    \|F_{u_j}^{(a)}-F_{u_j}^{(b)}\|_2>\varepsilon
    \bigr\}\), and
\(
  q_j:=\Pi_{\ell_0}(E_j\mid\mathscr H_{\ell_0}).
\)
Then
\begin{equation}
  \Pi_{\ell_0\to\ell_1}(M_{\ell_1}>\varepsilon)
  \ge
\smash{
  \EE_{\Pi_{\ell_0}}\! \big[
    1-\prod_{j=1}^s(1-\rho_j q_j)
    \big].
  }
  \label{eq:main-stochastic-skip}
\end{equation}
\end{theorem}

Here \(q_j\) is the posterior probability, given the ancestral states, that source \(u_j\) carries a
genuine latent \(\varepsilon\)-contrast, while \(\rho_j\) measures how
likely route \(\gamma_j\) propagates it downstream. Thus \(\rho_j q_j\) lower-bounds the chance that route \(j\) succeeds, and the product in \eqref{eq:main-stochastic-skip} upper-bounds the chance that none does.
In a chain (Fig.~\ref{fig:dag-precision-overview}a) with one observed internal layer and one partially observed downstream route, this reduces to
\(
  \Pi_{\ell_0\to\ell_1}(M_{\ell_1}>\varepsilon)
  \ge
  \EE_{\Pi_{\ell_0}}[\rho q].
\)
In DAGs---e.g.\ the disjoint-route DAG of
Fig.~\ref{fig:dag-precision-overview}b---several observed sources combine
via the probability that at least one propagated contrast reaches the
target antichain; Fig.~\ref{fig:main-theory-validation}d compares the
two regimes empirically.

For a scalar \(u\) with constant prior variance and Gaussian observations \(Y_u^{(i)}=F_u^{(i)}+\xi_u^{(i)}\),
\(\xi_u^{(i)}\sim\cN(0,\sigma_u^2)\), set \(\Gamma=\Gamma_u(a,b)\) and
\(D_Y=Y_u^{(a)}-Y_u^{(b)}\). If \(\Gamma>0\), then conditionally on
\(\mathscr H_{\ell_0}\),
\begin{equation}
  F_u^{(a)}-F_u^{(b)}
  \mid Y_u^{(a)},Y_u^{(b)},\mathscr H_{\ell_0}
  \sim
  \smash[t]{
  \cN\!\left(
    \frac{\Gamma}{\Gamma+2\sigma_u^2}D_Y,\,
    \frac{2\sigma_u^2\Gamma}{\Gamma+2\sigma_u^2}
  \right)}.
  \label{eq:main-gaussian-refresh}
\end{equation}
The shrinkage factor \(\Gamma/(\Gamma+2\sigma_u^2)\) interpolates between observation-driven (\(\sigma_u^2\ll\Gamma\)) and
noise-dominated (\(\sigma_u^2\gg\Gamma\)) regimes, so the \(q_j\) of
Thm.~\ref{thm:stochastic-skip-main} reduce to explicit Gaussian tails,
empirically validated in Fig.~\ref{fig:main-theory-validation}d.
App.~\ref{app:source-strengths} extends the analysis to bounded-curvature
one-parameter exponential families, including Bernoulli and binomial nodes.
\begin{figure}[t]
\centering
\newlength{\panelheight}
\setlength{\panelheight}{0.23\textwidth}
\begin{minipage}[t]{0.30\textwidth}
    \vspace{0pt}
    \centering
    \includegraphics[
        height=\panelheight,
        trim=9pt 8pt 8pt 7pt,
        clip
    ]{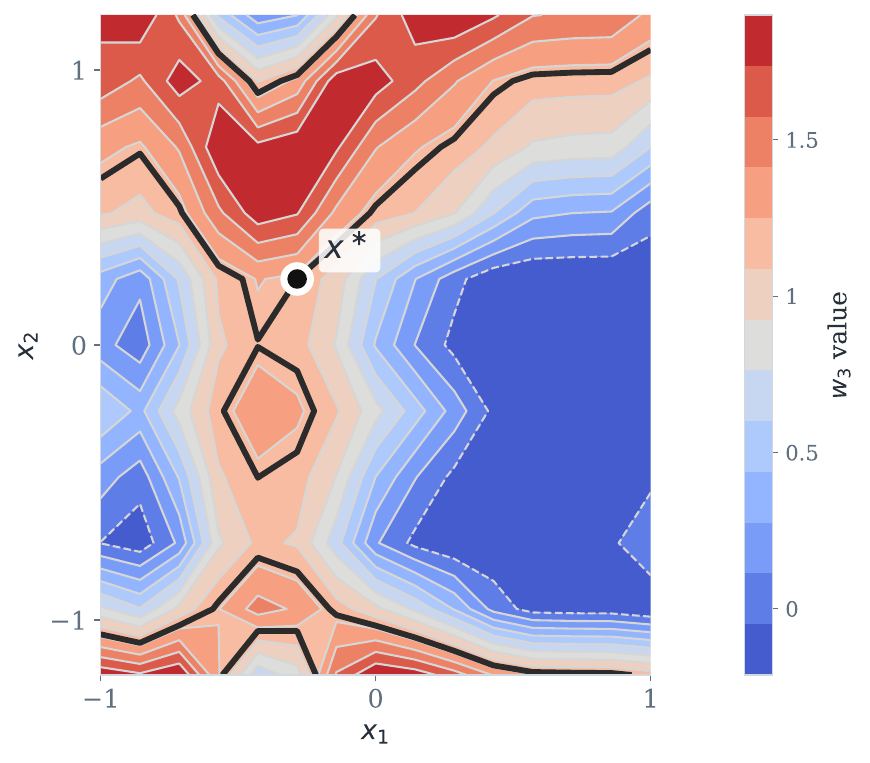}
\end{minipage}
\hfill
\begin{minipage}[t]{0.27\textwidth}
    \vspace{0pt}
    \centering
    \includegraphics[
        height=\panelheight,
        trim=7pt 8pt 7pt 7pt,
        clip
    ]{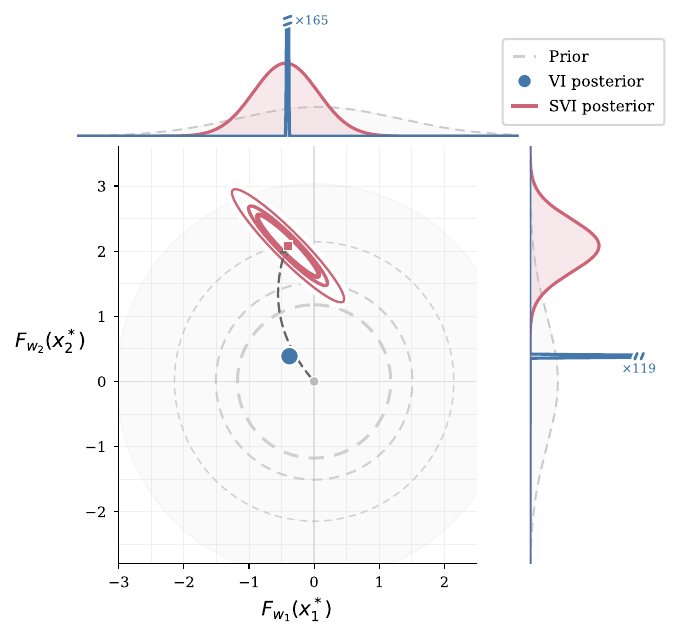}
\end{minipage}
\hfill
\begin{minipage}[t][\panelheight][c]{0.36\textwidth}
    \centering
    \includegraphics[
        width=\linewidth,
        trim=7pt 8pt 7pt 7pt,
        clip
    ]{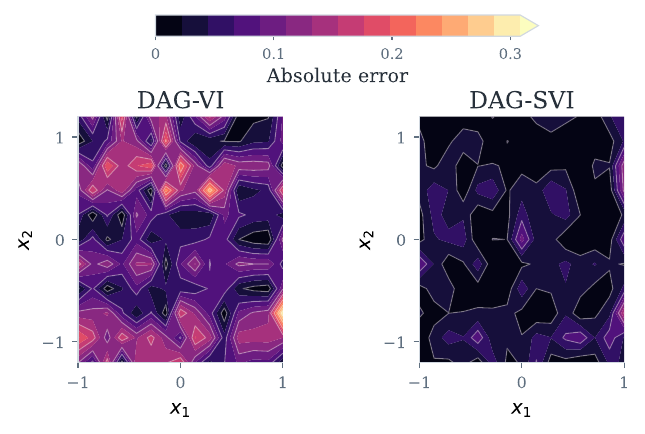}
\end{minipage}
\caption{\textbf{Latent-collider experiment.} Left: ground-truth $w_3$, input marked. Centre: parent posterior under DAG-VI and DAG-SVI. Right: per-method absolute error on $w_3$.}
\label{fig:explaining-away-collider}
\end{figure}
\section{Compositional Uncertainty, Explaining Away and DAG Misspecification}\label{sec:collider}
\begin{wrapfigure}{r}{0.28\textwidth}
  \captionsetup{type=figure}\centering
  \includegraphics[width=\linewidth]{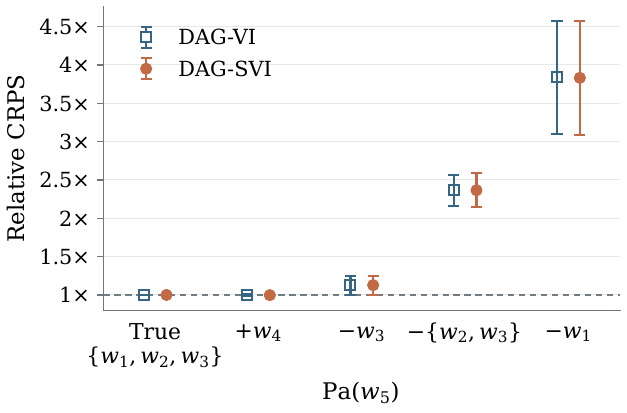}
  \caption{Test CRPS relative to the true DAG as parents are removed or added.}
  \label{fig:misspecification}
\end{wrapfigure}
We consider synthetic collider DAG-DGPs in which \(K\) independent GP
parents, \(w_1,\dots,w_K\), feed a partially observed child \(w_{K+1}\).
For \(K=2\) the parents are unobserved and the child combines them through an RBF kernel on
their sum.
Fig.~\ref{fig:explaining-away-collider} (centre) visualises the posterior over the two parent latent nodes at a fixed input, illustrating two posterior phenomena that DAG-SVI captures and DAG-VI cannot represent. In terms of variability, the DAG-SVI samples cover a broad range of latent parent representations, preserving compositional uncertainty; DAG-VI, restricted by construction to factorised marginals over \(w_1\) and \(w_2\), instead concentrates onto a near-degenerate point. In terms of dependencies, the DAG-SVI posterior is negatively correlated, encoding the reciprocal compensation through which the two parents explain the child, signature of explaining away; DAG-VI produces uncorrelated samples by design and cannot represent this coupling. App.~\ref{cu_ea} illustrates it geometrically, and Fig.~\ref{fig:explaining-away-collider}(right) shows that coupling improves reconstruction of the partially observed child.

We then study sensitivity to the specified DAG with $K=4$ partially observed
candidate parents and additive RBF fusion, where $w_1,w_2,w_3$ are
informative with decreasing contributions, while $w_4$ is irrelevant.
Adding $w_4$ as a parent leaves the accuracy of DAG-SVI and DAG-VI unchanged in root-mean-square error (RMSE) and continuous ranked probability score (CRPS),
whereas removing informative parents degrades it according to their
contributions (Fig.~\ref{fig:misspecification},
App.~\ref{app:misspecification}).

\section{Protein Signalling Network from the Sachs Flow Cytometry Dataset}\label{sec:sachs}

We evaluate the proposed DAG-DGP framework on the real-world protein signalling network dataset of \citep{sachs2005}, a widely used benchmark in the causal discovery literature \citep{moiji2020}. Since our focus is observational modelling, we consider the \textit{cd3cd28 icam2} subset, which corresponds to a specific intervention regime. The dataset contains $902$ observations over $11$ variables, which we log-transform prior to modelling. We use the \textit{bnlearn} DAG \citep{bnlearn:2010} (Fig.~\ref{fig:sachs_dag}), random 80/20 train--test splits and 10k joint training iterations. Multi-parent nodes use product fusion of RBF kernels, so that parents can interact. Intermediate observations anchor the deeper nodes, acting as stochastic skip connections (Thm.~\ref{thm:stochastic-skip-main}).

\suppressfloats[t]
\begin{figure}[ht]
    \centering
    \begin{lrbox}{\sidebysidebox}\begin{minipage}[t]{0.55\textwidth}
        \vspace{0pt}\centering
        \captionof{table}{Sachs results (averaged over non root nodes) and training time. We report mean $\pm$ std (5 seeds).
        The top block assesses accuracy, the bottom one scalability.}
        \label{table:mf_svi_results}
        \setlength{\tabcolsep}{3pt}
        \resizebox{\linewidth}{!}{\begin{tabular}{@{}clcccc@{}}
                \toprule
                $N$ & Model
                & RMSE$\downarrow$
                & CRPS$\downarrow$
                & ES$\downarrow$
                & Time (min)$\downarrow$ \\
                \midrule
                \multirow{3}{*}{902}
                & DAG-VI (full batch)
                & $0.635\pm0.015$
                & $0.327\pm0.007$
                & $1.454\pm0.041$
                & $11.9\pm0.5$ \\
                & DAG-SVI (full batch)
                & $\mathbf{0.633}\pm0.016$
                & $\mathbf{0.323}\pm0.005$
                & $\mathbf{1.444}\pm0.038$
                & $15.9\pm0.8$ \\
                & MO-GP (full batch)
                & $0.745\pm0.021$
                & $0.401\pm0.009$
                & $1.803\pm0.057$
                & $\mathbf{3.1}\pm0.1$ \\
                \midrule
                \multirow{3}{*}{$11{,}672$}
                & DAG-VI (1024 mini-batch)
                & $\mathbf{0.710}\pm0.005$
                & $\mathbf{0.358}\pm0.002$
                & $\mathbf{1.161}\pm0.009$
                & $\mathbf{17.8}\pm0.7$ \\
                & DAG-SVI (1024 mini-batch)
                & $0.715\pm0.004$
                & $0.361\pm0.001$
                & $1.170\pm0.005$
                & $35.4\pm0.4$ \\
                & MO-GP (full batch)
                & $0.823\pm0.005$
                & $0.445\pm0.002$
                & $1.369\pm0.008$
                & $29.0\pm0.0$ \\
                \bottomrule
            \end{tabular}}
    \end{minipage}\end{lrbox}\begin{minipage}[t][\dp\sidebysidebox][s]{0.42\textwidth}
        \vspace{0pt}\centering
        \includegraphics[
            width=0.66\linewidth,
            trim=10 8 10 9,
            clip
        ]{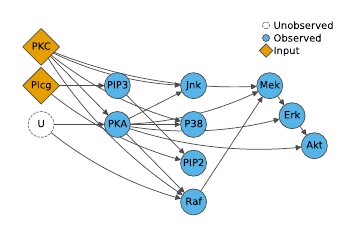}
        \vfill
        \captionof{figure}{DAG for Sachs, where the dotted node (U) is a 2D latent.}
        \label{fig:sachs_dag}
    \end{minipage}\hfill
    \usebox{\sidebysidebox}
\end{figure}

It is well known that the Sachs data were not generated under ideal interventions and may contain latent confounding, which we model through a 2D latent variable layer \citep{salimbeni2019}; inferring it at test time requires PKA and Raf, which are therefore observed in both training and test sets for all models. We compare with a multi-output GP (MO-GP), a sparse linear model of coregionalisation \citep{AlvarezRosascoLawrence2012KernelsVectorValuedReview} taking the two root variables as inputs and the other nine nodes as outputs, and report, besides RMSE and CRPS, the Energy Score (ES), a proper scoring rule for the joint predictive distribution \citep{gneiting2007} (Tab.~\ref{table:mf_svi_results}). Both DAG-VI and DAG-SVI outperform MO-GP on all metrics, suggesting that, in this task, composing the variables along the DAG is more effective than modelling cross-output covariance alone. The favourable marginal and multivariate predictive performance suggests that DAG-VI and DAG-SVI capture relevant dependencies encoded by the DAG. App.~\ref{app:sachs_exp} details training and an extrapolation split following \citep{lindinger2020}, where DAG-SVI performs best. Finally, we pool all experimental conditions (\(N=11{,}672\)) to assess scalability only, since their different interventions alter node mechanisms, which we do not model here; both DAG variants still outperform the full-batch MO-GP and train in 18--35 minutes over 30k joint iterations.

\section{Deploying DAG-DGPs for Multi-Fidelity Heavy-ion Emulation}\label{sec:heavy-ion}
We evaluate DAG-DGPs on the heavy-ion collision dataset of
\citep{ji2024}, a graphical multi-fidelity emulation problem with a shared
nine-dimensional input and a scalar pion-yield-ratio output. The elicited
simulator graph has two lower-fidelity nodes, \(L_1\) and \(L_2\), feeding
the high-fidelity node \(H\) (Fig.~\ref{fig:heavy_ion_dag}). Rather than sequential, these lower
fidelities are complementary: \(L_1\) uses simplified linearized conformal
hydrodynamics followed by Cooper--Frye conversion, whereas \(L_2\) uses
\(1+1\)D ideal QCD (quantum chromodynamics) hydrodynamics but omits this conversion stage. This makes
it meaningful to study how each approximation contributes to explaining the
high-fidelity response, and whether one provides more informative support
for \(H\) than the other. Accordingly, we fit a DAG-DGP with RBF kernels at \(L_1\)
and \(L_2\) and the graphical multi-fidelity kernel of \citep{ji2024} at \(H\).
Both anchoring mechanisms of Sec.~\ref{sec:theory} apply, as the kernel's
input-connection term makes every non-root node separating
(Rem.~\ref{rem:additive-root-separation}) and the observations at \(L_1,L_2\) act as stochastic skip connections (Thm.~\ref{thm:stochastic-skip-main}).

We report three evaluations. On the \emph{published split} of
\citep{ji2024}, with 200 observations at each lower fidelity, 25 observations
at \(H\), and the original 75 high-fidelity test points, DAG-VI improves on the graphical multi-fidelity Gaussian process (GMGP) family and their benchmark comparison; DAG-SVI gives the best performance
on RMSE, normalised RMSE (N-RMSE; see \cite{ji2024}), and CRPS; see Tab.~\ref{tab:heavy_ion_results}.
In Tab.~\ref{tab:heavy_ion_results}, a two-layer DGP on \((X,L_1,L_2)\) has the same data and comparable depth, so its gap to the DAG-DGPs isolates the graphical inductive bias from the mere availability of the fidelity data, while the MO-GP suggests that cross-output correlations alone do not capture the fidelity structure (further special cases in App.~\ref{app:heavyion}).
Shapley values for DAG-SVI (Fig.~\ref{fig:heavy_ion}) show that both lower-fidelity branches contribute substantially for a large fraction of high-fidelity test points, with \(L_2\), whose QCD hydrodynamics is more realistic, dominant overall.

Two further protocols isolate different effects. The \emph{high-fidelity-scarce} protocol performs repeated 5-fold cross-validation over the 25 observations at \(H\), testing cross-fidelity transfer when high-fidelity supervision is most limited; both DAG-DGPs improve on d-GMGP and MO-GP, with DAG-SVI best. The \emph{full-hierarchy} protocol uses 200, 200, and 100 observations at \(L_1,L_2,H\) and 10-fold cross-validation with held-out data at all fidelities to evaluate joint prediction; both DAG-DGPs improve on MO-GP in the joint energy score and at \(H\), with DAG-SVI marginally ahead of DAG-VI and both stable across two inducing-input initialisations at \(H\). App.~\ref{app:heavyion} details results and training, where pretraining alone stabilises \(H\) and the improvement arrives once the joint objective takes over.
\begin{figure*}[ht]
  \centering
  \begin{lrbox}{\sidebysidebox}\begin{minipage}[t]{0.48\textwidth}
    \vspace*{0pt}
    \centering
    \small
    \setlength{\tabcolsep}{4pt}
    \captionof{table}{Published split (75 test points). Top block from \citep{ji2024}; $^\dagger$ours, mean\,$\pm$\,std (4 seeds).}
    \label{tab:heavy_ion_results}
    \resizebox{\linewidth}{!}{\begin{tabular}[t]{@{}lccc@{}}
      \toprule
      Model & \shortstack[c]{RMSE\\$/10^{-2}$ $\downarrow$}
            & \shortstack[c]{N-RMSE\\$\uparrow$}
            & \shortstack[c]{CRPS\\$/10^{-2}$ $\downarrow$} \\
      \midrule
      High-fidelity GP \citep{ji2024}                    & 5.49 & 0.46 & 3.54 \\
      KO-path \citep{kennedy2000}         & 3.48 & 0.66 & 1.99 \\
      KO-misspecified \citep{kennedy2000} & 3.95 & 0.71 & 2.30 \\
      NARGP \citep{perdikaris2017}        & 3.66 & 0.73 & 2.13 \\
      r-GMGP \citep{ji2024}               & 2.92 & 0.72 & 1.64 \\
      d-GMGP \citep{ji2024}               & 2.17 & 0.79 & 1.34 \\
      \midrule
      MO-GP$^\dagger$ & $2.70 \pm 0.19$ & $0.739 \pm 0.019$ & $1.53 \pm 0.15$ \\
      2-DGP \((X,L_1,L_2)\)$^\dagger$ & $2.64 \pm 0.06$ & $0.745 \pm 0.006$ & $1.26 \pm 0.03$ \\
      DAG-VI$^\dagger$  & $2.12 \pm 0.02$ & $0.796 \pm 0.002$ & $1.13 \pm 0.02$ \\
      DAG-SVI$^\dagger$ & $\mathbf{2.03} \pm 0.01$ & $\mathbf{0.804} \pm 0.001$ & $\mathbf{1.11} \pm 0.01$ \\
      \bottomrule
    \end{tabular}}
  \end{minipage}\end{lrbox}\usebox{\sidebysidebox}\hfill
  \begin{minipage}[t][\dp\sidebysidebox][s]{0.48\textwidth}
    \vspace*{0pt}
    \centering
    \includegraphics[
      width=\linewidth
    ]{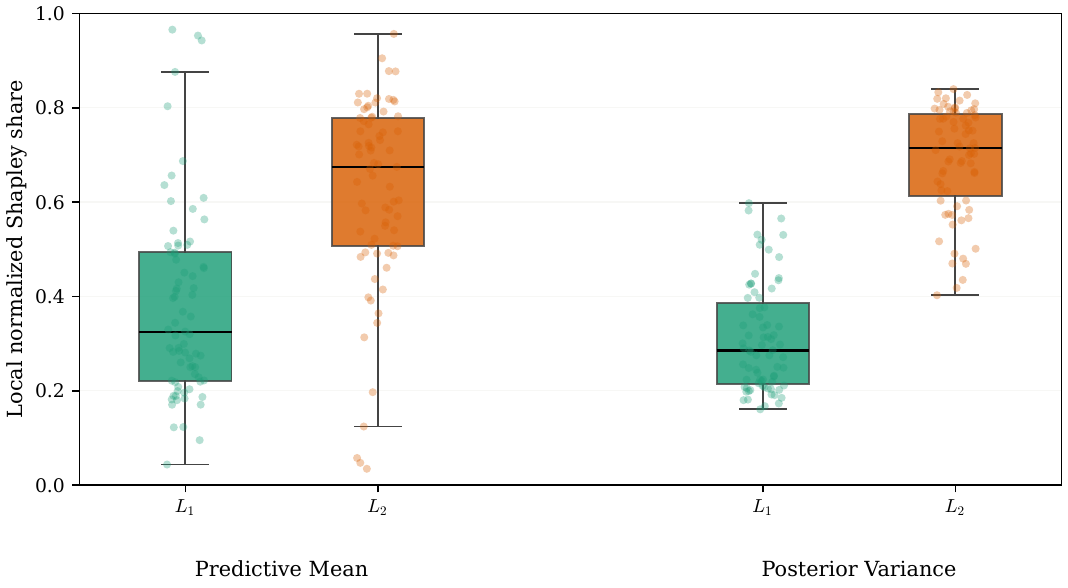}
    \vfill
    \captionof{figure}{Normalized Shapley shares of \(L_1,L_2\) for the mean (left) and variance (right) of \(H\).}
    \label{fig:heavy_ion}
    \vspace{1pt}\end{minipage}
\end{figure*}
\section{Discussion}
\subsection{DAG-DGPs as a General Framework}
DAG-DGPs recover and extend several GP architectures by
restricting three components: the DAG topology, the nodewise observation
pattern, and the fusion rule. This yields two concrete benefits. First, it lets
us flexibly model noisy, heterogeneously observed data organised by general DAG
topologies, including graphical multi-fidelity problems and
Bayesian network datasets. Second, whenever these restrictions recover an
existing architecture, the resulting special case inherits our theoretical
analysis and structured variational inference scheme. Our codebase implements this general framework and is available, with all experiments, at \coderepo.

\paragraph{DGP models.}
Standard DGPs \citep{lawrence2007,damianou2013} are recovered by choosing a
chain DAG, using the trivial single-parent fusion rule, and observing only the
terminal layer. Input-connected DGPs \citep{duvenaud2014} add the deterministic
input as a parent of every latent layer. Multi-fidelity DGPs
\citep{cutajar2019} arise by interpreting the chain as an ordered fidelity
hierarchy, observing the corresponding fidelity nodes, and using a
multi-fidelity fusion rule. d-GMGPs \citep{ji2024} further restrict this
multi-fidelity topology to a directed in-tree with nested designs, where each
simulator node is connected to the shared deterministic input and its
lower-fidelity parents. Stochastic deep Gaussian processes over graphs
\citep{li2020} target input--output transformations between signals on a fixed
graph. Although their modelling aim differs from ours, they can be recovered as
layered DAG-DGPs by unrolling the fixed graph over depth, as we show in Prop. \ref{prop:dgpg-dagdgp} of App.~\ref{app:DGPG_app}.
\paragraph{GPN-based models.} GPN-based models \citep{friedman2000,giudice2023,kiroriwal2025bayesian} are obtained by choosing the DAG to be a process network, e.g. a multi-stage system where subprocess outputs feed downstream stages. Classical GPNs correspond to the fully observed case, where all measurements are available, while a DAG-DGP provides propagation of uncertainty through composition. When full observability is relaxed, as recently proposed in Bayesian optimisation \citep{kiroriwal2025bayesian}, subprocess GPs are conditioned on deterministic stage inputs, yielding an input-connected DAG; the RBF kernel specified in that model over the concatenated parent/input space then corresponds to a product fusion rule in our framework.
\subsection{Open Challenges and Future Directions}
We introduced a unifying modelling framework for composition of functions over Directed Acyclic Graphs, motivated by the need to represent such inductive biases of domain-knowledge systems in science and engineering within probabilistic machine learning. Modelling such systems naturally calls for a probabilistic treatment, in which uncertainty over latent functions and their graph-induced dependencies is retained. When the DAG is given a causal interpretation, the framework opens the door for structural causal modelling and causal representation learning \citep{pearl2009,spirtes2001causation,peters2017elements,scholkopf2021toward}. Several directions remain open. Beyond the collapse phenomenon addressed in our theoretical analysis, an intriguing avenue is to explore posterior contraction rates of DAG-DGPs, extending recent advances developed for chain DGPs \citep{finocchio2023}. In practice, domain knowledge specifies the graph only imperfectly; our study in Sec.~\ref{sec:collider} opens a first avenue, but it raises the question of how to robustify DAG-DGPs against DAG misspecification or perform joint inference over DAGs and composing latent functions, drawing inspiration from, e.g., \cite{branchini2023causal,chickering2002optimal,zheng2018dags, witty2020causal, aglietti2020multitask, giudice2024bayesian}. In terms of robustness to likelihood or prior misspecification, our variational framework could be easily extended towards Generalised Variational Inference (GVI) \citep{knoblauch2022}. Our structured and mean-field doubly stochastic VI schemes trade off posterior dependencies and uncertainty quantification against computational efficiency. DAG-SVI can represent compositional uncertainty, whose loss can be detrimental, notably in latent variable models or in warped GPs \citep{ustyuzhaninov2020,kaiser2018,ustyuzhaninov2020monotonic}. The impact on predictive uncertainty is less clear, as both schemes constrain only the inducing outputs while retaining flexible mixtures (App.~\ref{app:marginal-ancestral-sampling}). Despite some evidence from chain DGPs \citep{lindinger2020} and our own experiments, identifying when structured VI improves predictive performance over mean-field, and when mean-field suffices, remains an open question, particularly without a full Bayesian reference for DAG-DGPs. Alternative modelling trade-offs could be explored via different graph-layering strategies for the approximate posterior \cite{harrigan2006}, or by extending sampling methodologies or new recent hybrid optimization sampling schemes developed for chain DGPs \citep{havasi2018, sauer2023, latz2025sparse} to the full DAG setting. Finally, scaling DAG-DGPs to much larger graphs remains an open challenge, with potential directions ranging from amortised inference and asynchronous distributed training to message passing, state-space formulations, and back-propagation of evidence.
\begin{ack}
We are especially grateful to Yi Ji, Simon Mak, Derek Soeder, J.-F. Paquet, and Steffen A. Bass for making available the code and data for the heavy-ion collision experiment in \citep{ji2024}. We also thank the three anonymous reviewers for their valuable feedback and suggestions. This work was supported by United Kingdom Research and Innovation (UKRI) via grant number EP/Y014650/1, as part of the ERC Synergy project OCEAN.
\end{ack}

\bibliography{references}
\bibliographystyle{plain}

\newpage
\appendix
\newpage
\appendix
\startcontents[appendix]

\section*{Appendix Contents}
{\small
\printcontents[appendix]{}{1}{\setcounter{tocdepth}{2}}
}
\vspace{1em}
\section{Graph-theoretic preliminaries and standing notation}
\label{app:graph-preliminaries}
This appendix provides the necessary background for the theoretical developments which follow. For clarity and completeness, we introduce graph-theoretic and order-theoretic notions, setting the general notation and fixing the conventions adopted. Whenever  possible, we follow standard conventions; see, for
example, \citep[Ch.~1]{bangjensen2010digraphs} for directed graphs, directed
paths, and acyclic digraphs, and \citep[Ch.~3]{stanley2012} for the basic
language of finite partially ordered sets, including chains and antichains.

\subsection{Directed acyclic graphs and reachability}
\label{app:dag-preliminaries}

A \emph{directed graph} is a pair \(\cG=(\cV,\cE)\), where \(\cV\) is a finite
set of vertices and \(\cE\subseteq \cV\times\cV\) is a set of directed
edges.  We write \(u\to v\) when \((u,v)\in\cE\). A \emph{directed path} from
\(v_0\) to \(v_L\) is a sequence
\[
  (v_0,v_1,\dots,v_L)
\]
such that \(v_{r-1}\to v_r\) for every \(r=1,\dots,L\). Its length is
\(L\). Paths of length zero are allowed.  The graph is a directed \emph{acyclic}
graph (DAG) if it contains no directed cycle of positive length.

For a node \(w\in\cV\), its \emph{parent} and \emph{child} sets are
\[
  \PA(w):=\{v\in\cV:v\to w\},
  \qquad
  \CH(w):=\{v\in\cV:w\to v\}.
\]
A \emph{root} is a node with no parents.  In the main text the root set is denoted
by \(\cR\), and \(\cU:=\cV\setminus\cR\) denotes the set of non-root nodes.

We write \(u\sqsubseteq v\) if there exists a directed path from \(u\) to
\(v\), allowing the length-zero path.  We write \(u\sqsubset v\) when
\(u\sqsubseteq v\) and \(u\neq v\).  Because \(\cG\) is acyclic,
\(\sqsubseteq\) is a partial order on \(\cV\).  For any
\(S\subseteq\cV\), throughout the appendix
\[
  \Anc(S):=\{u\in\cV:\exists v\in S\text{ such that }u\sqsubset v\}
\]
denotes the set of strict \emph{ancestors} of \(S\). We also use the shorthand
\(\Anc(v):=\Anc(\{v\})\), and occasionally write
\[
  \overline{\Anc}(S):=\Anc(S)\cup S
\]
for the weak ancestral closure of \(S\).  Similarly,
\(\Desc(S):=\{u\in\cV:\exists v\in S\text{ such that }v\sqsubset u\}\)
denotes strict \emph{descendants}.

\subsection{Chains, antichains, and progressivity}
\label{app:antichain-preliminaries}

A \emph{chain} in the reachability order is a set of vertices any two of which are
comparable.  An \emph{antichain} is a set of vertices no two distinct elements of
which are comparable.  Equivalently, \(A\subseteq\cV\) is an antichain if
there is no directed path from one element of \(A\) to another distinct
element of \(A\).

A \emph{finite progressive antichain decomposition} of a DAG is a partition
\[
  \cV=\bigsqcup_{\ell=0}^{h-1}\cA_\ell
\]
into $h \in \mathbb{N}$ non-empty antichains such that
\[
  \cA_\ell\subseteq \Anc(\cA_{\ell+1}),
  \qquad \ell=0,\dots,h-2.
\]
For asymptotic statements we use an infinite progressive antichain sequence
\((\cA_\ell)_{\ell\ge0}\), or finite truncations thereof, satisfying the same
successive ancestry condition.  The antichains should be read as successive
cross-sections of the DAG rather than as independent layers.
Figure~\ref{fig:progressive-antichains} gives a visualization of a simple DAG partitioned into three antichains.

\usetikzlibrary{fit,backgrounds}

\begin{figure}[t]
\centering
\begin{tikzpicture}[
  x=1.35cm,
  y=1.10cm,
  >=latex,
  latent/.style={
    circle,
    draw=black,
    fill=white,
    line width=0.45pt,
    inner sep=1.2pt,
    minimum size=5.2mm
  },
  dag/.style={
    ->,
    semithick,
    draw=black
  },
  faint/.style={
    ->,
    dashed,
    draw=black!45,
    line width=0.45pt
  },
  antichain/.style={
    draw=black!45,
    fill=black!3,
    line width=0.55pt
  },
  rowlabel/.style={
    anchor=east,
    font=\small
  }
]

\path[antichain] (0.0,  0.0) ellipse [x radius=1.85, y radius=0.43];
\path[antichain] (0.1, -1.7) ellipse [x radius=2.35, y radius=0.43];
\path[antichain] (0.25,-3.4) ellipse [x radius=2.05, y radius=0.43];

\node[rowlabel] at (-2.75, 0.0) {$\cA_0$};
\node[rowlabel] at (-2.75,-1.7) {$\cA_1$};
\node[rowlabel] at (-2.75,-3.4) {$\cA_2$};

\node[latent] (a01) at (-1.3, 0.0) {};
\node[latent] (a02) at ( 0.0, 0.0) {};
\node[latent] (a03) at ( 1.3, 0.0) {};

\node[latent] (a11) at (-1.8,-1.7) {};
\node[latent] (a12) at (-0.4,-1.7) {};
\node[latent] (a13) at ( 1.0,-1.7) {};
\node[latent] (a14) at ( 2.0,-1.7) {};

\node[latent] (a21) at (-1.2,-3.4) {};
\node[latent] (a22) at ( 0.4,-3.4) {};
\node[latent] (a23) at ( 1.7,-3.4) {};

\draw[dag] (a01) -- (a11);
\draw[dag] (a01) -- (a12);
\draw[dag] (a02) -- (a12);
\draw[dag] (a02) -- (a13);
\draw[dag] (a03) -- (a14);

\draw[dag] (a11) -- (a21);
\draw[dag] (a12) -- (a21);
\draw[dag] (a12) -- (a22);
\draw[dag] (a13) -- (a22);
\draw[dag] (a14) -- (a23);
\draw[dag] (a02) -- (a22);

\draw[faint] (a01) to[bend right=12] (a21);
\draw[faint] (a01) to[bend left=10]  (a22);
\draw[faint] (a02) to[bend right=10] (a21);
\draw[faint] (a03) to[bend left=14]  (a23);

\end{tikzpicture}
\caption{Progressive antichain sequence on an example DAG. Each highlighted row is an antichain:
there are no directed paths between distinct nodes in the same row. Solid
arrows show direct parent--child relations, whereas dashed arrows indicate
longer reachability relations. The progressive condition only requires every
node in \(\cA_\ell\) to be a strict ancestor of some node in
\(\cA_{\ell+1}\).
}
\label{fig:progressive-antichains}
\end{figure}

\subsection{Standing probabilistic notation}
\label{app:standing-probabilistic-notation}

The DAG-DGP prior is the one defined in Section~\ref{sec:dagdgp}. Throughout, let
\((\Omega,\mathscr A,\PP)\) be a probability space supporting the mutually
independent nodewise GP modules \(\{f_w:w\in\cU\}\), and let \(\PP\) denote
the joint law induced by the DAG-DGP prior recursion.
For an antichain \(\cA_\ell\), we use the module-generated filtration
\[
  \cF_\ell
  :=
  \sigma(f_v:v\in\Anc(\cA_\ell)).
\] All GP
modules \(\{f_w:w\in\cU\}\) are mutually independent under the prior
probability measure \(\PP\), and \(\EE\) denotes expectation under \(\PP\).
Root states are deterministic (i.e. almost surely constant), and no GP module is associated with a root.

Once two distinct cases \(a\neq b\in[n]\)
are fixed, we write
\begin{equation*}
  \Delta_w:=F_w^{(a)}-F_w^{(b)},
  \qquad w\in\cV,
\end{equation*}
recalling that $F_w^{(i)} = f_w\!\bigl(
    \{F_p^{(i)}\}_{p\in\PA(w)}
  \bigr)$, and here $i \in \{a,b\}$.
For an antichain \(\cA_\ell\), the maximum contrast at depth \(\ell\) is
\begin{equation*}
  M_\ell:=\max_{w\in\cA_\ell}\|\Delta_w\|_2.
\end{equation*}
Let \(\mathbb P_0\) denote the joint law of the DAG-DGP prior, including
the independent nodewise GP modules and the latent states induced by the
DAG recursion. For a non-root node \(w\) and two cases \(a,b\), say
\[
  U_{w,a}:=(F_p^{(a)})_{p\in\PA(w)},
  \qquad
  U_{w,b}:=(F_p^{(b)})_{p\in\PA(w)}
\]
for the parent-input configurations generated by the DAG-DGP prior. The
two-point contrast covariance at \(w\) is defined as
\[
  \Gamma_w(a,b)
  :=
  \operatorname{Cov}_{\PP}\!\left(
    F_w^{(a)}-F_w^{(b)}
    \,\middle|\,
    U_{w,a},U_{w,b}
  \right).
\]
Equivalently, by the finite-dimensional distributions of the fresh GP module
\(f_w\),
\begin{equation*}
  \Gamma_w(a,b)
  =
  K_w(U_{w,a},U_{w,a})
  +K_w(U_{w,b},U_{w,b})
  -K_w(U_{w,a},U_{w,b})
  -K_w(U_{w,b},U_{w,a}).
\end{equation*}
A non-root node \(w\) is \emph{\(v_\star\)-separating} for \((a,b)\) if
\begin{equation*}
  \Gamma_w(a,b)\succeq v_\star I_{d_w}
  \qquad \PP\text{-almost surely}.
\end{equation*}
This condition means that, for \(\PP\)-almost every pair of parent-input
configurations generated by the DAG-DGP prior, the fresh two-point GP
contrast at node \(w\), evaluated at those inputs, has covariance bounded
below. This is a pathwise conditional non-degeneracy condition along the DAG-DGP
prior.  It is meant to abstract mechanisms that refresh the deep
composition by ensuring that, after the parent inputs of \(w\) have been
realised, the fresh nodewise GP module still sees a non-degenerate
two-case contrast.  In the architectures that we study in this work, the
refresh is provided by structural coordinates or kernel components whose
separation is not washed out by the upstream stochastic composition, but other such mechanisms can be considered in specialised modelling settings.
Developing similar results under weaker conditions, perhaps under specific kernel choices and fusion rules, would require a different analysis
and is left beyond the scope of the present work.

Here and below \(\succeq\) denotes the Loewner order on symmetric matrices. Notice that $v_\star$ separation is not a purely local property but imposes some requirements on the ancestors of the node in question: two distinct root input case configurations must propagate through the earlier part of the graph and remain sufficiently distinguished that they can be separated at this node. In this sense, although $v_\star$-separation is a useful concept, establishing it in particular graph topologies can be non-trivial. We provide some examples in Appendix~\ref{app:sufficient-separation}; these cover a number of important cases.

For a progressive antichain sequence \((\cA_\ell)_{\ell\ge0}\), the prior
non-collapse proofs use the module-generated filtration with elements
\begin{equation}
  \cF_\ell
  :=
  \sigma\!\left(f_v:v\in\Anc(\cA_\ell)\right),
  \qquad \ell\ge0.
  \label{eq:app-module-filtration-def}
\end{equation}
The posterior-refresh section uses a different, state-generated notation.
\(\mathscr H_\ell\) records strict ancestral state matrices, whereas
\(\mathscr F_\ell\) also includes the current antichain.  These are motivated and defined
precisely in Appendix~\ref{app:filtering-kernels}.

\subsection{Routes below an antichain}
\label{app:route-conventions}

The stochastic-skip arguments use directed routes whose off-route parents
are already known at the source antichain.  We define these route notions
once here.

\begin{definition}[Admissible route]
\label{def:admissible-route}
Fix \(\ell_1\ge\ell_0\). A directed path
\[
  \gamma=(v_0,v_1,\dots,v_L)
\]
with \(v_0\in\cA_{\ell_0}\) and \(v_L\in\cA_{\ell_1}\) is
\emph{admissible below \(\cA_{\ell_0}\)} if, for each
\(r=1,\dots,L\), every parent of \(v_r\) is either \(v_{r-1}\), a root
node, or belongs to \(\overline{\Anc}(\cA_{\ell_0})\).  Its interior is
\begin{equation*}
  \operatorname{int}(\gamma):=\{v_1,\dots,v_L\}.
\end{equation*}
If \(L=0\), the route is degenerate and its interior is empty.
\end{definition}

\begin{definition}[Disjoint routes below an antichain]
\label{def:disjoint-routes}
A family of admissible routes
\(\gamma_j=(v_{j,0},\dots,v_{j,L_j})\), \(j\in[s]\), is
\emph{pairwise disjoint below \(\cA_{\ell_0}\)} if
\begin{equation*}
  \operatorname{int}(\gamma_j)
  \cap
  \operatorname{int}(\gamma_k)
  =
  \emptyset,
  \qquad j\neq k.
\end{equation*}
\end{definition}
Figure~\ref{fig:route-conventions} provides an illustration of the two definitions.
\begin{figure}[t]
\centering
\resizebox{\linewidth}{!}{\begin{tikzpicture}[
  x=0.95cm,
  y=0.95cm,
  >=latex,
  latent/.style={
    circle,
    draw=black,
    fill=white,
    line width=0.45pt,
    inner sep=1.1pt,
    minimum size=4.8mm
  },
  badnode/.style={
    circle,
    draw=black,
    fill=black!18,
    line width=0.8pt,
    inner sep=1.1pt,
    minimum size=4.8mm
  },
  route/.style={
    ->,
    very thick,
    draw=black
  },
  routealt/.style={
    ->,
    very thick,
    draw=black!65
  },
  side/.style={
    ->,
    semithick,
    draw=black!45
  },
  bad/.style={
    ->,
    semithick,
    densely dashed,
    draw=black
  },
  faint/.style={
    draw=black!40,
    fill=black!3,
    line width=0.5pt
  },
  known/.style={
    draw=black!35,
    dashed,
    fill=black!2,
    rounded corners=2pt,
    line width=0.45pt
  },
  intset/.style={
    draw=black!35,
    dashed,
    line width=0.45pt
  },
  lab/.style={font=\small},
  slab/.style={font=\scriptsize}
]

\begin{scope}

\node[lab] at (0,2.10) {(a) admissible route};

\draw[known] (-2.10,-0.38) rectangle (2.10,1.45);
\node[slab] at (0,0.6) {\(\overline{\Anc}(\cA_{\ell_0})\cup \cR\)};

\path[faint] (0,0) ellipse [x radius=1.55, y radius=0.34];
\path[faint] (0,-3.10) ellipse [x radius=1.55, y radius=0.34];

\node[slab,anchor=east] at (-1.92,0) {$\cA_{\ell_0}$};
\node[slab,anchor=east] at (-1.92,-3.10) {$\cA_{\ell_1}$};

\node[latent] (k1) at (-1.00,1.00) {};
\node[latent] (k2) at ( 1.00,1.00) {};

\node[latent] (a1) at (-1.00,0) {};
\node[latent] (a2) at ( 0.00,0) {};
\node[latent] (a3) at ( 1.00,0) {};

\node[latent] (v1) at (-0.45,-1.05) {};
\node[latent] (v2) at ( 0.35,-2.05) {};

\node[latent] (b1) at (-1.00,-3.10) {};
\node[latent] (b2) at ( 0.00,-3.10) {};
\node[latent] (b3) at ( 1.00,-3.10) {};

\draw[intset] (-0.05,-1.55) ellipse [x radius=0.82, y radius=1.05];
\node[slab] at (-1.18,-2.72) {$\operatorname{int}(\gamma)$};

\draw[route] (a2) -- (v1);
\draw[route] (v1) -- (v2);
\draw[route] (v2) -- (b2);

\node[slab] at (-0.92,-0.58) {$\gamma$};

\draw[side] (k1) -- (v1);
\draw[side] (k2) -- (v2);

\end{scope}

\begin{scope}[shift={(5.55,0)}]

\node[lab] at (0,2.10) {(b) disjoint routes};

\path[faint] (0,0) ellipse [x radius=1.65, y radius=0.34];
\path[faint] (0,-3.10) ellipse [x radius=1.65, y radius=0.34];

\node[slab,anchor=east] at (-1.98,0) {$\cA_{\ell_0}$};
\node[slab,anchor=east] at (-1.98,-3.10) {$\cA_{\ell_1}$};

\node[latent] (c1) at (-1.10,0) {};
\node[latent] (c2) at ( 0.00,0) {};
\node[latent] (c3) at ( 1.10,0) {};

\node[latent] (p1) at (-1.30,-1.05) {};
\node[latent] (p2) at (-0.95,-2.05) {};

\node[latent] (q1) at ( 1.30,-1.05) {};
\node[latent] (q2) at ( 0.95,-2.05) {};

\node[latent] (d1) at (-1.10,-3.10) {};
\node[latent] (d2) at ( 0.00,-3.10) {};
\node[latent] (d3) at ( 1.10,-3.10) {};

\draw[intset] (-1.12,-1.55) ellipse [x radius=0.55, y radius=1.05];
\draw[intset] ( 1.12,-1.55) ellipse [x radius=0.55, y radius=1.05];

\draw[route] (c1) -- (p1);
\draw[route] (p1) -- (p2);
\draw[route] (p2) -- (d1);

\draw[routealt] (c3) -- (q1);
\draw[routealt] (q1) -- (q2);
\draw[routealt] (q2) -- (d3);

\node[slab] at (-1.78,-0.62) {$\gamma_1$};
\node[slab] at ( 1.78,-0.62) {$\gamma_2$};

\node[slab] at (-1.78,-2.72) {$\operatorname{int}(\gamma_1)$};
\node[slab] at ( 1.78,-2.72) {$\operatorname{int}(\gamma_2)$};

\end{scope}

\begin{scope}[shift={(11.10,0)}]

\node[lab] at (0,2.10) {(c) inadmissible route};

\path[faint] (0,0) ellipse [x radius=1.65, y radius=0.34];
\path[faint] (0,-3.10) ellipse [x radius=1.65, y radius=0.34];

\node[slab,anchor=east] at (-1.98,0) {$\cA_{\ell_0}$};
\node[slab,anchor=east] at (-1.98,-3.10) {$\cA_{\ell_1}$};

\node[latent] (e1) at (-1.10,0) {};
\node[latent] (e2) at ( 0.00,0) {};
\node[latent] (e3) at ( 1.10,0) {};

\node[latent] (r1) at (-0.45,-1.05) {};
\node[latent] (r2) at ( 0.25,-2.05) {};

\node[badnode] (u1) at ( 1.20,-1.10) {};
\node[slab,anchor=west] at (1.37,-1.10) {$u_1$};

\node[latent] (f1) at (-1.10,-3.10) {};
\node[latent] (f2) at ( 0.00,-3.10) {};
\node[latent] (f3) at ( 1.10,-3.10) {};

\draw[intset] (-0.05,-1.55) ellipse [x radius=0.78, y radius=1.05];
\node[slab] at (-1.14,-2.72) {$\operatorname{int}(\gamma)$};

\draw[route] (e2) -- (r1);
\draw[route] (r1) -- (r2);
\draw[route] (r2) -- (f2);

\node[slab] at (-0.92,-0.58) {$\gamma$};

\draw[side] (e3) -- (u1);
\draw[bad] (u1) -- (r2);

\end{scope}

\end{tikzpicture}}
\caption{Route conventions below an antichain. In panel (a), the thick path is an
admissible route \(\gamma\) from \(\cA_{\ell_0}\) to \(\cA_{\ell_1}\); its
off-route parents are roots or belong to
\(\overline{\Anc}(\cA_{\ell_0})\). In panel (b), two admissible routes are
disjoint below \(\cA_{\ell_0}\), since their interiors are disjoint. In panel
(c), the displayed path is inadmissible: the highlighted node \(u_1\) is an
additional parent of a route node, but it lies below \(\cA_{\ell_0}\) and is
therefore not available at the source antichain.}
\label{fig:route-conventions}
\end{figure}
\section{Prior non-collapse from separating nodes}
\label{app:skip-theory}
This appendix uses concepts introduced in Appendix~\ref{app:graph-preliminaries} to establish Theorem~\ref{thm:separating-noncollapse}.
We first verify that finite DAGs admit progressive
antichain decompositions, so depth can be represented by successive
antichain cross-sections. This allows us to define a rigorous notion of depth across the DAG.  We then prove the key probabilistic lemma,
which gives a uniform lower bound on the conditional probability that a
separating node preserves a non-trivial two-case contrast, and combine
this lemma with a martingale averaging argument to prove
Theorem~\ref{thm:separating-noncollapse}.  After the main proof, we
relax the uniform per-antichain assumption to sparse families of
separating nodes, obtaining both an almost-sure non-collapse statement
and a quantitative lower-frequency bound.

The second part of the appendix identifies concrete mechanisms that
produce separating nodes.  We give sufficient conditions based on
root-retaining additive and ANOVA-type kernel decompositions, including
the multi-fidelity kernels used in deep graphical multi-fidelity GPs.
We then specialise the general result to single-node antichains and use
it to formalise the input-connection mechanism for DGP
chains.  In particular, Corollary~\ref{cor:dunlop-skip-noncollapse}
proves, for the squared-exponential input-connected chain considered in
\citep[Remark~5(3)]{dunlop2018}, that the two-input contrast does not
collapse almost surely.  Finally, we show that exactly observed internal
nodes can play the same anchoring role under the conditional prior: once
conditioned on, their states act as deterministic coordinates for
downstream kernels and yield the same non-collapse conclusions.
\subsection{Progressive antichain decomposition for DAGs}
\label{app:antichain-decomposition}

We begin with a structural fact on finite DAGs that will be used
extensively:  Endowing the vertex set with the
reachability order turns a finite DAG into a finite poset.  Following
the level-decomposition logic underlying Mirsky's theorem~\citep{mirsky1971},
yields a partition of the vertex set into progressive antichains.

\begin{proposition}[Finite DAGs admit progressive antichain decompositions]
\label{prop:canonical-antichains}
Let \(\cG=(\cV,\cE)\) be a finite DAG. Then there exist an integer
\(h\ge 1\) and non-empty antichains
\[
\cA_0,\cA_1,\dots,\cA_{h-1}\subseteq \cV
\]
such that
\[
\cV=\bigsqcup_{\ell=0}^{h-1}\cA_\ell
\qquad\text{and}\qquad
\cA_\ell\subseteq \Anc(\cA_{\ell+1}),
\quad \ell=0,\dots,h-2.
\]
\end{proposition}

\begin{proof}
For each \(v\in\cV\), define
\begin{equation*}
  d(v):=
  \max\Bigl\{k\ge 0:
  \exists\, v_1,\dots,v_k\in\cV
  \text{ such that }
  v\sqsubset v_1\sqsubset\cdots\sqsubset v_k
  \Bigr\}.
\end{equation*}
Thus \(d(v)\) is the maximum number of strict comparability steps in a
chain starting at \(v\).  Since \(\cV\) is finite, \(d(v)\) is well defined
for every~\(v\). Let
\begin{equation*}
  h:=1+\max_{v\in\cV} d(v),
  \qquad
  \cA_\ell:=\{v\in\cV:d(v)=h-1-\ell\},
  \quad \ell=0,\dots,h-1.
\end{equation*}
Then \((\cA_\ell)_{\ell=0}^{h-1}\) is a partition of~\(\cV\).

We next show that each \(\cA_\ell\) is non-empty. Let
\begin{equation*}
  v_0\sqsubset v_1\sqsubset\cdots\sqsubset v_{h-1}
\end{equation*}
be a chain of maximum length \(h-1\). For each \(j=0,\dots,h-1\), there
exists a tail chain
\begin{equation*}
  v_j\sqsubset v_{j+1}\sqsubset\cdots\sqsubset v_{h-1}
\end{equation*}
of length \(h-1-j\), so \(d(v_j)\ge h-1-j\). Conversely, if
\(d(v_j)\ge h-j\), then there would exist a chain of length at least
\(h-j\) starting at~\(v_j\), and prepending
\begin{equation*}
  v_0\sqsubset\cdots\sqsubset v_{j-1}\sqsubset v_j
\end{equation*}
would yield a chain starting at \(v_0\) of length at least \(h\),
contradicting the maximality of \(h-1\). Hence
\begin{equation*}
  d(v_j)=h-1-j,
  \qquad j=0,\dots,h-1.
\end{equation*}
Therefore every value in \(\{0,\dots,h-1\}\) is attained by some
\(d(v)\), so every \(\cA_\ell\) is non-empty.

We claim that each \(\cA_\ell\) is an antichain. Indeed, suppose
\(u,v\in\cA_\ell\) and \(u\sqsubset v\). Then any chain of length
\(d(v)\) starting at \(v\) can be prepended by~\(u\), yielding a chain
of length at least \(d(v)+1\) starting at~\(u\). Hence
\begin{equation*}
  d(u)\ge d(v)+1,
\end{equation*}
contradicting \(d(u)=d(v)=h-1-\ell\). Thus distinct vertices in
\(\cA_\ell\) are incomparable, so \(\cA_\ell\) is an antichain.

It remains to prove the progressive property. Fix
\(\ell\in\{0,\dots,h-2\}\) and let \(v\in\cA_\ell\). Then
\begin{equation*}
  d(v)=h-1-\ell\ge 1.
\end{equation*}
By definition of \(d(v)\), there exists a chain
\begin{equation*}
  v=v_0\sqsubset v_1\sqsubset\cdots\sqsubset v_{d(v)}
\end{equation*}
of length \(d(v)\) starting at~\(v\). For \(w:=v_1\), the tail
\begin{equation*}
  w=v_1\sqsubset v_2\sqsubset\cdots\sqsubset v_{d(v)}
\end{equation*}
shows that \(d(w)\ge d(v)-1\). Conversely, if \(d(w)\ge d(v)\), then
prepending \(v\) would yield a chain of length at least \(d(v)+1\)
starting at~\(v\), contradicting the definition of \(d(v)\). Therefore
\begin{equation*}
  d(w)=d(v)-1=h-1-\ell-1 = h-1-(\ell+1),
\end{equation*}
so \(w\in\cA_{\ell+1}\). Since \(v\sqsubset w\), it follows that
\(v\in\Anc(w)\subseteq \Anc(\cA_{\ell+1})\).

As \(v\in\cA_\ell\) was arbitrary, we conclude that

\begin{equation*}
  \cA_\ell\subseteq \Anc(\cA_{\ell+1}),
  \qquad \ell=0,\dots,h-2.
\end{equation*}
This proves the result.
\end{proof}

\begin{remark}
\label{rem:finite-antichain-length}
Proposition~\ref{prop:canonical-antichains} provides a progressive
antichain decomposition for every finite DAG. In particular, for a fixed
finite DAG, every non-empty progressive antichain family is necessarily
finite. Accordingly, asymptotic statements in Section~\ref{sec:prior-noncollapse-separating} should
be interpreted over increasing sequences of graphs whose common components coincide.
\end{remark}

\subsection{A conditional non-degeneracy bound}
\label{app:antichain-bound}

The proofs of Theorem~\ref{thm:separating-noncollapse} and its
sparse generalisation both rest on the same one-step conditional bound, which
we state and prove here as a self-contained lemma.

Throughout this subsection, \(\PP\) denotes the probability measure induced
by the DAG-DGP prior of Section~\ref{sec:dagdgp}, and \(\EE\) the
corresponding expectation.  We use the strict-ancestor convention, the
contrast notation \(\Delta_w\), the maximum \(M_\ell\), the contrast
covariance \(\Gamma_w(a,b)\), and the module filtration \(\cF_\ell\) from
Appendix~\ref{app:standing-probabilistic-notation}.

\begin{lemma}[Conditional antichain bound]
\label{lem:antichain-conditional-bound}
Under the DAG-DGP prior of Section~\ref{sec:dagdgp}, fix distinct cases
\(a\neq b\in[n]\) and a progressive antichain sequence
\((\cA_\ell)_{\ell\ge 0}\). Then \((\cF_\ell)_{\ell\ge 0}\) is a
filtration. Moreover, for every \(\ell\ge 0\) and every \(\varepsilon>0\), if
\(\cG_\ell\subseteq\cA_\ell\) is a deterministic set of nodes that are
\(v_\star\)-separating for \((a,b)\) with common constant \(v_\star>0\), then
\begin{equation}
  \PP\bigl(M_\ell\le\varepsilon\mid\cF_\ell\bigr)
  \;\le\;
  (1-p_\varepsilon)^{|\cG_\ell|}
  \qquad \PP\text{-almost surely},
  \label{eq:lem-antichain-bound}
\end{equation}
where
\begin{equation}
  p_\varepsilon:=2\,\Phi_{\mathrm N}\!\left(
  -\varepsilon/\sqrt{v_\star}\right)\in(0,1).
  \label{eq:p-epsilon-def}
\end{equation}
Furthermore, \(\mathbf{1}\{M_\ell>\varepsilon\}\) is
\(\cF_{\ell+1}\)-measurable.
\end{lemma}

\begin{proof}
Let \(u\in\Anc(\cA_\ell)\). Then there exists \(v\in\cA_\ell\) such that
\(u\) is a strict ancestor of~\(v\). Since
\(v\in\cA_\ell\subseteq\Anc(\cA_{\ell+1})\), there exists
\(t\in\cA_{\ell+1}\) such that \(v\) is a strict ancestor of~\(t\). By
transitivity of the strict ancestor relation, \(u\) is a strict ancestor
of~\(t\), hence \(u\in\Anc(\cA_{\ell+1})\). Therefore
\begin{equation*}
  \Anc(\cA_\ell)\subseteq\Anc(\cA_{\ell+1}).
\end{equation*}
Consequently \(\cF_\ell\subseteq\cF_{\ell+1}\), so
\((\cF_\ell)_{\ell\ge0}\) is increasing and forms a filtration.

We next establish a measurability fact that will be used repeatedly: for
every node \(v\in\cV\) and every case \(i\in[n]\), the state
\(F_v^{(i)}\) is measurable with respect to
\begin{equation}
  \sigma\!\left(f_u:
  u\in\Anc(v)\cup\{v\}
  \right).
  \label{eq:state-measurable-modules}
\end{equation}
If \(v\in\cR\), then \(F_v^{(i)}=\bx_v^{(i)}\) is deterministic, so the
claim is immediate. If \(v\in\cU\), choose any topological ordering of the
DAG and argue by induction along that ordering. For each parent
\(p\in\PA(v)\), the induction hypothesis gives that \(F_p^{(i)}\) is
measurable with respect to
\begin{equation*}
  \sigma\!\left(f_u:
  u\in\Anc(p)\cup\{p\}
  \right).
\end{equation*}
Since every ancestor of \(p\) is also an ancestor of~\(v\), and \(p\) itself
is a strict ancestor of \(v\), one has
\begin{equation*}
  \Anc(p)\cup\{p\}\subseteq\Anc(v).
\end{equation*}
Hence each parent state \(F_p^{(i)}\) is measurable with respect to
\(\sigma(f_u:u\in\Anc(v))\). Because
\begin{equation*}
  F_v^{(i)}
  =
  f_v\!\left(\{F_p^{(i)}\}_{p\in\PA(v)}\right),
\end{equation*}
it follows that \(F_v^{(i)}\) is measurable with respect to
\eqref{eq:state-measurable-modules}, as claimed.

As a first consequence, we verify that \(M_\ell\) is
\(\cF_{\ell+1}\)-measurable. Let \(v\in\cA_\ell\) and \(i\in[n]\). Since
\(v\in\cA_\ell\subseteq\Anc(\cA_{\ell+1})\), the node \(v\) is a strict
ancestor of~\(\cA_{\ell+1}\). Also, every strict ancestor of~\(v\) is a
strict ancestor of \(\cA_{\ell+1}\), so
\begin{equation*}
  \Anc(v)\cup\{v\}
  \subseteq
  \Anc(\cA_{\ell+1}).
\end{equation*}
By \eqref{eq:state-measurable-modules}, \(F_v^{(i)}\) is therefore
\(\cF_{\ell+1}\)-measurable. Since this holds for every \(v\in\cA_\ell\),
the random variable \(M_\ell\) is \(\cF_{\ell+1}\)-measurable, and so is
\(\mathbf{1}\{M_\ell>\varepsilon\}\).
We now turn to the conditional bound. Fix \(\ell\ge0\) and
\(w\in\cG_\ell\subseteq\cA_\ell\). Recall
\[
  U_{w,a}:=(F_p^{(a)})_{p\in\PA(w)},
  \qquad
  U_{w,b}:=(F_p^{(b)})_{p\in\PA(w)} .
\]
If \(p\in\PA(w)\), then \(p\) is a strict ancestor of~\(w\), and since
\(w\in\cA_\ell\), it follows that \(p\in\Anc(\cA_\ell)\). Moreover, every
ancestor of~\(p\) is also a strict ancestor of~\(w\), hence of~\(\cA_\ell\),
so
\[
  \Anc(p)\cup\{p\}
  \subseteq
  \Anc(\cA_\ell).
\]
By \eqref{eq:state-measurable-modules}, \(F_p^{(a)}\) and \(F_p^{(b)}\)
are therefore \(\cF_\ell\)-measurable. Root coordinates are deterministic
by construction. Hence the parent-input configurations
\(U_{w,a},U_{w,b}\) are \(\cF_\ell\)-measurable.

Since the GP modules are mutually independent under \(\PP\), and since
\(w\notin\Anc(\cA_\ell)\) by the antichain property of \(\cA_\ell\), the
fresh module \(f_w\) is independent of \(\cF_\ell\). Hence, under the
conditional law given \(\cF_\ell\), the parent inputs of \(f_w\) are fixed
at \(U_{w,a},U_{w,b}\), while the remaining randomness comes only from
\(f_w\). Therefore \(\Delta_w\)
is conditionally centred Gaussian with variance
\[
  V_w
  :=
  \operatorname{Cov}_{\PP}(\Delta_w\mid\cF_\ell)
  =
  \Gamma_w(a,b).
\]
The last equality follows from the definition of \(\Gamma_w(a,b)\) and the
\(\cF_\ell\)-measurability of \(U_{w,a},U_{w,b}\). Since \(w\) is
\(v_\star\)-separating for \((a,b)\), we have
\(
  V_w\succeq v_\star I_{d_w}
\)
($\PP\text{-almost surely}$).
Fix any deterministic unit vector \(e_w\in\R^{d_w}\).  By the
Cauchy--Schwarz inequality,
\begin{equation*}
  |e_w^\top \Delta_w|
  \le
  \|e_w\|_2\|\Delta_w\|_2
  =
  \|\Delta_w\|_2.
\end{equation*}
Therefore
\begin{equation}
  \PP\bigl(\|\Delta_w\|_2\le\varepsilon\mid\cF_\ell\bigr)
  \le
  \PP\bigl(|e_w^\top\Delta_w|\le\varepsilon\mid\cF_\ell\bigr).
  \label{eq:norm-small-projection-small}
\end{equation}
Conditional on \(\cF_\ell\),
\begin{equation}
  e_w^\top\Delta_w
  \sim
  \cN(0,\sigma_w^2),
  \qquad
  \sigma_w^2=e_w^\top V_w e_w\ge v_\star.
  \label{eq:projected-contrast-law}
\end{equation}
Let \(Z\sim\cN(0,1)\). Then
\begin{equation*}
  \PP\bigl(|e_w^\top\Delta_w|>\varepsilon\mid\cF_\ell\bigr)
  =
  \PP\bigl(|Z|>\varepsilon/\sigma_w\bigr).
\end{equation*}
Since the map \(\sigma\mapsto\PP(|Z|>\varepsilon/\sigma)\) is increasing on
\((0,\infty)\), \eqref{eq:projected-contrast-law} implies
\begin{equation}
  \PP\bigl(|e_w^\top\Delta_w|>\varepsilon\mid\cF_\ell\bigr)
  \ge
  \PP\bigl(|Z|>\varepsilon/\sqrt{v_\star}\bigr)
  =
  p_\varepsilon.
  \label{eq:single-projection-tail-lower}
\end{equation}
Combining \eqref{eq:norm-small-projection-small} and
\eqref{eq:single-projection-tail-lower} gives
\begin{equation}
  \PP\bigl(\|\Delta_w\|_2\le\varepsilon\mid\cF_\ell\bigr)
  \le
  1-p_\varepsilon.
  \label{eq:app-single-node-bound}
\end{equation}

We now pass from a single node to the whole antichain. The family
\(\{\Delta_w\}_{w\in\cG_\ell}\) is conditionally independent given
\(\cF_\ell\). Indeed, for each \(w\in\cG_\ell\), the variable \(\Delta_w\) is obtained
by evaluating the module~\(f_w\) at the \(\cF_\ell\)-measurable inputs
\(U_{w,a},U_{w,b}\) and then taking a difference, so conditional on \(\cF_\ell\) it
is a measurable function of \(f_w\) alone. Since the modules
\(\{f_w\}_{w\in\cG_\ell}\) are mutually independent under the prior and each
is independent of~\(\cF_\ell\), the conditional independence follows.

Using conditional independence and \eqref{eq:app-single-node-bound},
\begin{align*}
  \PP\bigl(M_\ell\le\varepsilon\mid\cF_\ell\bigr)
  &\le
  \PP\left(
    \bigcap_{w\in\cG_\ell}
    \{\|\Delta_w\|_2\le\varepsilon\}
    \middle|\cF_\ell
  \right)
  \nonumber
  =
  \prod_{w\in\cG_\ell}
  \PP\bigl(\|\Delta_w\|_2\le\varepsilon\mid\cF_\ell\bigr)
  \nonumber
  \le
  (1-p_\varepsilon)^{|\cG_\ell|}.
\end{align*}
This establishes \eqref{eq:lem-antichain-bound}.
\end{proof}

\subsection{Proof of Theorem~\ref{thm:separating-noncollapse}}
\label{app:proof-main}

\begin{proof}[Proof of Theorem~\ref{thm:separating-noncollapse}]
Fix distinct cases \(a\neq b\in[n]\) and an antichain sequence
\((\cA_\ell)_{\ell\ge0}\) with \(\cA_\ell\subseteq\Anc(\cA_{\ell+1})\).
Let \(p_\varepsilon\) be as in \eqref{eq:p-epsilon-def}.

By hypothesis, every antichain \(\cA_\ell\) contains at least \(s\ge1\)
nodes that are \(v_\star\)-separating for \((a,b)\).  Fix once and for all a
deterministic ordering of the vertices, and let \(\cG_\ell\) be the first
\(s\) separating nodes in \(\cA_\ell\) under this ordering. Applying
Lemma~\ref{lem:antichain-conditional-bound} gives
\begin{equation}
  \PP\bigl(M_\ell>\varepsilon\mid\cF_\ell\bigr)
  \ge
  1-(1-p_\varepsilon)^s
  \qquad\text{a.s.}
  \label{eq:app-main-success}
\end{equation}

Define $D_\ell := \mathbf{1}\{M_\ell>\varepsilon\} - \mathbb{E}[\mathbf{1}\{M_\ell>\varepsilon\}|\mathcal{F}_\ell]$, which is a bounded martingale difference sequence with respect to the filtration \((\cF_{\ell+1})_{\ell\geq 0}\). Indeed, Lemma~\ref{lem:antichain-conditional-bound} tells us that the first term is $\cF_{\ell+1}$-measurable, and hence so is $D_\ell$. Also, $\mathbb{E}[D_\ell|\mathcal{F}_\ell] = 0$ and \(|D_\ell|\le1\) follow by construction.
The partial sums
\begin{equation*}
  S_m:=\sum_{\ell=0}^{m-1}D_\ell,
  \qquad m\ge1,
\end{equation*}
with \(S_0=0\), form a martingale with respect to \((\cF_m)_{m\ge1}\). Indeed,
\(S_m\) is \(\cF_m\)-measurable and
\begin{align*}
  \EE[S_{m+1}\mid\cF_m]
  &=
  \EE\!\left[
    \sum_{\ell=0}^{m}D_\ell
    \middle|\cF_m
  \right]
  \nonumber\\
  &=
  \sum_{\ell=0}^{m-1}D_\ell+\EE[D_m\mid\cF_m]
  =
  S_m.
\end{align*}
Furthermore, \(|S_m-S_{m-1}|=|D_{m-1}|\le1\) almost surely.

By the Azuma--Hoeffding inequality~\citep{azuma1967,hoeffding1963}, for every
\(\eta>0\) and every \(m\ge1\),
\begin{equation}
  \PP\bigl(|S_m-S_0|\ge\eta m\bigr)
  \le
  2\exp\!\left(-\frac{\eta^2m^2}{2m}\right)
  =
  2\exp\!\left(-\frac{\eta^2m}{2}\right).
  \label{eq:main-proof-azuma}
\end{equation}
We now apply the first Borel--Cantelli lemma. Fix \(\eta>0\) and define
\begin{equation*}
  A_m:=\{|S_m|\ge\eta m\},
  \qquad m\ge1.
\end{equation*}
From \eqref{eq:main-proof-azuma},
\begin{equation*}
  \sum_{m=1}^{\infty}\PP(A_m)
  \le
  \sum_{m=1}^{\infty}2\exp\!\left(-\frac{\eta^2m}{2}\right)
  <\infty.
\end{equation*}
Hence
\begin{equation*}
  \PP\!\left(\limsup_{m\to\infty}A_m\right)=0.
\end{equation*}
Equivalently, for this fixed \(\eta\), there exists an almost surely finite
random integer \(M_\eta\) such that
\begin{equation}
  |S_m|<\eta m
  \qquad\text{for all }m\ge M_\eta.
  \label{eq:main-proof-eventual-bound}
\end{equation}
For each fixed integer \(k\ge1\), apply
\eqref{eq:main-proof-eventual-bound} with \(\eta=1/k\). Thus there is
an event \(E_k\subseteq\Omega\) with \(\PP(E_k)=1\) such that, for every
outcome \(\omega\in E_k\), there exists an integer
\(M_k(\omega)<\infty\) satisfying
\begin{equation}
  |S_m(\omega)|<\frac{m}{k}
  \qquad\text{for all }m\ge M_k(\omega).
  \label{eq:main-proof-eventual-bound-k}
\end{equation}
Define
\begin{equation}
  E:=\bigcap_{k=1}^{\infty}E_k .
  \label{eq:main-proof-probability-one-event}
\end{equation}
Since the intersection in \eqref{eq:main-proof-probability-one-event} is
countable and each \(E_k\) has probability one,
\begin{equation}
  \PP(E)=1.
  \label{eq:main-proof-probability-one}
\end{equation}
We now prove convergence on this probability-one event. Fix an outcome
\(\omega\in E\), and let \(\delta>0\) be arbitrary. Choose an integer
\(k>1/\delta\). Since \(\omega\in E\subseteq E_k\), there exists an
integer \(M_k(\omega)<\infty\) such that
\eqref{eq:main-proof-eventual-bound-k} holds. Hence, for every
\(m\ge M_k(\omega)\),
\begin{equation*}
  \left|\frac{S_m(\omega)}{m}\right|
  <\frac{1}{k}
  <\delta .
\end{equation*}
Because \(\delta>0\) was arbitrary, this proves that, for every
\(\omega\in E\),
\begin{equation}
  \frac{S_m(\omega)}{m}\to0
  \qquad\text{as }m\to\infty.
  \label{eq:main-proof-sample-path-average-zero}
\end{equation}
Combining \eqref{eq:main-proof-probability-one} and
\eqref{eq:main-proof-sample-path-average-zero}, we obtain
\begin{equation}
  \frac{S_m}{m}\to0
  \qquad\PP\text{-almost surely.}
  \label{eq:main-proof-martingale-average-zero}
\end{equation}

Taking lower limits as \(m\to\infty\), using
\eqref{eq:main-proof-martingale-average-zero} and the lower bound
\eqref{eq:app-main-success}, gives
\begin{align*}
  \liminf_{m\to\infty}\frac{1}{m}
  \sum_{\ell=0}^{m-1}\mathbf{1}\{M_\ell>\varepsilon\}
  &=
  \liminf_{m\to\infty}\left(
    \frac{1}{m}\sum_{\ell=0}^{m-1}
    \PP\bigl(M_\ell>\varepsilon\mid\cF_\ell\bigr)
    +
    \frac{S_m}{m}
  \right)
  \nonumber\\
  &\ge
  \liminf_{m\to\infty}\frac{1}{m}\sum_{\ell=0}^{m-1}
  \bigl[1-(1-p_\varepsilon)^s\bigr]
  \nonumber\\
  &=
  1-(1-p_\varepsilon)^s
  \qquad\text{almost surely.}
\end{align*}
\end{proof}

\subsection{Sparse separating families}
\label{app:separating}

Theorem~\ref{thm:separating-noncollapse} generalises to families
of separating nodes that may be unevenly distributed across depths:

\begin{proposition}[Non-collapse from sparse separating families]
\label{prop:sparse-noncollapse}
Under the DAG-DGP prior of Section~\ref{sec:dagdgp}, fix distinct cases
\(a\neq b\in[n]\) and a progressive antichain sequence
\((\cA_\ell)_{\ell\ge0}\). Let \(\cG_\ell\subseteq\cA_\ell\) be a deterministic
set of \(v_\star\)-separating nodes at each depth~\(\ell\).
\begin{enumerate}
\item[\textup{(i)}] If \(\sum_{\ell\ge0}|\cG_\ell|=\infty\), then
\[\PP(M_\ell\to0\text{ as }\ell\to\infty)=0,\] and for every fixed
\(\varepsilon>0\),
\[
  \PP\bigl(M_\ell>\varepsilon\text{ for infinitely many }\ell\bigr)=1.
\]
\item[\textup{(ii)}] For \(s\in\N\), let
\(L_s:=\{\ell\ge0:|\cG_\ell|\ge s\}\). Then, for every \(\varepsilon>0\),
\[
\liminf_{m\to\infty}\frac{1}{m}
\sum_{\ell=0}^{m-1}\mathbf{1}\{M_\ell>\varepsilon\}
\;\ge\;
\bigl[1-(1-p_\varepsilon)^s\bigr]\,
\liminf_{m\to\infty}\frac{|L_s\cap\{0,\dots,m-1\}|}{m}
\qquad\text{a.s.},
\]
where \(p_\varepsilon\) is defined in \eqref{eq:p-epsilon-def}.
\end{enumerate}
\end{proposition}

\begin{proof}
By Lemma~\ref{lem:antichain-conditional-bound},
\begin{equation}
  \PP(M_\ell\le\varepsilon\mid\cF_\ell)
  \le
  (1-p_\varepsilon)^{|\cG_\ell|}
  \qquad\text{a.s.}
  \label{eq:app-sparse-step}
\end{equation}

For part~\textup{(i)}, fix \(\varepsilon>0\), \(n\ge0\), and \(N\ge n\), and
define
\begin{equation*}
  E_{n,N}:=\bigcap_{\ell=n}^{N}\{M_\ell\le\varepsilon\}.
\end{equation*}
Since \(M_\ell\) is \(\cF_{\ell+1}\)-measurable by
Lemma~\ref{lem:antichain-conditional-bound}, the event \(E_{n,N-1}\) is
\(\cF_N\)-measurable. Therefore, using the tower property and
\eqref{eq:app-sparse-step},
\begin{align}
  \PP(E_{n,N})
  &=
  \EE\!\left[
    \mathbf{1}_{E_{n,N-1}}
    \mathbf{1}_{\{M_N\le\varepsilon\}}
  \right]
  \nonumber\\
  &=
  \EE\!\left[
    \mathbf{1}_{E_{n,N-1}}
    \PP(M_N\le\varepsilon\mid\cF_N)
  \right]
  \le
  (1-p_\varepsilon)^{|\cG_N|}\PP(E_{n,N-1}).
  \label{eq:sparse-recursive-bound}
\end{align}
Iterating \eqref{eq:sparse-recursive-bound} from \(N\) down to \(n\) yields
\begin{equation*}
  \PP(E_{n,N})
  \le
  \prod_{\ell=n}^{N}(1-p_\varepsilon)^{|\cG_\ell|}
  =
  (1-p_\varepsilon)^{\sum_{\ell=n}^{N}|\cG_\ell|}.
\end{equation*}
Letting \(N\to\infty\) and using continuity from above gives
\begin{equation}
  \PP\bigl(M_\ell\le\varepsilon\text{ for all }\ell\ge n\bigr)
  \le
  \lim_{N\to\infty}
  (1-p_\varepsilon)^{\sum_{\ell=n}^{N}|\cG_\ell|}.
  \label{eq:sparse-tail-event-bound}
\end{equation}
If \(\sum_{\ell\ge0}|\cG_\ell|=\infty\), then for every \(n\ge0\) the tail
sum \(\sum_{\ell=n}^{N}|\cG_\ell|\to\infty\) as \(N\to\infty\), so the
right-hand side of \eqref{eq:sparse-tail-event-bound} is zero. Hence
\begin{equation*}
  \PP\bigl(M_\ell\le\varepsilon\text{ for all }\ell\ge n\bigr)=0
  \qquad\text{for every }n\ge0.
\end{equation*}
Equivalently,
\begin{equation}
  \PP\bigl(M_\ell>\varepsilon\text{ for infinitely many }\ell\bigr)=1.
  \label{eq:sparse-io}
\end{equation}
Since \eqref{eq:sparse-io} holds for any fixed \(\varepsilon>0\), it follows
in particular that \(\PP(M_\ell\to0)=0\).

For part~\textup{(ii)}, fix \(s\in\N\) and define
\begin{equation*}
  Y_\ell
  :=
  \mathbf{1}_{\{\ell\in L_s\}}\,
  \mathbf{1}\{M_\ell>\varepsilon\},
  \qquad
  D_\ell
  :=
  Y_\ell-\EE[Y_\ell\mid\cF_\ell].
\end{equation*}
Because \(L_s\) is deterministic and \(\mathbf{1}\{M_\ell>\varepsilon\}\) is
\(\cF_{\ell+1}\)-measurable, \(Y_\ell\) is \(\cF_{\ell+1}\)-measurable. Also,
\(|D_\ell|\le1\) almost surely and \(\EE[D_\ell\mid\cF_\ell]=0\). Thus the
partial sums
\begin{equation}
  S_m:=\sum_{\ell=0}^{m-1}D_\ell
  \label{eq:sparse-martingale-sum}
\end{equation}
form a martingale with respect to \((\cF_m)_{m\ge1}\) with bounded
increments.
Applying the pathwise conclusion proved in
\eqref{eq:main-proof-probability-one}--\eqref{eq:main-proof-martingale-average-zero}
to the martingale in \eqref{eq:sparse-martingale-sum}, there exists an event
\(E_{\mathrm{sparse}}\subseteq\Omega\) with
\(\PP(E_{\mathrm{sparse}})=1\) such that, for every
\(\omega\in E_{\mathrm{sparse}}\),
\begin{equation*}
  \frac{1}{m}\sum_{\ell=0}^{m-1}D_\ell(\omega)
  =
  \frac{S_m(\omega)}{m}
  \longrightarrow0
  \qquad\text{as }m\to\infty.
\end{equation*}
Equivalently,
\begin{equation}
  \frac{1}{m}\sum_{\ell=0}^{m-1}D_\ell\to0
  \qquad\PP\text{-almost surely as }m\to\infty.
  \label{eq:sparse-mds-average-zero}
\end{equation}

For every \(\ell\in L_s\), one has \(|\cG_\ell|\ge s\), so by
\eqref{eq:app-sparse-step},
\begin{equation*}
  \PP(M_\ell>\varepsilon\mid\cF_\ell)
  \ge
  1-(1-p_\varepsilon)^s
  \qquad\text{a.s.}
\end{equation*}
Therefore,
\begin{equation*}
  \EE[Y_\ell\mid\cF_\ell]
  =
  \mathbf{1}_{\{\ell\in L_s\}}
  \PP(M_\ell>\varepsilon\mid\cF_\ell)
  \ge
  \mathbf{1}_{\{\ell\in L_s\}}
  \bigl[1-(1-p_\varepsilon)^s\bigr].
\end{equation*}
Using \(\mathbf{1}\{M_\ell>\varepsilon\}\ge Y_\ell\) and the definition of
\(D_\ell\), we obtain
\begin{align*}
  \frac{1}{m}\sum_{\ell=0}^{m-1}\mathbf{1}\{M_\ell>\varepsilon\}
  &\ge
  \frac{1}{m}\sum_{\ell=0}^{m-1}Y_\ell
  \nonumber\\
  &=
  \frac{1}{m}\sum_{\ell=0}^{m-1}\EE[Y_\ell\mid\cF_\ell]
  +
  \frac{1}{m}\sum_{\ell=0}^{m-1}D_\ell
  \nonumber\\
  &\ge
  \bigl[1-(1-p_\varepsilon)^s\bigr]
  \frac{|L_s\cap\{0,\dots,m-1\}|}{m}
  +
  \frac{1}{m}\sum_{\ell=0}^{m-1}D_\ell.
\end{align*}
Taking lower limits as \(m\to\infty\) and using
\eqref{eq:sparse-mds-average-zero} proves the claim.
\end{proof}

Theorem~\ref{thm:separating-noncollapse} assumes that every antichain
\(\cA_\ell\) contains at least \(s\) separating nodes, leading to a uniform
positive lower bound on the asymptotic frequency of depths where the maximum
contrast exceeds \(\varepsilon\).

Proposition~\ref{prop:sparse-noncollapse} relaxes this
uniformity.  Part~\textup{(i)} shows that it suffices that the total number
of separating nodes across all depths is infinite: then the prior does not
collapse, in the sense that \(M_\ell\) does not converge to zero, and for
every fixed \(\varepsilon>0\), the event \(M_\ell>\varepsilon\) occurs
infinitely often almost surely. Part~\textup{(ii)} provides a quantitative
refinement: the empirical frequency of large contrasts is bounded below by
the product of the per-depth guarantee \(1-(1-p_\varepsilon)^s\) and the
lower asymptotic density of depths that contain at least \(s\) separating
nodes. Hence the non-collapse phenomenon persists even when separating nodes
appear only sporadically, as long as the set of such depths has positive
lower density. When every antichain contains at least \(s\) separating nodes,
the lower density equals \(1\) and the bound reduces to that of
Theorem~\ref{thm:separating-noncollapse}, so the latter is a special, more intuitive, case of the more
general result.

\subsection{Sufficient conditions for separation}
\label{app:sufficient-separation}

A convenient class of fusion kernels that automatically produce separating
nodes is given by those retaining a uniformly non-trivial root-only main
effect. This provides the natural link between the separating condition required for prior non-collapse and the connection of intermediate latent nodes to the roots of the DAG: the natural generalization of input connection from DGPs to DAG-DGPs.

\begin{proposition}[ANOVA root retention implies separation]
\label{prop:anova-separating}
Fix a non-root node \(w\) and split its input as \((x,z)\), where \(x\)
collects one or more root coordinates. For the pair \((a,b)\), write
\(x_a,x_b\) for the corresponding values of these root-coordinate components.
Assume that the kernel at \(w\) has the form
\[
  K_w\bigl((x,z),(x',z')\bigr)
  =
  B_w\!\left[
    \sum_{j=1}^{J_w}\beta_{w,j}\,\kappa_{w,j}^{(0)}(x,x')
    +
    \sum_{m=1}^{M_w}\alpha_{w,m}\,
    \rho_{w,m}\bigl((x,z),(x',z')\bigr)
  \right],
\]
with \(B_w\succeq\underline\sigma^2 I_{d_w}\), nonnegative weights satisfying
\[
  \sum_{j=1}^{J_w}\beta_{w,j}
  +
  \sum_{m=1}^{M_w}\alpha_{w,m}
  =1,
\]
and scalar correlation kernels \(\kappa_{w,j}^{(0)}\) and
\(\rho_{w,m}\). If
\[
  \sum_{j=1}^{J_w}\beta_{w,j}\ge\gamma_0>0
  \qquad\text{and}\qquad
  \sup_j\bigl|\kappa_{w,j}^{(0)}(x_a,x_b)\bigr|\le r<1,
\]
then \(w\) is \(v_\star\)-separating for \((a,b)\) with
\[
  v_\star
  :=
  2\,\underline\sigma^2\bigl(1-\bar r\bigr),
  \qquad
  \bar r:=(1-\gamma_0)+\gamma_0 r<1.
\]
\end{proposition}

\begin{proof}
For the pair \((a,b)\), write
\[
  u_a=(x_a,z_a),
  \qquad
  u_b=(x_b,z_b),
\]
and define
\[
  \vartheta_w
  :=
  \sum_{j=1}^{J_w}\beta_{w,j}\,\kappa_{w,j}^{(0)}(x_a,x_b)
  +
  \sum_{m=1}^{M_w}\alpha_{w,m}\,\rho_{w,m}(u_a,u_b).
\]
Since each \(\rho_{w,m}\) is a scalar correlation kernel,
\(|\rho_{w,m}(u_a,u_b)|\le1\). Using also
\(\sup_j|\kappa_{w,j}^{(0)}(x_a,x_b)|\le r\), we obtain
\begin{align*}
  |\vartheta_w|
  &\le
  \sum_{j=1}^{J_w}\beta_{w,j}\,
  |\kappa_{w,j}^{(0)}(x_a,x_b)|
  +
  \sum_{m=1}^{M_w}\alpha_{w,m}\,
  |\rho_{w,m}(u_a,u_b)|\\
  &\le
  r\sum_{j=1}^{J_w}\beta_{w,j}
  +
  \sum_{m=1}^{M_w}\alpha_{w,m}\\
  &=
  r\sum_{j=1}^{J_w}\beta_{w,j}
  +
  1-
  \sum_{j=1}^{J_w}\beta_{w,j}\\
  &=
  1-(1-r)\sum_{j=1}^{J_w}\beta_{w,j}\\
  &\le
  1-\gamma_0(1-r)
  =
  (1-\gamma_0)+\gamma_0r
  =
  \bar r.
\end{align*}
Moreover, since \(\kappa_{w,j}^{(0)}\) and \(\rho_{w,m}\) are scalar
correlation kernels, evaluating on the diagonal gives
\[
  K_w(u_a,u_a)
  =
  K_w(u_b,u_b)
  =
  B_w.
\]
The cross-covariances are
\[
  K_w(u_a,u_b)=B_w\vartheta_w,
  \qquad
  K_w(u_b,u_a)=B_w\vartheta_w.
\]
Hence
\[
  \Gamma_w(a,b)
  =
  2(1-\vartheta_w)B_w.
\]
Since \(\vartheta_w\le|\vartheta_w|\le\bar r<1\), it follows that
\[
  \Gamma_w(a,b)
  \succeq
  2(1-\bar r)B_w
  \succeq
  2\,\underline\sigma^2(1-\bar r)I_{d_w}
  =
  v_\star I_{d_w}.
\]
Thus \(w\) is \(v_\star\)-separating.
\end{proof}

The following remark shows that the same mechanism applies more broadly,
including to multi-fidelity kernels that do not need to be written as a
normalised convex combination.

\begin{remark}[Separation from an additive root-only kernel component]
\label{rem:additive-root-separation}
Suppose the kernel at a non-root node \(w\) decomposes as
\[
  K_w\bigl((x,z),(x',z')\bigr)
  =
  K_w^{(\mathrm{int})}\!\bigl((x,z),(x',z')\bigr)
  +
  K_w^{(0)}(x,x'),
\]
where \(K_w^{(\mathrm{int})}\) is a valid matrix-valued positive-semidefinite
kernel on the full input space and \(K_w^{(0)}\) is a valid matrix-valued
positive-semidefinite kernel depending only on root coordinates. For the pair
\((a,b)\), write \(x_a,x_b\) for the corresponding values of these
root-coordinate components. Since both summands are valid kernels, the
contrast covariance decomposes as
\[
  \Gamma_w(a,b)
  =
  \Gamma_w^{(\mathrm{int})}(a,b)
  +
  \Gamma_w^{(0)}(a,b),
\]
with each summand positive semidefinite. In particular,
\begin{equation*}
  \Gamma_w(a,b)
  \succeq
  \Gamma_w^{(0)}(a,b),
\end{equation*}
where
\begin{equation}
  \Gamma_w^{(0)}(a,b)
  =
  K_w^{(0)}(x_a,x_a)+K_w^{(0)}(x_b,x_b)
  -K_w^{(0)}(x_a,x_b)-K_w^{(0)}(x_b,x_a).
  \label{eq:additive-root-gamma0-general}
\end{equation}
If, in particular,
\[
  K_w^{(0)}(x,x')=\sigma_0^2\,\bar k(x,x')I_{d_w},
\]
where \(\bar k\) is a scalar correlation kernel, then
\eqref{eq:additive-root-gamma0-general} becomes
\[
  \Gamma_w^{(0)}(a,b)
  =
  2\sigma_0^2\bigl(1-\bar k(x_a,x_b)\bigr)I_{d_w}.
\]
Therefore, whenever \(\bar k(x_a,x_b)<1\),
\[
  \Gamma_w(a,b)
  \succeq
  2\sigma_0^2\bigl(1-\bar k(x_a,x_b)\bigr)I_{d_w}
  =:
  v_\star I_{d_w},
\]
with \(v_\star>0\), and \(w\) is separating.  This condition holds, for
example, for a squared-exponential root kernel whenever \(x_a\neq x_b\).

This criterion applies to the multi-fidelity deep GP kernel of
\citep{cutajar2019} and its graphical generalisation in \citep{ji2024}. In
their formulation, the kernel at each non-source node takes the form
\[
  K_w\bigl([x,z],[x',z']\bigr)
  =
  K_{\mathrm{SE},\rho}(x,x')
  \bigl[K_{\mathrm{LIN}}(z,z')+K_{\mathrm{SE}}(z,z')\bigr]
  +
  K_{\mathrm{SE},\delta}(x,x'),
\]
where \(z\) collects the parent-node outputs and \(x\) is the exogenous
input, and the subscripts SE and LIN denote the squared-exponential and
linear kernel, respectively.  The discrepancy kernel
\(K_{\mathrm{SE},\delta}(x,x')\) depends only on root coordinates and is a
squared-exponential kernel with variance \(\sigma_\delta^2\). Since
\(K_{\mathrm{SE},\delta}(x_a,x_b)<\sigma_\delta^2\) whenever \(x_a\neq x_b\),
the above argument gives
\[
  v_\star
  =
  2\sigma_\delta^2
  \left(
  1-\exp\!\left[-\sum_l(x_{a,l}-x_{b,l})^2/(2\lambda_l^2)\right]
  \right)
  >0,
\]
where the \(\lambda_l\) are the length-scale parameters of
\(K_{\mathrm{SE},\delta}\). Thus every non-source node in a graphical
multi-fidelity DGP with this kernel is separating for any pair of distinct
input cases, regardless of the parent latent states.
\end{remark}

\subsection{Recovery of the chain input-connection mechanism}
\label{app:chain-recovery}

The results of Section~\ref{app:sufficient-separation} refer directly to
the empirical evidence in the literature that connecting all latent nodes to
the input space prevents the prior-collapse pathology. This observation goes
back to \citep{duvenaud2014}, who provided the first theoretical result on prior collapse and investigated in detail the proposal, originally elaborated by \citep[Chapter~2]{neal1995} as a general recommendation for arbitrarily deep Bayesian neural networks, of connecting every latent layer of a DGP to the corresponding input. In their theoretical paper on this pathological
prior behaviour, after formalising and proving that the pathology occurs for
non-input-connected DGPs under RBF kernels, \citep[Remark~5(3)]{dunlop2018}
conjecture that the same collapse result does not hold when every layer is
input-connected, and give an intuition for this. With our general theory, we
can formalise and prove that conjectured non-collapse mechanism as a special
case.

The following corollary shows that uniform $v_\star$-separation is sufficient to avoid prior collapse in the case of chain DAGs (an unsurprising fact), whereas the subsequent corollary demonstrates that skip-connection is enough to ensure this property and hence non-collapse.

\begin{corollary}[Separating input-connected chains]
\label{cor:skip-chain}
Assume that \(\cA_\ell=\{v_\ell\}\) for every \(\ell\ge0\), and that each
\(v_\ell\) is \(v_\star\)-separating for the pair \((a,b)\), with common
constant \(v_\star>0\). Then, for every \(\varepsilon>0\),
\[
  \liminf_{m\to\infty}\frac{1}{m}
  \sum_{\ell=0}^{m-1}
  \mathbf{1}\!\left\{
  \|F_{v_\ell}^{(a)}-F_{v_\ell}^{(b)}\|_2>\varepsilon
  \right\}
  \ge
  p_\varepsilon
  \qquad\text{a.s.},
\]
where \(p_\varepsilon=2\Phi_{\mathrm N}(-\varepsilon/\sqrt{v_\star})\in(0,1)\).
Moreover,
\[
  \PP\!\left(
  \|F_{v_\ell}^{(a)}-F_{v_\ell}^{(b)}\|_2>\varepsilon
  \text{ for infinitely many }\ell
  \right)=1
  \qquad\text{for every }\varepsilon>0,
\]
and in particular
\[
  \PP\!\left(
  \|F_{v_\ell}^{(a)}-F_{v_\ell}^{(b)}\|_2\to0
  \text{ as }\ell\to\infty
  \right)=0.
\]
\end{corollary}

\begin{proof}
Set \(\cG_\ell:=\{v_\ell\}\), so that \(|\cG_\ell|=1\) for every \(\ell\).
Then \(\sum_{\ell\ge0}|\cG_\ell|=\infty\), and
Proposition~\ref{prop:sparse-noncollapse}\,\textup{(i)} gives
\[
  \PP\!\left(
  \|F_{v_\ell}^{(a)}-F_{v_\ell}^{(b)}\|_2>\varepsilon
  \text{ for infinitely many }\ell
  \right)=1
  \qquad\text{for every }\varepsilon>0.
\]
Moreover, part~\textup{(ii)} with \(s=1\) and lower density equal to~\(1\)
yields
\[
  \liminf_{m\to\infty}\frac{1}{m}
  \sum_{\ell=0}^{m-1}
  \mathbf{1}\!\left\{
  \|F_{v_\ell}^{(a)}-F_{v_\ell}^{(b)}\|_2>\varepsilon
  \right\}
  \ge
  p_\varepsilon
  \qquad\text{a.s.}
\]
Finally, if \(\|F_{v_\ell}^{(a)}-F_{v_\ell}^{(b)}\|_2\to0\) as
\(\ell\to\infty\), then for every \(\varepsilon>0\) one must have
\(\|F_{v_\ell}^{(a)}-F_{v_\ell}^{(b)}\|_2\le\varepsilon\) for all sufficiently
large~\(\ell\), which excludes the event that the norm exceeds
\(\varepsilon\) for infinitely many~\(\ell\). Since that event has probability
one, the convergence-to-zero event has probability zero.
\end{proof}

\begin{corollary}[Input-connected chain of \citep{dunlop2018}]
\label{cor:dunlop-skip-noncollapse}
Consider the input-connected chain in \citep[Remark~5(3)]{dunlop2018},
\[
  u_{n+1}(x)=\xi_{n+1}(u_n(x),x),
\]
where \(\xi_{n+1}=(\xi_{n+1}^1,\dots,\xi_{n+1}^m)\) and the scalar fields
\(\xi_{n+1}^1,\dots,\xi_{n+1}^m\) are independent copies of a centred Gaussian
process on \(\R^m\times\R^d\) with squared-exponential kernel
\[
  h\bigl((z,x),(z',x')\bigr)
  =
  \sigma^2
  \exp\!\left(
  -\frac{\|z-z'\|_2^2+\|x-x'\|_2^2}{2w^2}
  \right).
\]
Let \(x,x_0\in D\) with \(x\neq x_0\). Then, for every \(\varepsilon>0\),
\[
  \liminf_{N\to\infty}\frac{1}{N}
  \sum_{n=0}^{N-1}
  \mathbf{1}\!\left\{
  \|u_n(x)-u_n(x_0)\|_2>\varepsilon
  \right\}
  \ge
  p_\varepsilon
  \qquad\text{a.s.},
\]
where
\[
  p_\varepsilon
  =
  2\,\Phi_{\mathrm N}\!\left(
  -\frac{\varepsilon}
  {\sqrt{2\sigma^2\!\left(
    1-\exp\!\bigl(-\|x-x_0\|_2^2/(2w^2)\bigr)
  \right)}}
  \right)\in(0,1).
\]
Moreover,
\(\PP(\|u_n(x)-u_n(x_0)\|_2\to0\text{ as }n\to\infty)=0\).
\end{corollary}

\begin{proof}
This is an application of Corollary~\ref{cor:skip-chain}. The model is a
chain, so the antichains are the singletons \(\cA_n=\{v_n\}\).

Fix \(n\ge0\) and condition on \(u_n\). Set \(z:=u_n(x)\) and
\(z_0:=u_n(x_0)\). The matrix-valued kernel is
\[
  K\bigl((z,x),(z',x')\bigr)
  =
  h\bigl((z,x),(z',x')\bigr)I_m.
\]
Hence
\[
  K\bigl((z,x),(z,x)\bigr)
  =
  K\bigl((z_0,x_0),(z_0,x_0)\bigr)
  =
  \sigma^2 I_m,
\]
and
\[
  K\bigl((z,x),(z_0,x_0)\bigr)
  =
  \sigma^2
  \exp\!\left(
  -\frac{\|z-z_0\|_2^2+\|x-x_0\|_2^2}{2w^2}
  \right)I_m.
\]
Therefore the conditional contrast covariance is
\begin{align*}
  \Gamma_n(x,x_0)
  &=
  K\bigl((z,x),(z,x)\bigr)
  +K\bigl((z_0,x_0),(z_0,x_0)\bigr)\\
  &\qquad
  -K\bigl((z,x),(z_0,x_0)\bigr)
  -K\bigl((z_0,x_0),(z,x)\bigr)\\
  &=
  2\sigma^2\left[
  1-
  \exp\!\left(
  -\frac{\|z-z_0\|_2^2+\|x-x_0\|_2^2}{2w^2}
  \right)
  \right]I_m.
\end{align*}
Since \(\|z-z_0\|_2^2\ge0\), the exponential is bounded above by
\(\exp(-\|x-x_0\|_2^2/(2w^2))\), so
\[
  \Gamma_n(x,x_0)
  \succeq
  2\sigma^2\left[
  1-
  \exp\!\left(
  -\frac{\|x-x_0\|_2^2}{2w^2}
  \right)
  \right]I_m
  =:
  v_\star I_m.
\]
Since \(x\neq x_0\), one has \(v_\star>0\), so every layer is
\(v_\star\)-separating for \((x,x_0)\). Applying
Corollary~\ref{cor:skip-chain} gives both the positive lower-frequency bound
and the almost-sure non-collapse.
\end{proof}

\subsection{Conditional refresh from intermediate supervision}
\label{app:conditional-refresh}

The following corollary formalises the claim from Section~\ref{sec:prior-noncollapse-separating}
that fully observed and noiseless internal nodes stabilise the conditional
prior without any architectural modification. An observed internal
node can be used by downstream kernels in the same way that input-connected
kernels use root inputs.

Accordingly, throughout this subsection let \(\cE\subseteq\cU\) be a set of
non-root nodes that are observed without error, in the sense that
\(\cO_e=[n]\times[d_e]\) and the realised observed value is
\(\bY_e=\bF_e\) for every \(e\in\cE\). Since exact observations of continuous
latent variables should be interpreted through regular conditional laws, we
write
\begin{equation*}
  \PP_{\cE}(\,\cdot\,)
  :=
  \PP\!\left(
    \,\cdot\,\middle|\sigma(\bF_e:e\in\cE)
  \right)
  \text{ evaluated at }\{\bF_e=\bY_e:e\in\cE\}
\end{equation*}
for a chosen regular conditional version of the prior given the exactly
observed internal states. Under \(\PP_{\cE}\), each conditioned state
\(F_e^{(i)}\) is almost surely equal to the observed value \(Y_e^{(i)}\).
Hence any downstream kernel that depends on coordinates coming from nodes in
\(\cE\) sees them as fixed anchor coordinates.

\begin{corollary}[Perfect intermediate supervision as conditional refresh]
\label{cor:conditional-refresh}
Fix distinct cases \(a\neq b\in[n]\), and let
\((\cA_\ell)_{\ell\ge0}\) be antichains in \(\cU\setminus\cE\) with
\(\cA_\ell\subseteq\Anc(\cA_{\ell+1})\). Assume that no node in
\(\bigcup_{\ell\ge0}\cA_\ell\) is an ancestor of any node in~\(\cE\).

Let \(\cG_\ell\subseteq\cA_\ell\) be a deterministic family of nodes such
that, under the conditional prior \(\PP_{\cE}\), each \(w\in\cG_\ell\)
satisfies the assumptions of Proposition~\ref{prop:anova-separating}, with
anchor coordinates taken from~\(\cE\), and with a common resulting lower
bound \(v_\star>0\). Then each \(w\in\cG_\ell\) is \(v_\star\)-separating
for \((a,b)\) under \(\PP_{\cE}\), and the conclusions of
Proposition~\ref{prop:sparse-noncollapse} apply under the conditional prior
\(\PP_{\cE}\). In particular, if every antichain contains at least \(s\ge1\)
such nodes, the conclusion of Theorem~\ref{thm:separating-noncollapse}
applies under \(\PP_{\cE}\).
\end{corollary}

\begin{proof}
For every node \(w\in\bigcup_{\ell\ge0}\cG_\ell\), any input split
\((c,z)\) in which \(c\) collects coordinates from nodes in \(\cE\) becomes,
under \(\PP_{\cE}\), an input split with fixed anchor coordinates. These
coordinates therefore play the same role that root inputs play in
Proposition~\ref{prop:anova-separating}.

By assumption, each \(w\in\cG_\ell\) satisfies the hypotheses of
Proposition~\ref{prop:anova-separating} with these conditioned anchor
coordinates in place of the root coordinates. Proposition~\ref{prop:anova-separating}
therefore applies under \(\PP_{\cE}\), and the common lower-bound assumption
yields that every such node is \(v_\star\)-separating for \((a,b)\) under the
conditional prior.

It remains to verify that the conditional independence structure required by
Lemma~\ref{lem:antichain-conditional-bound} is preserved under
\(\PP_{\cE}\). Under the prior~\(\PP\), the exactly observed variables are
measurable with respect to the modules
\begin{equation}
  \mathcal M_{\cE}
  :=
  \left\{f_v:
  v\in
  \bigcup_{e\in\cE}
  \Anc(e)\cup\{e\}
  \right\}.
  \label{eq:modules-determining-exact-observations}
\end{equation}
By hypothesis, no node in \(\bigcup_{\ell\ge0}\cA_\ell\) is an ancestor of
any node in~\(\cE\). Since every \(w\in\cA_\ell\) is a non-root node and
\(w\notin\cE\), it follows that
\begin{equation}
  w\notin
  \bigcup_{e\in\cE}(\Anc(e)\cup\{e\}).
  \label{eq:w-not-in-observed-modules}
\end{equation}
Thus \(f_w\) is one of the prior-independent GP modules outside
\(\mathcal M_{\cE}\).

Now consider the filtration \((\cF_\ell)_{\ell\ge0}\) from
\eqref{eq:app-module-filtration-def} and a separating node
\(w\in\cG_\ell\subseteq\cA_\ell\). Under \(\PP\), the module \(f_w\) is
jointly independent of the sigma-field generated by
\begin{equation*}
  \{f_v:v\in\Anc(\cA_\ell)\}\cup\mathcal M_{\cE}.
\end{equation*}
Indeed, \(w\notin\Anc(\cA_\ell)\) by the antichain argument in the proof of
Lemma~\ref{lem:antichain-conditional-bound}, and \(w\) is not contained in
the module index set \eqref{eq:modules-determining-exact-observations} by
\eqref{eq:w-not-in-observed-modules}. Mutual independence of the GP modules
therefore gives the joint independence.

Consequently, after conditioning on the realised values of
\(\{\bF_e:e\in\cE\}\), the module \(f_w\) remains independent of
\(\cF_\ell\) and retains its prior GP law. Similarly, for distinct
\(w,w'\in\cG_\ell\), the modules \(f_w\) and \(f_{w'}\) remain conditionally
independent given \(\cF_\ell\) under \(\PP_{\cE}\). Therefore every argument
in the proof of Lemma~\ref{lem:antichain-conditional-bound} applies under
\(\PP_{\cE}\) without modification. Hence
Proposition~\ref{prop:sparse-noncollapse} holds under the conditional prior,
and Theorem~\ref{thm:separating-noncollapse} holds under \(\PP_{\cE}\) when
its uniform per-antichain assumption is satisfied.
\end{proof}

\section{Indegree and outdegree effects}
\label{app:degree-effects}
This section isolates how local graph degrees affect two-case contrasts when
the fusion kernel is radial in the concatenated parent state. The indegree
analysis shows that, in product-type radial blocks, aggregating many parents
can attenuate expected squared contrasts and yields a local contraction
criterion. The outdegree analysis gives the complementary mechanism: if a node
has sufficiently many disjoint designated children, branching can sustain
threshold-size contrasts with positive probability. These results are local to
the analysed block and do not alter the separating-node non-collapse criterion
above; additive root-retaining components remain separating regardless of the
number of additional parents.

We continue with the notation of Appendix~\ref{app:graph-preliminaries}; in
particular, \(\Delta_w=F_w^{(a)}-F_w^{(b)}\) for a fixed pair of distinct
cases \(a\neq b\in[n]\).

\subsection{Layered radial blocks}
\label{app:degree-assumptions}

To isolate the effect of indegree, we work on a layered block of the DAG in
which all parents of a node in the current antichain lie in the immediately
preceding antichain, and the nodewise fusion kernel is radial in the
concatenated parent state.  This is a local structural assumption on the
block being analysed, not a global restriction on the whole DAG.

\begin{assumption}[Layered Laplace--radial block]
\label{ass:laplace-radial-degree}
Let \((\cA_\ell)_{\ell\ge0}\) be non-root antichains such that
\(\PA(w)\subseteq\cA_{\ell-1}\) for every \(w\in\cA_\ell\) and every
\(\ell\ge1\). Assume that all nodes in \(\bigcup_{\ell\ge0}\cA_\ell\) have the
same output dimension \(d\), and that there exist \(\tau>0\) and a probability
measure \(\nu\) on \([0,\infty)\) such that
\[
  K_w(u,u')
  =
  \tau^2 \kappa(\|u-u'\|_2^2)I_d,
  \qquad
  \kappa(r)
  =
  \int_0^\infty e^{-sr}\,\nu(ds),
\]
for every node \(w\in\bigcup_{\ell\ge0}\cA_\ell\).
\end{assumption}

A visualisation of such a layered block is given in
Figure~\ref{fig:layered-radial-block}.

\begin{figure}[H]
\centering
\begin{tikzpicture}[
  x=1.35cm,
  y=1.10cm,
  >=latex,
  latent/.style={
    circle,
    draw=black,
    fill=white,
    line width=0.45pt,
    inner sep=1.2pt,
    minimum size=4.9mm
  },
  focus/.style={
    circle,
    draw=black,
    fill=white,
    line width=0.9pt,
    inner sep=1.2pt,
    minimum size=5.3mm
  },
  dag/.style={
    ->,
    semithick,
    draw=black
  },
  focusdag/.style={
    ->,
    thick,
    draw=black
  },
  omit/.style={
    ->,
    dashed,
    draw=black!45,
    line width=0.45pt
  },
  faint/.style={
    draw=black!40,
    fill=black!3,
    line width=0.5pt
  },
  lab/.style={font=\small},
  slab/.style={font=\scriptsize}
]

\path[faint] (0.20, 0.00) ellipse [x radius=3.00, y radius=0.36];
\path[faint] (0.00,-1.85) ellipse [x radius=2.30, y radius=0.36];
\path[faint] (0.00,-3.70) ellipse [x radius=2.55, y radius=0.36];

\node[lab,anchor=east] at (-3.35, 0.00) {$\cA_{\ell-1}$};
\node[lab,anchor=east] at (-3.35,-1.85) {$\cA_{\ell}$};
\node[lab,anchor=east] at (-3.35,-3.70) {$\cA_{\ell+1}$};

\node[latent] (u1) at (-2.20, 0.00) {};
\node[latent] (u2) at (-1.10, 0.00) {};
\node[latent] (u3) at ( 0.00, 0.00) {};
\node[latent] (u4) at ( 1.10, 0.00) {};
\node[latent] (u5) at ( 2.20, 0.00) {};

\node[latent] (m1) at (-1.55,-1.85) {};
\node[focus]  (w)  at ( 0.00,-1.85) {$w$};
\node[latent] (m3) at ( 1.55,-1.85) {};

\node[latent] (v1) at (-2.10,-3.70) {};
\node[latent] (v2) at (-0.70,-3.70) {};
\node[latent] (v3) at ( 0.70,-3.70) {};
\node[latent] (v4) at ( 2.10,-3.70) {};

\draw[dag]      (u1) -- (m1);
\draw[dag]      (u2) -- (m1);

\draw[focusdag] (u2) -- (w);
\draw[focusdag] (u3) -- (w);
\draw[focusdag] (u4) -- (w);

\draw[dag]      (u4) -- (m3);
\draw[dag]      (u5) -- (m3);

\draw[dag]      (m1) -- (v1);
\draw[dag]      (m1) -- (v2);

\draw[focusdag] (w)  -- (v2);
\draw[focusdag] (w)  -- (v3);

\draw[dag]      (m3) -- (v3);
\draw[dag]      (m3) -- (v4);

\draw[omit] (-2.20, 0.70) -- (u1);
\draw[omit] ( 0.00, 0.70) -- (u3);
\draw[omit] ( 2.20, 0.70) -- (u5);

\draw[omit] (v1) -- (-2.10,-4.40);
\draw[omit] (v2) -- (-0.70,-4.40);
\draw[omit] (v3) -- ( 0.70,-4.40);
\draw[omit] (v4) -- ( 2.10,-4.40);

\end{tikzpicture}
\caption{Layered block as in Assumption~\ref{ass:laplace-radial-degree}. Each row is an antichain, and all displayed edges in the analysed block run from one layer to the next. The highlighted node \(w\in\cA_\ell\) illustrates the local indegree mechanism: all its parents lie in \(\cA_{\ell-1}\), and its descendants lie in \(\cA_{\ell+1}\). Dashed arrows indicate omitted portions of the surrounding DAG above and below the displayed block.}
\label{fig:layered-radial-block}
\end{figure}

By Schoenberg's theorem~\citep{schoenberg1938}, a continuous function
\(\kappa\colon[0,\infty)\to\R\) with \(\kappa(0)=1\) yields a radial kernel
\((x,x')\mapsto\kappa(\|x-x'\|_2^2)\) that is positive definite on \(\R^m\) for
every \(m\) if and only if \(\kappa\) is completely monotone. By the
Bernstein--Widder theorem~\citep{widder1940laplace}, this is equivalent to the
Laplace-transform representation
\[
  \kappa(r)=\int_0^\infty e^{-sr}\,\nu(ds)
\]
for a probability measure \(\nu\) on \([0,\infty)\). Therefore
Assumption~\ref{ass:laplace-radial-degree} covers isotropic radial kernels
that are valid in every ambient dimension, including the squared-exponential
and rational-quadratic families. In particular, \(\kappa\) is nonincreasing,
takes values in \([0,1]\), with \(\kappa(0)=1\). The homogeneity
across nodes is imposed only for notational simplicity; a layer-dependent
version is stated in Remark~\ref{rem:layer-dependent-dimensions} below.

A relevant special case is product fusion of squared-exponential parent
kernels. If
\[
  K_w^{(p)}(u_p,u_p')
  =
  \exp\!\left(
    -\frac{\|u_p-u_p'\|_2^2}{2\lambda^2}
  \right)
\]
for every \(p\in\PA(w)\), and the fused kernel is normalised so that
\(K_w(u,u)=\tau^2I_d\), then
\[
  K_w(u,u')
  =
  \tau^2
  \prod_{p\in\PA(w)}
  \exp\!\left(
    -\frac{\|u_p-u_p'\|_2^2}{2\lambda^2}
  \right) I_d
  =
  \tau^2
  \exp\!\left(
    -\frac{\|u-u'\|_2^2}{2\lambda^2}
  \right) I_d,
\]
where \(u=(u_p)_{p\in\PA(w)}\) denotes the concatenated parent state. Hence
Assumption~\ref{ass:laplace-radial-degree} holds with
\(\nu=\delta_{1/(2\lambda^2)}\).

For each non-root node \(w\) we define the squared parent contrast
\begin{equation}
  \varrho_w^2
  =
  \sum_{p\in\PA(w)}\|\Delta_p\|_2^2,
  \label{eq:parent-contrast-def}
\end{equation}
which is the squared Euclidean distance between the concatenated parent
states evaluated at cases \(a\) and~\(b\).

\begin{lemma}[Two-point law for a single radial module]
\label{lem:single-radial-two-point}
Let \(w\in\cU\) be a non-root node whose kernel has the form
\(K_w(u,u')=\tau^2\kappa(\|u-u'\|_2^2)I_d\), and let \(\varrho_w^2\) be the
squared parent contrast defined in \eqref{eq:parent-contrast-def}. Conditional
on the parent states of~\(w\),
\[
  \Delta_w
  \sim
  \mathrm N\!\left(
    0,
    2\tau^2\bigl(1-\kappa(\varrho_w^2)\bigr)I_d
  \right).
\]
In particular,
\begin{equation}
  \EE\!\left[
  \|\Delta_w\|_2^2
  \mid
  \{F_p^{(a)},F_p^{(b)}\}_{p\in\PA(w)}
  \right]
  =
  2d\tau^2\bigl(1-\kappa(\varrho_w^2)\bigr).
  \label{eq:radial-conditional-second-moment}
\end{equation}
\end{lemma}

\begin{proof}
By the conditional Gaussianity of \(\Delta_w\) established in the proof of
Lemma~\ref{lem:antichain-conditional-bound}, the contrast \(\Delta_w\) is
conditionally centred Gaussian with covariance \(\Gamma_w(a,b)\) given the
parent states. Under the radial kernel assumption, the diagonal evaluations
give \(K_w(u_a,u_a)=K_w(u_b,u_b)=\tau^2I_d\), and the cross-evaluations give
\(K_w(u_a,u_b)=K_w(u_b,u_a)=\tau^2\kappa(\varrho_w^2)I_d\). Hence
\[
  \Gamma_w(a,b)
  =
  2\tau^2\bigl(1-\kappa(\varrho_w^2)\bigr)I_d.
\]
For the second-moment identity, write
\(\Sigma_w=2\tau^2(1-\kappa(\varrho_w^2))I_d\) for the conditional
covariance. Since \(\Delta_w\) is conditionally centred,
\begin{align*}
\EE\!\left[
  \|\Delta_w\|_2^2
  \mid
  \{F_p^{(a)},F_p^{(b)}\}_{p\in\PA(w)}
\right]
&=
\EE\!\left[
  \Delta_w^\top\Delta_w
  \mid
  \{F_p^{(a)},F_p^{(b)}\}_{p\in\PA(w)}
\right]\\
&=
\tr(\Sigma_w)\\
&=
2d\tau^2\bigl(1-\kappa(\varrho_w^2)\bigr),
\end{align*}
via the standard identity
\(\EE[Z^\top Z]=\tr(\Var(Z))\) for any centred random vector~\(Z\).
\end{proof}

\subsection{Indegree effects}
\label{app:degree-proofs}

Throughout this subsection, recall from Section~\ref{sec:degree-effects} that
\(k_\ell=\max_{w\in\cA_\ell}|\PA(w)|\) denotes the maximal indegree on the
\(\ell\)-th antichain and
\(C_\ell=\max_{w\in\cA_\ell}\EE[\|\Delta_w\|_2^2]\) the maximal expected
squared contrast.

The following proposition is the main result of this section and it shows how indegree relates to the passage of contrasts through the graph.

\begin{proposition}[Indegree controls the contrast recursion]
\label{prop:indegree}
Under Assumption~\ref{ass:laplace-radial-degree}, for every $\ell\ge1$, with $\tau^2$ the marginal variance and $\nu$ the Schoenberg mixing measure,
\[
  C_\ell \;\le\; \Psi_{k_\ell}(C_{\ell-1}),
  \qquad
  \Psi_k(u)=
  2d\tau^2\!\left[
    1-\int_0^\infty
    \left(1+\frac{2s}{d}\,u\right)^{\!-dk/2}\!\nu(ds)
  \right],
\]
Moreover, for every fixed $u\ge0$, the map
$k\mapsto\Psi_k(u)$ is nondecreasing.
\end{proposition}
\begin{proof}
Fix \(\ell\ge1\) and \(w\in\cA_\ell\). Enumerate the parents of \(w\) as
\(\PA(w)=\{p_1,\dots,p_m\}\), where \(m=|\PA(w)|\le k_\ell\). By
Lemma~\ref{lem:single-radial-two-point},
\begin{equation*}
  \EE[\|\Delta_w\|_2^2]
  =
  2d\tau^2\Bigl(1-\EE[\kappa(\varrho_w^2)]\Bigr).
\end{equation*}
Using the Laplace-transform representation of \(\kappa\) and Tonelli's
theorem, since the integrand is nonnegative, gives
\begin{equation}
  \EE[\|\Delta_w\|_2^2]
  =
  2d\tau^2\left[
    1-
    \int_0^\infty
    \EE[e^{-s\varrho_w^2}]\,
    \nu(ds)
  \right].
  \label{eq:indegree-proof-laplace-start}
\end{equation}
We therefore seek a lower bound on \(\EE[e^{-s\varrho_w^2}]\).

Let \(\mathscr G_w\) denote the sigma-field generated by the latent states at
all strict ancestors of the parents of~\(w\),
\begin{equation*}
  \mathscr G_w
  =
  \sigma\!\left(
    F_u^{(i)}:
    u\in\bigcup_{p\in\PA(w)}\Anc(p),
    i\in\{a,b\}
  \right).
\end{equation*}
Conditional on \(\mathscr G_w\), the inputs of each parent module
\(f_{p_j}\) are fixed. By the mutual independence of the GP modules, the
random vectors \(\Delta_{p_1},\dots,\Delta_{p_m}\) are therefore conditionally
independent given \(\mathscr G_w\). Applying
Lemma~\ref{lem:single-radial-two-point} to each parent \(p_j\) gives
\begin{equation*}
  \Delta_{p_j}\mid\mathscr G_w
  \sim
  \mathrm N(0,\xi_j I_d),
  \qquad
  \xi_j
  =
  2\tau^2\bigl(1-\kappa(\varrho_{p_j}^2)\bigr).
\end{equation*}
Since \(\|\Delta_{p_j}\|_2^2\) is conditionally distributed as a scaled
\(\chi_d^2\) variable, for every \(s\ge0\),
\begin{equation}
  \EE\!\left[
    e^{-s\|\Delta_{p_j}\|_2^2}
    \middle|\mathscr G_w
  \right]
  =
  (1+2s\xi_j)^{-d/2}.
  \label{eq:parent-laplace-chi-square}
\end{equation}
Combining \eqref{eq:parent-laplace-chi-square} with conditional independence,
\begin{align}
  \EE[e^{-s\varrho_w^2}\mid\mathscr G_w]
  &=
  \prod_{j=1}^{m}
  \EE\!\left[
    e^{-s\|\Delta_{p_j}\|_2^2}
    \middle|\mathscr G_w
  \right]
  \nonumber\\
  &=
  \prod_{j=1}^{m}(1+2s\xi_j)^{-d/2}.
  \label{eq:indegree-conditional-laplace-product}
\end{align}
By the arithmetic--geometric mean inequality,
\begin{equation}
  \prod_{j=1}^{m}(1+2s\xi_j)
  \le
  \left(
    \frac{1}{m}\sum_{j=1}^{m}(1+2s\xi_j)
  \right)^{m}
  =
  \left(
    1+\frac{2s}{m}\sum_{j=1}^{m}\xi_j
  \right)^{m}.
  \label{eq:indegree-amgm}
\end{equation}
Since \(x\mapsto x^{-d/2}\) is decreasing on \((0,\infty)\),
\eqref{eq:indegree-conditional-laplace-product} and \eqref{eq:indegree-amgm}
give
\begin{equation}
  \EE[e^{-s\varrho_w^2}\mid\mathscr G_w]
  \ge
  \left(
    1+\frac{2s}{m}\sum_{j=1}^{m}\xi_j
  \right)^{-dm/2}.
  \label{eq:indegree-after-amgm}
\end{equation}
Define \(g_{s,m}:\mathbb{R}_+\to\mathbb{R}_+\) by:
\begin{equation*}
  g_{s,m}(x)
  :=
  \left(1+\frac{2s}{m}x\right)^{-dm/2}.
\end{equation*}
Then its first two derivatives are
\begin{align}
  g_{s,m}'(x)
  =&
  -ds\left(1+\frac{2s}{m}x\right)^{-dm/2-1}
  &&\le0, \hfill
  \label{eq:g-nonincreasing}\\
  g_{s,m}''(x)
  =&
  \frac{dm}{2}\left(\frac{dm}{2}+1\right)
  \left(\frac{2s}{m}\right)^2
  \left(1+\frac{2s}{m}x\right)^{-dm/2-2} \hspace*{-1in}
  &&\ge0. \hfill \nonumber
\end{align}
Thus \(g_{s,m}\) is nonincreasing and convex. Equation
\eqref{eq:indegree-after-amgm} and Jensen's inequality give
\begin{align*}
  \EE[e^{-s\varrho_w^2}]
  &=
  \EE\!\left[\EE[e^{-s\varrho_w^2}\mid\mathscr G_w]\right]
  \ge
  \EE\!\left[g_{s,m}\!\left(\sum_{j=1}^{m}\xi_j\right)\right]
  \ge
  g_{s,m}\!\left(\sum_{j=1}^{m}\EE[\xi_j]\right).
\end{align*}
We now bound \(\EE[\xi_j]\). By \eqref{eq:radial-conditional-second-moment},
\[
  \EE[\|\Delta_{p_j}\|_2^2\mid
  \{F_p^{(a)},F_p^{(b)}\}_{p\in\PA(p_j)}]
  =
  d\xi_j.
\]
Taking expectations of both sides gives
\begin{equation*}
  \EE[\xi_j]
  =
  \frac{1}{d}\EE[\|\Delta_{p_j}\|_2^2]
  \le
  \frac{C_{\ell-1}}{d},
\end{equation*}
because \(p_j\in\cA_{\ell-1}\). Thus
\begin{equation*}
  \sum_{j=1}^{m}\EE[\xi_j]
  \le
  \frac{mC_{\ell-1}}{d}.
\end{equation*}
By the monotonicity of \(g_{s,m}\) in \eqref{eq:g-nonincreasing},
\begin{equation}
  \EE[e^{-s\varrho_w^2}]
  \ge
  g_{s,m}\!\left(\frac{mC_{\ell-1}}{d}\right)
  =
  \left(1+\frac{2s}{d}C_{\ell-1}\right)^{-dm/2}.
  \label{eq:indegree-laplace-lower-final-m}
\end{equation}
Substituting \eqref{eq:indegree-laplace-lower-final-m} into
\eqref{eq:indegree-proof-laplace-start} gives

\begin{equation*}
  \EE[\|\Delta_w\|_2^2]
  \le
  2d\tau^2\left[
    1-
    \int_0^\infty
    \left(1+\frac{2s}{d}C_{\ell-1}\right)^{-dm/2}
    \nu(ds)
  \right].
\end{equation*}
Since \(m\le k_\ell\) and
\(m\mapsto(1+2sC_{\ell-1}/d)^{-dm/2}\) is nonincreasing, we further obtain
\begin{equation*}
  \EE[\|\Delta_w\|_2^2]
  \le
  2d\tau^2\left[
    1-
    \int_0^\infty
    \left(1+\frac{2s}{d}C_{\ell-1}\right)^{-dk_\ell/2}
    \nu(ds)
  \right]
  =
  \Psi_{k_\ell}(C_{\ell-1}).
\end{equation*}
Taking the maximum over \(w\in\cA_\ell\) proves
\(C_\ell\le\Psi_{k_\ell}(C_{\ell-1})\).

Finally, for every fixed \(u,s\ge0\), the map
\[
  k\longmapsto
  \left(1+\frac{2s}{d}u\right)^{-dk/2}
\]
is nonincreasing, so \(k\mapsto\Psi_k(u)\) is nondecreasing.
\end{proof}

The next corollary gives a local contraction criterion from the recursion of
Proposition~\ref{prop:indegree}, showing that the derivative of \(\Psi_k\) at
the origin determines whether small contrasts decay geometrically.

\begin{corollary}[Local contraction criterion]
\label{cor:indegree-local}
Suppose Assumption~\ref{ass:laplace-radial-degree} holds, and let
\(\bar k=\sup_{\ell\ge1}k_\ell\) and
\(\mu_1=\int_0^\infty s\,\nu(ds)<\infty\). If
\[
  2d\tau^2\bar k\,\mu_1<1,
\]
then there exist \(\rho\in(0,1)\) and \(u_\star>0\) such that, whenever
\(C_{\ell_0}\le u_\star\) for some \(\ell_0\ge0\),
\[
  C_{\ell_0+m}\le\rho^m C_{\ell_0},
  \qquad m\ge0.
\]
In particular, for the squared-exponential kernel with
\(\nu=\delta_{1/(2\lambda^2)}\), the condition becomes
\(d\tau^2\bar k/\lambda^2<1\).
\end{corollary}

\begin{proof}
For every \(k\ge1\), \(\Psi_k(0)=0\). Since \(\mu_1<\infty\), dominated
convergence justifies differentiating under the integral sign and gives
\begin{equation*}
  \Psi_k'(u)
  =
  2d\tau^2
  \int_0^\infty
  k s
  \left(1+\frac{2s}{d}u\right)^{-dk/2-1}
  \nu(ds).
\end{equation*}
In particular,
\(
  \Psi_k'(0)=2d\tau^2 k\mu_1
\), and
by assumption,
\(
  \Psi_{\bar k}'(0)=2d\tau^2\bar k\mu_1<1.
\)
Choose any \(\rho\in(\Psi_{\bar k}'(0),1)\). Since \(\Psi_{\bar k}\) is
differentiable at \(0\) with \(\Psi_{\bar k}(0)=0\), there exists
\(u_\star>0\) such that
\begin{equation*}
  \Psi_{\bar k}(u)\le\rho u
  \qquad\text{for all }u\in[0,u_\star].
\end{equation*}
Whenever \(C_{\ell-1}\le u_\star\), Proposition~\ref{prop:indegree} and the
monotonicity of \(\Psi_k\) in \(k\) give
\begin{equation*}
  C_\ell
  \le
  \Psi_{k_\ell}(C_{\ell-1})
  \le
  \Psi_{\bar k}(C_{\ell-1})
  \le
  \rho C_{\ell-1}.
\end{equation*}
In particular, \(C_\ell\le\rho C_{\ell-1}\le u_\star\), so the bound
propagates. Iterating from \(\ell_0\) onward gives
\[
  C_{\ell_0+m}\le\rho^m C_{\ell_0},
  \qquad m\ge0.
\]
For the squared-exponential case, \(\mu_1=1/(2\lambda^2)\), so the condition
\(2d\tau^2\bar k\mu_1<1\) reduces to \(d\tau^2\bar k/\lambda^2<1\).
\end{proof}

\begin{remark}[Global contraction]
\label{rem:indegree-global}
Each \(\Psi_k\) is concave on \([0,\infty)\), since \(u\mapsto(1+2su/d)^{-dk/2}\) is convex for every \(s\ge0\) and integration against \(\nu\) preserves convexity. Hence \(\Psi_k(u)\le\Psi_k(0)+\Psi_k'(0)\,u=2d\tau^2k\mu_1u\) for all \(u\ge0\). Under the condition of Corollary~\ref{cor:indegree-local} one can therefore take \(\rho=2d\tau^2\bar k\mu_1\) and \(u_\star=\infty\), so that \(C_\ell\le\rho^{\ell}C_0\) for every \(\ell\ge0\): the contraction is global and does not require the expected squared contrast to first enter a neighbourhood of zero.
\end{remark}

The corollary shows that, once the expected squared contrast enters a
sufficiently small neighbourhood of zero, it decays geometrically fast, with
the rate governed by \(\Psi_{\bar k}'(0)\). The criterion
\(2d\tau^2\bar k\mu_1<1\) makes the role of indegree clear: higher indegree
\(\bar k\) tightens the condition, because the concatenated parent input lives
in a space of dimension \(\bar k\cdot d\) and the kernel sees a larger total
squared distance. For a concrete illustration, consider a layered block of
squared-exponential modules with output dimension \(d=2\), marginal variance
\(\tau^2=1\), and length-scale \(\lambda=2\). The local contraction criterion
then reads \(\bar k<\lambda^2/d=2\), so a chain \((\bar k=1)\) contracts
locally, while a block with maximal indegree \(\bar k=2\) or higher may fail
to satisfy this sufficient contraction condition.

\begin{remark}[Layer-dependent dimensions and kernel parameters]
\label{rem:layer-dependent-dimensions}
The homogeneity assumptions in Proposition~\ref{prop:indegree} are only used
to keep the notation compact. Suppose instead that every node in
\(\cA_{\ell-1}\) has common output dimension \(d_{\ell-1}\), every node in
\(\cA_\ell\) has common output dimension \(d_\ell\), and every node in
\(\cA_\ell\) uses a common radial kernel
\[
  K_w(u,u')
  =
  \tau_\ell^2\kappa_\ell(\|u-u'\|_2^2)I_{d_\ell},
  \qquad
  \kappa_\ell(r)=\int_0^\infty e^{-sr}\,\nu_\ell(ds).
\]
The proof of Proposition~\ref{prop:indegree} \emph{mutatis mutandis} gives
\[
  C_\ell
  \le
  2d_\ell\tau_\ell^2
  \left[
    1-
    \int_0^\infty
    \left(
      1+\frac{2s}{d_{\ell-1}}C_{\ell-1}
    \right)^{-d_{\ell-1}k_\ell/2}
    \nu_\ell(ds)
  \right].
\]
Thus the argument extends verbatim to layer-dependent dimensions and kernel
hyperparameters.
\end{remark}

The contraction result of Proposition~\ref{prop:indegree} is specific to
product-type fusion. Under fusion mechanisms that have an additive
root-only summand, the separating property is robust to indegree:

\begin{proposition}[Additive kernel decompositions preserve separating contributions]
\label{prop:additive-monotone}
Let \(w\in\cU\) be a non-root node with kernel
\(K_w=\sum_{S\in\cS_w}K_w^{(S)}\), where
\(\cS_w\subseteq2^{\PA(w)}\setminus\{\emptyset\}\) and each
\(K_w^{(S)}\) is a valid positive-semidefinite kernel on the coordinates
indexed by~\(S\). Then, for any two distinct cases \(a\neq b\),
\[
  \Gamma_w(a,b)
  =  \sum_{S\in\cS_w}\Gamma_w^{(S)}(a,b),
  \qquad
  \Gamma_w^{(S)}(a,b)\succeq0
  \quad\text{for every }S\in\cS_w.
\]
In particular, if some subfamily \(\cS_w^\star\subseteq\cS_w\) satisfies
\[
  \sum_{S\in\cS_w^\star}\Gamma_w^{(S)}(a,b)
  \succeq v_\star I_{d_w}
\]
almost surely, then \(w\) is \(v_\star\)-separating for \((a,b)\),
irrespective of the remaining summands.
\end{proposition}

\begin{proof}
Write \(u_a=(F_p^{(a)})_{p\in\PA(w)}\) and
\(u_b=(F_p^{(b)})_{p\in\PA(w)}\) for the concatenated parent states, and let
\(u_{a,S}\), \(u_{b,S}\) denote their restrictions to the coordinates indexed
by \(S\). For each \(S\in\cS_w\), define
\[
  \Gamma_w^{(S)}(a,b)
  :=
  K_w^{(S)}(u_{a,S},u_{a,S})
  +K_w^{(S)}(u_{b,S},u_{b,S})
  -K_w^{(S)}(u_{a,S},u_{b,S})
  -K_w^{(S)}(u_{b,S},u_{a,S}).
\]
Then, by the additive structure of \(K_w\),
\[
  \Gamma_w(a,b)=\sum_{S\in\cS_w}\Gamma_w^{(S)}(a,b).
\]
For each \(S\in\cS_w\), validity of the kernel \(K_w^{(S)}\) implies
\(\Gamma_w^{(S)}(a,b)\succeq0\). Therefore
\[
  \Gamma_w(a,b)
  =
  \sum_{S\in\cS_w}\Gamma_w^{(S)}(a,b)
  \succeq
  \sum_{S\in\cS_w^\star}\Gamma_w^{(S)}(a,b)
  \succeq
  v_\star I_{d_w},
\]
where the first inequality drops the positive-semidefinite summands outside
\(\cS_w^\star\). Thus \(w\) is \(v_\star\)-separating.
\end{proof}

\subsection{Outdegree effects through branching}
\label{app:outdegree-branching}

We now isolate a branching mechanism inside the layered block of
Assumption~\ref{ass:laplace-radial-degree}. Let
\(\mathcal B_\ell\subseteq\cA_\ell\) be non-empty subsets and, for each
\(w\in\mathcal B_\ell\), choose a set
\(\mathrm{Ch}_\star(w)\subseteq\mathcal B_{\ell+1}\) of designated children
such that \(w\in\PA(v)\) for every \(v\in\mathrm{Ch}_\star(w)\). Assume that
\(|\mathrm{Ch}_\star(w)|=b\) for all \(w\) and that the families
\(\{\mathrm{Ch}_\star(w):w\in\mathcal B_\ell\}\) are pairwise disjoint for
each \(\ell\). The designated edges then form a rooted \(b\)-ary branching
subgraph.

Nodes in this branching subgraph are allowed to have additional parents
outside the designated edges. This can only strengthen the lower bound used
below. To see this, consider a designated edge \(w\to v\). Since
\(w\in\PA(v)\), the term corresponding to \(w\) appears in the sum defining
\(\varrho_v^2\), so
\[
  \varrho_v^2
  =
  \sum_{p\in\PA(v)}\|\Delta_p\|_2^2
  \ge
  \|\Delta_w\|_2^2.
\]
Therefore, on the event \(\{\|\Delta_w\|_2\ge t\}\),
\begin{equation}
  \varrho_v^2\ge t^2.
  \label{eq:var-rhosq-grt-tsq}
\end{equation}
Because \(\kappa\) is nonincreasing, \eqref{eq:var-rhosq-grt-tsq} implies
\[
  2\tau^2\bigl(1-\kappa(\varrho_v^2)\bigr)I_d
  \succeq
  2\tau^2\bigl(1-\kappa(t^2)\bigr)I_d
  =:
  \underline\sigma_t^2 I_d,
\]
where \(\underline\sigma_t^2\) is the smallest conditional contrast variance
parameter compatible with the event that one designated parent already has
contrast at least \(t\). For each \(t>0\), define
\[
  p_t
  :=
  \begin{cases}
    \PP\!\left(
      \chi_d^2\ge t^2/\underline\sigma_t^2
    \right),
    & \underline\sigma_t^2>0,\\[1.5ex]
    0, & \underline\sigma_t^2=0.
  \end{cases}
\]

Figure~\ref{fig:outdegree-branching} shows a binary instance \((b=2)\) of
this designated branching pattern inside the layered block.

\begin{figure}[H]
\centering
\begin{tikzpicture}[
  x=1.35cm,
  y=1.10cm,
  >=latex,
  latent/.style={
    circle,
    draw,
    line width=0.45pt,
    inner sep=1.2pt,
    minimum size=4.9mm
  },
  desig/.style={
    ->,
    thick
  },
  omit/.style={
    ->,
    dashed,
    thin
  },
  faint/.style={
    draw,
    line width=0.45pt
  },
  subset/.style={
    draw,
    dashed,
    line width=0.45pt
  },
  lab/.style={font=\small},
  slab/.style={font=\scriptsize}
]

\path[faint] (0.00, 0.00) ellipse [x radius=1.15, y radius=0.36];
\path[faint] (0.00,-1.85) ellipse [x radius=2.05, y radius=0.36];
\path[faint] (0.00,-3.70) ellipse [x radius=2.90, y radius=0.36];

\node[lab,anchor=east] at (-3.05, 0.00) {$\cA_{\ell}$};
\node[lab,anchor=east] at (-3.05,-1.85) {$\cA_{\ell+1}$};
\node[lab,anchor=east] at (-3.05,-3.70) {$\cA_{\ell+2}$};

\node[latent] (w0) at ( 0.00, 0.00) {$w_0$};

\node[latent] (v1) at (-0.85,-1.85) {};
\node[latent] (v2) at ( 0.85,-1.85) {};

\node[latent] (x1) at (-1.80,-3.70) {};
\node[latent] (x2) at (-0.60,-3.70) {};
\node[latent] (x3) at ( 0.60,-3.70) {};
\node[latent] (x4) at ( 1.80,-3.70) {};

\draw[subset] (0.00,0.00) ellipse [x radius=0.34, y radius=0.28];
\draw[subset] (0.00,-1.85) ellipse [x radius=1.35, y radius=0.28];
\draw[subset] (0.00,-3.70) ellipse [x radius=2.30, y radius=0.28];

\node[slab,anchor=west] at (1.45, 0.00) {$\mathcal B_{\ell}$};
\node[slab,anchor=west] at (2.10,-1.85) {$\mathcal B_{\ell+1}$};
\node[slab,anchor=west] at (3.05,-3.70) {$\mathcal B_{\ell+2}$};

\draw[desig] (w0) -- (v1);
\draw[desig] (w0) -- (v2);

\draw[desig] (v1) -- (x1);
\draw[desig] (v1) -- (x2);
\draw[desig] (v2) -- (x3);
\draw[desig] (v2) -- (x4);

\draw[omit] (0.00, 0.70) -- (w0);

\draw[omit] (-1.60,-1.70) -- (v1);
\draw[omit] ( 1.60,-1.70) -- (v2);

\draw[omit] (-2.35,-3.55) -- (x1);
\draw[omit] ( 2.35,-3.55) -- (x4);

\draw[omit] (x1) -- (-1.80,-4.40);
\draw[omit] (x2) -- (-0.60,-4.40);
\draw[omit] (x3) -- ( 0.60,-4.40);
\draw[omit] (x4) -- ( 1.80,-4.40);

\end{tikzpicture}
\caption{Binary designated branching pattern \((b=2)\) inside the layered block. The horizontal contours indicate the antichains \(\cA_\ell\), \(\cA_{\ell+1}\), and \(\cA_{\ell+2}\). The dashed inner contours indicate the selected node subsets \(\mathcal B_\ell\subseteq\cA_\ell\), \(\mathcal B_{\ell+1}\subseteq\cA_{\ell+1}\), and \(\mathcal B_{\ell+2}\subseteq\cA_{\ell+2}\). Solid arrows denote the designated branching edges between these subsets; dashed arrows indicate omitted incoming or outgoing edges in the surrounding DAG.}
\label{fig:outdegree-branching}
\end{figure}

\begin{proposition}[Outdegree can sustain contrasts through branching]
\label{prop:outdegree-branching}
Assume Assumption~\ref{ass:laplace-radial-degree}. Let
\((\mathcal B_\ell)_{\ell\ge0}\) and \(\{\mathrm{Ch}_\star(w)\}\) be as above,
with common branching factor \(b\ge1\), and fix a seed
\(w_0\in\mathcal B_0\). If \(\PP(\|\Delta_{w_0}\|_2\ge t)>0\) and
\(b p_t>1\), then
\[
  \PP\!\left(
    M_\ell\ge t
    \text{ for all }\ell\ge0
  \right)>0.
\]
\end{proposition}

\begin{proof}
Fix \(t>0\). By the definition above, the assumption \(b p_t>1\) implies
\(p_t>0\), hence \(\underline\sigma_t^2>0\). Therefore
\[
  p_t
  =
  \PP\!\left(
    \chi_d^2\ge t^2/\underline\sigma_t^2
  \right).
\]

We prove the stronger statement that the threshold persists already on the
designated branching subgraph, namely
\begin{equation}
  \PP\!\left(
    \max_{w\in\mathcal B_\ell}\|\Delta_w\|_2\ge t
    \text{ for all }\ell\ge0
  \right)>0.
  \label{eq:branching-strong-event}
\end{equation}
Since \(\mathcal B_\ell\subseteq\cA_\ell\) for every \(\ell\),
\eqref{eq:branching-strong-event} immediately implies the proposition.

For each \(\ell\ge0\), define the sigma-field
\begin{equation*}
  \mathcal H_\ell
  :=
  \sigma\!\left(
    f_u:u\in\Anc(\mathcal B_{\ell+1})
  \right).
\end{equation*}
Because \(|\mathrm{Ch}_\star(w)|=b\ge1\) for every \(w\in\mathcal B_\ell\),
each \(w\in\mathcal B_\ell\) has at least one designated child in
\(\mathcal B_{\ell+1}\), and therefore
\(\mathcal B_\ell\subseteq\Anc(\mathcal B_{\ell+1})\). By the same
measurability argument used in Lemma~\ref{lem:antichain-conditional-bound},
every parent state of a node in \(\mathcal B_{\ell+1}\) is
\(\mathcal H_\ell\)-measurable, and \(\Delta_w\) is
\(\mathcal H_\ell\)-measurable for every \(w\in\mathcal B_\ell\).

We first derive a one-step lower bound. Fix \(\ell\ge0\),
\(w\in\mathcal B_\ell\), and \(v\in\mathrm{Ch}_\star(w)\). The squared parent
contrast at \(v\) is
\[
  \varrho_v^2
  =
  \sum_{p\in\PA(v)}\|\Delta_p\|_2^2.
\]
The parent states of \(v\) are \(\mathcal H_\ell\)-measurable. Therefore, by
Lemma~\ref{lem:single-radial-two-point}, under the conditional law given
\(\mathcal H_\ell\),
\[
  \Delta_v\mid\mathcal H_\ell
  \sim
  \mathrm N\!\left(
    0,
    2\tau^2\bigl(1-\kappa(\varrho_v^2)\bigr)I_d
  \right).
\]
On the event \(A_w:=\{\|\Delta_w\|_2\ge t\}\), equation
\eqref{eq:var-rhosq-grt-tsq} gives \(\varrho_v^2\ge t^2\). Therefore, on
\(A_w\),
\[
  2\tau^2\bigl(1-\kappa(\varrho_v^2)\bigr)
  \ge
  \underline\sigma_t^2.
\]
Since
\[
  \frac{\|\Delta_v\|_2^2}
  {2\tau^2(1-\kappa(\varrho_v^2))}
  \Bigm|\mathcal H_\ell
  \sim
  \chi_d^2,
\]
we obtain
\begin{equation}
  \PP\!\left(
    \|\Delta_v\|_2\ge t\mid\mathcal H_\ell
  \right)
  =
  \PP\!\left(
    \chi_d^2\ge
    \frac{t^2}{2\tau^2(1-\kappa(\varrho_v^2))}
  \right)
  \ge
  \PP\!\left(
    \chi_d^2\ge\frac{t^2}{\underline\sigma_t^2}
  \right)
  =
  p_t
  \qquad\text{on }A_w.
  \label{eq:outdegree-step-lower-bound}
\end{equation}

We next record the relevant conditional independence. Let \(v\neq v'\) be
distinct nodes in
\(\bigcup_{u\in\mathcal B_\ell}\mathrm{Ch}_\star(u)\subseteq\mathcal B_{\ell+1}\).
Conditional on \(\mathcal H_\ell\), the contrasts \(\Delta_v\) and
\(\Delta_{v'}\) are measurable functions of the independent GP modules
\(f_v\) and \(f_{v'}\), respectively. Hence \(\Delta_v\) and \(\Delta_{v'}\)
are conditionally independent given \(\mathcal H_\ell\). In particular, the
indicators \(\mathbf{1}\{\|\Delta_v\|_2\ge t\}\) are conditionally
independent across distinct designated children.

We now build an embedded Galton--Watson process by thinning these threshold
exceedances.  On a product extension of the original probability space, let
\(\{U_v\}_{v\in\bigcup_{\ell\ge1}\mathcal B_\ell}\) be an i.i.d. family of
uniform random variables in \((0,1)\), independent of the DAG-DGP prior. This
auxiliary extension does not change the marginal DAG-DGP probabilities.
For each \(\ell\ge0\), define
\[
  \widehat{\mathcal H}_\ell
  :=
  \sigma\!\left(
    \mathcal H_\ell,
    \{U_v:v\in\bigcup_{m=1}^{\ell}\mathcal B_m\}
  \right).
\]
Thus \(\widehat{\mathcal H}_\ell\) contains the past uniforms up to level
\(\ell\), but not the current uniforms on level \(\ell+1\).

We recursively construct random sets \(D_\ell\subseteq\mathcal B_\ell\), to be
interpreted as a distinguished surviving population. Set
\[
  D_0
  :=
  \begin{cases}
    \{w_0\}, & \|\Delta_{w_0}\|_2\ge t,\\
    \emptyset, & \|\Delta_{w_0}\|_2<t.
  \end{cases}
\]
Then \(D_0\) is \(\widehat{\mathcal H}_0\)-measurable and
\[
  |D_0|=\mathbf{1}\{\|\Delta_{w_0}\|_2\ge t\}.
\]

Suppose \(D_\ell\) has been constructed and is
\(\widehat{\mathcal H}_\ell\)-measurable. Let
\[
  I_\ell:=\bigcup_{w\in D_\ell}\mathrm{Ch}_\star(w).
\]
Because the designated child families are pairwise disjoint, each
\(v\in I_\ell\) belongs to the designated child set of a unique
\(w\in D_\ell\). For \(v\in I_\ell\), define
\[
  r_v
  :=
  \PP\!\left(
    \|\Delta_v\|_2\ge t\mid\mathcal H_\ell
  \right).
\]
The variable \(r_v\) is \(\mathcal H_\ell\)-measurable, hence
\(\widehat{\mathcal H}_\ell\)-measurable. If \(v\in I_\ell\), then
\(v\in\mathrm{Ch}_\star(w)\) for some \(w\in D_\ell\), and by construction
\(\|\Delta_w\|_2\ge t\). Therefore \eqref{eq:outdegree-step-lower-bound}
gives
\[
  r_v\ge p_t>0
  \qquad\text{for every }v\in I_\ell.
\]
Hence \(p_t/r_v\in[0,1]\). Define
\[
  Y_v
  :=
  \mathbf{1}\{\|\Delta_v\|_2\ge t\}
  \mathbf{1}\!\left\{U_v\le\frac{p_t}{r_v}\right\},
  \qquad v\in I_\ell,
\]
and set
\[
  D_{\ell+1}:=\{v\in I_\ell:Y_v=1\}.
\]
By construction, \(D_{\ell+1}\) is \(\widehat{\mathcal H}_{\ell+1}\)-measurable
and every retained child is above the threshold.

We claim that, conditional on \(\widehat{\mathcal H}_\ell\), the family
\(\{Y_v:v\in I_\ell\}\) is independent with each \(Y_v\sim\mathrm{Bernoulli}(p_t)\).
Let \(v_1,\dots,v_m\in I_\ell\) be distinct. For each \(j\), the variable
\(\Delta_{v_j}\) is measurable with respect to
\(\sigma(\mathcal H_\ell,f_{v_j})\), because the parent inputs of \(v_j\) are
\(\mathcal H_\ell\)-measurable. The GP modules
\(f_{v_1},\dots,f_{v_m}\) are mutually independent and jointly independent of
the past uniforms entering \(\widehat{\mathcal H}_\ell\). The current uniforms
\(U_{v_1},\dots,U_{v_m}\) are i.i.d. and independent of both the DAG-DGP prior
and the past uniforms. Hence the pairs
\[
  (\Delta_{v_1},U_{v_1}),\dots,(\Delta_{v_m},U_{v_m})
\]
are conditionally independent given \(\widehat{\mathcal H}_\ell\), and so are
the variables \(Y_{v_1},\dots,Y_{v_m}\).

For the conditional success probability, using that \(U_v\) is independent of
\(\sigma(\Delta_v,\widehat{\mathcal H}_\ell)\), we obtain
\begin{align}
  \PP(Y_v=1\mid\widehat{\mathcal H}_\ell)
  &=
  \EE\!\left[
    \mathbf{1}\{\|\Delta_v\|_2\ge t\}
    \mathbf{1}\!\left\{U_v\le\frac{p_t}{r_v}\right\}
    \middle|\widehat{\mathcal H}_\ell
  \right]
  \nonumber\\
  &=
  \EE\!\left[
    \mathbf{1}\{\|\Delta_v\|_2\ge t\}
    \frac{p_t}{r_v}
    \middle|\widehat{\mathcal H}_\ell
  \right]
  \nonumber\\
  &=
  \frac{p_t}{r_v}
  \PP\!\left(
    \|\Delta_v\|_2\ge t\mid\widehat{\mathcal H}_\ell
  \right).
  \label{eq:branching-Yv-probability-start}
\end{align}
The additional information in \(\widehat{\mathcal H}_\ell\) beyond
\(\mathcal H_\ell\) consists only of past uniforms, which are independent of
the current GP variables. Therefore
\[
  \PP\!\left(
    \|\Delta_v\|_2\ge t\mid\widehat{\mathcal H}_\ell
  \right)
  =
  \PP\!\left(
    \|\Delta_v\|_2\ge t\mid\mathcal H_\ell
  \right)
  =
  r_v.
\]
Substituting into \eqref{eq:branching-Yv-probability-start} gives
\[
  \PP(Y_v=1\mid\widehat{\mathcal H}_\ell)=p_t.
\]

For each \(w\in D_\ell\), define
\[
  N_w:=\sum_{v\in\mathrm{Ch}_\star(w)}Y_v.
\]
Because \(|\mathrm{Ch}_\star(w)|=b\) and the designated child sets are
pairwise disjoint, conditional on \(\widehat{\mathcal H}_\ell\), each \(N_w\)
has law \(\mathrm{Bin}(b,p_t)\), and the family \(\{N_w:w\in D_\ell\}\) is
conditionally independent.

Define \(Z_\ell:=|D_\ell|\). Then
\[
  Z_{\ell+1}=\sum_{w\in D_\ell}N_w.
\]
Thus, on the event \(\{Z_\ell=n\}\), the variable \(Z_{\ell+1}\) is the sum
of \(n\) independent \(\mathrm{Bin}(b,p_t)\) variables. Its conditional law
depends only on \(n\), and not on the earlier history. Hence
\((Z_\ell)_{\ell\ge0}\) is a Galton--Watson process with offspring
distribution \(\mathrm{Bin}(b,p_t)\) and random initial state
\[
  Z_0=\mathbf{1}\{\|\Delta_{w_0}\|_2\ge t\}.
\]
Its generating function is
\[
  g(s)=(1-p_t+p_ts)^b,
  \qquad s\in[0,1],
\]
and its mean offspring number is \(g'(1)=bp_t>1\). Hence the process is
supercritical. By the classical extinction-probability theorem for
Galton--Watson processes, the extinction probability \(\xi_t\) is the
smallest nonnegative solution of \(g(s)=s\). Since \(g'(1)=bp_t>1\), one has
\(\xi_t<1\), and the survival probability \(\pi_t:=1-\xi_t\) is strictly
positive~\citep[Ch.~I, Sec.~5, Thm.~1]{athreya1972}. Conditional on
\(Z_0=1\),
\[
  \PP\!\left(
    Z_\ell>0\text{ for all }\ell\ge0
    \middle| Z_0=1
  \right)
  =
  \pi_t>0.
\]
Since \(Z_0=\mathbf{1}\{\|\Delta_{w_0}\|_2\ge t\}\),
\begin{align*}
  \PP\!\left(Z_\ell>0\text{ for all }\ell\ge0\right)
  &=
  \PP\!\left(
    Z_\ell>0\text{ for all }\ell\ge0
    \middle|Z_0=1
  \right)
  \PP(Z_0=1)\\
  &=
  \pi_t\,\PP(\|\Delta_{w_0}\|_2\ge t)>0.
\end{align*}
On the event \(\{Z_\ell>0\text{ for all }\ell\ge0\}\), every set \(D_\ell\) is
non-empty, and every \(w\in D_\ell\) satisfies \(\|\Delta_w\|_2\ge t\). Since
\(D_\ell\subseteq\mathcal B_\ell\),
\[
  \max_{w\in\mathcal B_\ell}\|\Delta_w\|_2\ge t
  \qquad\text{for all }\ell\ge0.
\]
This proves \eqref{eq:branching-strong-event}. Since
\(\mathcal B_\ell\subseteq\cA_\ell\), on the same event
\[
  M_\ell
  \ge
  \max_{w\in\mathcal B_\ell}\|\Delta_w\|_2
  \ge
  t
  \qquad\text{for all }\ell\ge0.
\]
Therefore
\[
  \PP\!\left(
    M_\ell\ge t
    \text{ for all }\ell\ge0
  \right)>0,
\]
which proves the claim.
\end{proof}

\begin{corollary}[Scalar threshold probability]
\label{cor:outdegree-scalar-threshold}
Under the additional assumption \(d=1\), the threshold probability \(p_t\) in
Proposition~\ref{prop:outdegree-branching} takes the form
\[
  p_t
  :=
  \begin{cases}
    2\,\Phi_{\mathrm N}\!\left(-\dfrac{t}{\underline\sigma_t}\right),
    & \underline\sigma_t^2>0,\\[1.5ex]
    0, & \underline\sigma_t^2=0.
  \end{cases}
\]
\end{corollary}

\begin{proof}
If \(d=1\) and \(\underline\sigma_t^2>0\), then
\[
  p_t
  =
  \PP\!\left(
    \chi_1^2\ge\frac{t^2}{\underline\sigma_t^2}
  \right)
  =
  \PP\!\left(
    |Z|\ge\frac{t}{\underline\sigma_t}
  \right),
  \qquad Z\sim\mathrm N(0,1).
\]
Hence
\[
  p_t=2\Phi_{\mathrm N}\!\left(-\frac{t}{\underline\sigma_t}\right).
\]
If \(\underline\sigma_t^2=0\), the conclusion is immediate from the definition
of \(p_t\).
\end{proof}

\section{Intermediate observations as stochastic skip connections}
\label{app:posterior-refresh}

This appendix proves Theorem~\ref{thm:stochastic-skip-main} and the following corollary thereto, which characterises a simple setting in a readily interpretable way.

\begin{corollary}[Gaussian refresh at an observed source]
\label{cor:gaussian-refresh-main}
Let \(u\in\cA_{\ell_0}\) be scalar, assume
\(\cO_u=\{a,b\}\times\{1\}\), and suppose that
\(K_u(x,x)=\tau_u^2\) for all parent inputs \(x\). If
\[
  Y_u^{(i)} = F_u^{(i)}+\xi_u^{(i)},
  \qquad
  \xi_u^{(i)}\stackrel{\mathrm{i.i.d.}}{\sim}\cN(0,\sigma_u^2),
  \qquad i\in\{a,b\},
\]
then, conditional on the parent states and on
\((Y_u^{(a)},Y_u^{(b)})\), on the event
\(\Gamma_u(a,b)>0\),
\[
  F_u^{(a)}-F_u^{(b)}
  \sim
  \cN\!\left(
    \frac{\Gamma_u(a,b)}{\Gamma_u(a,b)+2\sigma_u^2}
    \bigl(Y_u^{(a)}-Y_u^{(b)}\bigr),\,
    \frac{2\sigma_u^2\,\Gamma_u(a,b)}
         {\Gamma_u(a,b)+2\sigma_u^2}
  \right).
\]
Consequently, if \(m_u\) and \(v_u\) denote the mean and variance in
the display above, then the local filtered law assigns probability at
least
\(
  1-\Phi_{\mathrm N}\!\left(
    (\varepsilon-|m_u|)/\sqrt{v_u}
  \right)
\)
to the event \(|F_u^{(a)}-F_u^{(b)}|>\varepsilon\). If
\(\Gamma_u(a,b)=0\), the conditional contrast is degenerate and this
source strength is zero.
\end{corollary}

Thus the Gaussian source strength is governed by the observed
discrepancy \(Y_u^{(a)}-Y_u^{(b)}\), the noise variance
\(\sigma_u^2\), and the local two-point prior variance
\(\Gamma_u(a,b)\).

The argument is organised so that
the main theorem follows from three ingredients. First, filtering on an
observed antichain yields a conditional product structure across the nodes of
that antichain. Second, a refreshed contrast at one such node can be
transported downstream by additive kernel components. Third, disjoint
retaining routes remain conditionally independent, so their effects combine
multiplicatively.

Throughout this section the dataset is fixed and only the latent states remain
random. We work under the standing well-posedness condition that every
posterior or local conditional distribution displayed below is well defined,
i.e. that the corresponding normalising constant is finite and strictly
positive. As in the previous sections, root states are deterministic and are
omitted from latent sigma-fields.
\subsection{Filtering on an observed antichain}
\label{app:filtering-kernels}

Consider a progressive antichain sequence \((\cA_\ell)_{\ell\ge0}\) of non-root
nodes. For each \(\ell\ge0\), let
\begin{equation*}
  \cU_{\le\ell}:=\Anc(\cA_\ell)\cup\cA_\ell.
\end{equation*}
The corresponding filtering posterior is
\begin{equation*}
  \Pi_\ell\!\left(
    d\{\bF_w\}_{w\in\cU_{\le\ell}}
  \right)
  \propto
  \prod_{w\in\cU_{\le\ell}}
  \left[
  p_0\!\left(
    d\bF_w\mid\{\bF_p\}_{p\in\PA(w)}
  \right)
  p_w(\bY_w\mid\bF_w,\cO_w)
  \right],
\end{equation*}
where \(p_w(\bY_w\mid\bF_w,\cO_w)\equiv1\) when \(\cO_w=\emptyset\). We also
write
\begin{equation*}
  \mathscr H_\ell
  :=
  \sigma\!\left(
    \bF_v:v\in\Anc(\cA_\ell)
  \right),
  \qquad
  \mathscr F_\ell
  :=
  \sigma\!\left(
    \bF_v:v\in\cU_{\le\ell}
  \right).
\end{equation*}
Thus \(\mathscr H_\ell\) records the strict latent ancestors of the current
antichain, whereas \(\mathscr F_\ell\) also contains the current antichain
variables.

If \(w\in\cA_\ell\), every latent parent of \(w\) is
\(\mathscr H_\ell\)-measurable. We define the local prior kernel
\begin{equation*}
  P_w^\ell(d\bF_w\mid\mathscr H_\ell)
  :=
  p_0\!\left(
    d\bF_w\mid\{\bF_p\}_{p\in\PA(w)}
  \right).
\end{equation*}
The one-node filtered kernel is the probability kernel
\begin{equation*}
  Q_w^\ell(d\bF_w\mid\mathscr H_\ell)
  \propto
  p_w(\bY_w\mid\bF_w,\cO_w)
  P_w^\ell(d\bF_w\mid\mathscr H_\ell),
\end{equation*}
For \(\ell_1\ge\ell_0\), we write \(\Pi_{\ell_0\to\ell_1}\) for the
predictive law obtained by first drawing from \(\Pi_{\ell_0}\) and then
propagating the DAG-DGP prior forward from \(\cA_{\ell_0}\) to
\(\cA_{\ell_1}\), without assimilating observations beyond level \(\ell_0\).

\begin{lemma}[Conditional factorisation of the filtered antichain]
\label{lem:filtered-factorization}
For every \(\ell\ge0\), the conditional law of the current antichain under
\(\Pi_\ell\) factorises as
\[
  \Pi_\ell\!\left(
    d\{\bF_w\}_{w\in\cA_\ell}
    \middle|
    \mathscr H_\ell
  \right)
  =
  \bigotimes_{w\in\cA_\ell}
  Q_w^\ell(d\bF_w\mid\mathscr H_\ell).
\]
\end{lemma}

\begin{proof}
The proof makes repeated use of standard conditional-independence properties, for which we refer to Sections~3 and~4 of \citep{dawid1979}.

Fix \(\ell\ge0\). By definition,
\[
  \mathscr H_\ell
  =
  \sigma\!\left(\bF_v:v\in\Anc(\cA_\ell)\right),
\]
so each state matrix \(\bF_v\), \(v\in\Anc(\cA_\ell)\), is
\(\mathscr H_\ell\)-measurable. Equivalently, under any regular conditional
law given \(\mathscr H_\ell\), these strict-ancestor states are degenerate at
their realised values. Since \(\cA_\ell\) is an antichain, no node in
\(\cA_\ell\) is a strict ancestor of any other node in \(\cA_\ell\).
Therefore, for each \(w\in\cA_\ell\), the parent states of \(w\) are either
deterministic roots or are \(\mathscr H_\ell\)-measurable.

We first justify the conditional-independence step. Temporarily regard the
local observations as random variables
\(\widetilde{\bY}_w\), \(w\in\cA_\ell\), generated from the nodewise
conditional distributions \(p_w(\cdot\mid\bF_w,\cO_w)\), and define
\[
  B_w:=(\bF_w,\widetilde{\bY}_w),
  \qquad w\in\cA_\ell.
\]
Given \(\mathscr H_\ell\), the blocks \(\{B_w:w\in\cA_\ell\}\) are jointly
conditionally independent, as the latent states are generated by distinct
independent GP modules evaluated at \(\mathscr H_\ell\)-measurable inputs,
and the observation variables are then generated nodewise from their
corresponding latent states.

Consider first two disjoint subcollections \(I,J\subset\cA_\ell\), and write
\(B_I=(B_w)_{w\in I}\), \(B_J=(B_w)_{w\in J}\), with analogous notation for
\(\widetilde{\bY}_I\) and \(\widetilde{\bY}_J\). We have
\[
  B_I\perp\!\!\!\perp B_J\mid\mathscr H_\ell.
\]
Since \(\widetilde{\bY}_I\) is a measurable function of \(B_I\), the
conditioning-stability property in
\citep[Sec.~4, Lemma~4.2(ii)]{dawid1979} gives
\[
  B_I\perp\!\!\!\perp B_J
  \mid
  \mathscr H_\ell,\widetilde{\bY}_I .
\]
Using the symmetry of conditional independence
\citep[Sec.~3.1, Theorem~3.1]{dawid1979}, applying again
\citep[Sec.~4, Lemma~4.2(ii)]{dawid1979} to the measurable function
\(\widetilde{\bY}_J\) of \(B_J\), and then using symmetry once more, gives
\[
  B_I\perp\!\!\!\perp B_J
  \mid
  \mathscr H_\ell,\widetilde{\bY}_I,\widetilde{\bY}_J .
\]
By the closure of conditional independence under measurable transformations
\citep[Sec.~4, Lemma~4.2(i)]{dawid1979}, applied to the maps
\(B_I\mapsto\bF_I\) and \(B_J\mapsto\bF_J\), this implies
\[
  \bF_I\perp\!\!\!\perp\bF_J
  \mid
  \mathscr H_\ell,\widetilde{\bY}_I,\widetilde{\bY}_J .
\]
The same argument applied inductively, using the joint-independence convention
following \citep[Sec.~4, Lemma~4.3]{dawid1979}, gives joint conditional
independence of \(\{\bF_w:w\in\cA_\ell\}\) after conditioning on
\(\mathscr H_\ell\) and on the realised local observations at the current
antichain.

The preceding conditional-independence argument shows that the conditional law
of \(\{\bF_w:w\in\cA_\ell\}\), given \(\mathscr H_\ell\) and the realised
current-antichain observations, factorises over nodes. It remains to identify
the corresponding nodewise conditional factor, and to verify that it is exactly
\(Q_w^\ell(\cdot\mid\mathscr H_\ell)\). The conditional prior law of
\(\{\bF_w:w\in\cA_\ell\}\) given \(\mathscr H_\ell\) is
\begin{equation}
  \bigotimes_{w\in\cA_\ell}
  P_w^\ell(d\bF_w\mid\mathscr H_\ell),
  \label{eq:antichain-prior-product-given-H}
\end{equation}
because distinct current-antichain states are obtained by evaluating distinct
independent GP modules at \(\mathscr H_\ell\)-measurable inputs. The
conditional distribution factors for the current-antichain observations also
factorise nodewise:
\begin{equation}
  \prod_{w\in\cA_\ell}p_w(\bY_w\mid\bF_w,\cO_w),
  \label{eq:antichain-observation-factor-product}
\end{equation}
with factors equal to one at unobserved nodes. All conditional distribution
factors and prior factors associated with strict ancestors are
\(\mathscr H_\ell\)-measurable and hence do not affect the conditional
distribution of \(\{\bF_w:w\in\cA_\ell\}\) beyond the realised value of
\(\mathscr H_\ell\).

Let \(B_w^\star\) be a measurable set in the state space of \(\bF_w\), one for
each \(w\in\cA_\ell\). By Bayes' rule,
\eqref{eq:antichain-prior-product-given-H}, and
\eqref{eq:antichain-observation-factor-product},
\begin{align}
  &\Pi_\ell\!\left(
    \bF_w\in B_w^\star\text{ for all }w\in\cA_\ell
    \middle|\mathscr H_\ell
  \right)
  \nonumber\\
  &\quad=
  \frac{
    \int_{\prod_{w\in\cA_\ell}B_w^\star}
    \prod_{w\in\cA_\ell}
    p_w(\bY_w\mid\bF_w,\cO_w)
    P_w^\ell(d\bF_w\mid\mathscr H_\ell)
  }{
    \int
    \prod_{w\in\cA_\ell}
    p_w(\bY_w\mid\bF_w,\cO_w)
    P_w^\ell(d\bF_w\mid\mathscr H_\ell)
  }.
  \label{eq:filtered-factorization-integral-start}
\end{align}
Since the integrands are non-negative products of nodewise terms, Tonelli's
theorem gives
\begin{align}
  &\int_{\prod_{w\in\cA_\ell}B_w^\star}
    \prod_{w\in\cA_\ell}
    p_w(\bY_w\mid\bF_w,\cO_w)
    P_w^\ell(d\bF_w\mid\mathscr H_\ell)
  \nonumber\\
  &\quad=
  \prod_{w\in\cA_\ell}
  \int_{B_w^\star}
  p_w(\bY_w\mid\bF_w,\cO_w)
  P_w^\ell(d\bF_w\mid\mathscr H_\ell),
  \label{eq:filtered-factorization-numerator-product}
\end{align}
and
\begin{align}
  &\int
    \prod_{w\in\cA_\ell}
    p_w(\bY_w\mid\bF_w,\cO_w)
    P_w^\ell(d\bF_w\mid\mathscr H_\ell)
  \nonumber\\
  &\quad=
  \prod_{w\in\cA_\ell}
  \int
  p_w(\bY_w\mid\bF_w,\cO_w)
  P_w^\ell(d\bF_w\mid\mathscr H_\ell).
  \label{eq:filtered-factorization-denominator-product}
\end{align}
Substituting \eqref{eq:filtered-factorization-numerator-product} and
\eqref{eq:filtered-factorization-denominator-product} into
\eqref{eq:filtered-factorization-integral-start} yields
\begin{align}
  &\Pi_\ell\!\left(
    \bF_w\in B_w^\star\text{ for all }w\in\cA_\ell
    \middle|\mathscr H_\ell
  \right)
  \nonumber\\
  &\quad=
  \prod_{w\in\cA_\ell}
  \frac{
    \int_{B_w^\star}
    p_w(\bY_w\mid\bF_w,\cO_w)
    P_w^\ell(d\bF_w\mid\mathscr H_\ell)
  }{
    \int
    p_w(\bY_w\mid\bF_w,\cO_w)
    P_w^\ell(d\bF_w\mid\mathscr H_\ell)
  }
  \nonumber\\
  &\quad=
  \prod_{w\in\cA_\ell}Q_w^\ell(B_w^\star\mid\mathscr H_\ell).
  \label{eq:filtered-factorization-final-sets}
\end{align}
Since the rectangles \(\prod_{w\in\cA_\ell}B_w^\star\) generate the product
sigma-field, \eqref{eq:filtered-factorization-final-sets} proves the stated
factorisation.
\end{proof}

\subsection{Retaining routes below the observed antichain}
\label{app:retaining-routes}

We recall that disjointness conventions are given in
Appendix~\ref{app:route-conventions} and now formalise a downstream
transport mechanism.

\begin{definition}[\(\varepsilon\)-retaining route]
\label{def:retaining-route}
An admissible route
\[
  \gamma=(v_0,v_1,\dots,v_L)
\]
below \(\cA_{\ell_0}\) is \(\varepsilon\)-retaining with factor
\(\rho\in[0,1]\) if, under the predictive law propagated from
\(\Pi_{\ell_0}\),
\[
  \PP\!\left(
    \|F_{v_L}^{(a)}-F_{v_L}^{(b)}\|_2>\varepsilon
    \middle|
    \mathscr F_{\ell_0}
  \right)
  \ge
  \rho\,
  \mathbf{1}\!\left\{
    \|F_{v_0}^{(a)}-F_{v_0}^{(b)}\|_2>\varepsilon
  \right\}
  \qquad
  \Pi_{\ell_0}\text{-a.s.}
\]
\end{definition}

\begin{definition}[Additive route-retention coefficients]
\label{def:additive-route-retention-coefficients}
Let \(\gamma=(v_0,\dots,v_L)\) be an admissible route below
\(\cA_{\ell_0}\), and fix \(\varepsilon>0\). We say that \(\gamma\) satisfies
the additive \(\varepsilon\)-retention condition if, for each
\(r=1,\dots,L\), the kernel at \(v_r\), as a function of the distinguished
parent coordinate \(v_{r-1}\), admits the decomposition
\[
  K_{v_r}\bigl((u,z),(u',z')\bigr)
  =
  K_{v_r}^{(\mathrm{rest})}\bigl((u,z),(u',z')\bigr)
  +
  \varsigma_r^2\,\bar k_r(u,u')I_{d_{v_r}},
\]
where \(K_{v_r}^{(\mathrm{rest})}\) is positive semidefinite,
\(\varsigma_r>0\), and \(\bar k_r\) is a scalar correlation kernel satisfying
\[
  \sup_{\|u-u'\|_2>\varepsilon}\bar k_r(u,u')
  \le \bar r_r(\varepsilon)<1.
\]
For such a route, define
\[
  \beta_r(\varepsilon)
  :=
  2\,\Phi_{\mathrm N}\!\left(
    -\frac{\varepsilon}{
      \sqrt{2\varsigma_r^2(1-\bar r_r(\varepsilon))}
    }
  \right),
  \qquad r=1,\dots,L,
\]
and
\[
  \rho_\gamma(\varepsilon)
  :=
  \prod_{r=1}^{L}\beta_r(\varepsilon),
\]
with the empty product interpreted as one.
\end{definition}

\begin{proposition}[Additive routes are retaining]
\label{prop:additive-route-retention}
If an admissible route \(\gamma\) below \(\cA_{\ell_0}\) satisfies the
additive \(\varepsilon\)-retention condition, then \(\gamma\) is
\(\varepsilon\)-retaining with factor \(\rho_\gamma(\varepsilon)\).
\end{proposition}

\begin{proof}
Start considering \(\gamma=(v_0,\dots,v_L)\). If \(L=0\), then \(v_L=v_0\), so
\[
  \PP\!\left(
    \|F_{v_L}^{(a)}-F_{v_L}^{(b)}\|_2>\varepsilon
    \middle|\mathscr F_{\ell_0}
  \right)
  =
  \mathbf{1}\!\left\{
    \|F_{v_0}^{(a)}-F_{v_0}^{(b)}\|_2>\varepsilon
  \right\},
\]
and the claim holds with the empty product equal to one. Assume now that
\(L\ge1\). For \(r=0,\dots,L\), define
\[
  E_r
  :=
  \{\|F_{v_r}^{(a)}-F_{v_r}^{(b)}\|_2>\varepsilon\},
  \qquad
  \mathscr G_r
  :=
  \mathscr F_{\ell_0}\vee
  \sigma(\bF_{v_1},\dots,\bF_{v_r}),
\]
where \(\mathscr G_0=\mathscr F_{\ell_0}\). We first prove the one-step
inequality
\begin{equation}
  \PP(E_r\mid\mathscr G_{r-1})
  \ge
  \beta_r(\varepsilon)\mathbf{1}_{E_{r-1}},
  \qquad r=1,\dots,L.
  \label{eq:route-one-step-retention}
\end{equation}

Consider \(r\in\{1,\dots,L\}\). Write
\[
  u_a:=F_{v_{r-1}}^{(a)},
  \qquad
  u_b:=F_{v_{r-1}}^{(b)}
\]
for the two values of the distinguished parent coordinate, and write
\(z_a,z_b\) for the remaining parent coordinates of \(v_r\) at cases
\(a,b\). By admissibility, every parent of \(v_r\) other than \(v_{r-1}\) is
either a root node or belongs to \(\cU_{\le\ell_0}\). Hence
\(z_a,z_b,u_a,u_b\) are \(\mathscr G_{r-1}\)-measurable.

The regular conditional law of
\[
  \Delta_{v_r}:=F_{v_r}^{(a)}-F_{v_r}^{(b)}
\]
given \(\mathscr G_{r-1}\) is centred Gaussian with covariance
\begin{align*}
  \Gamma_{v_r}(a,b)
  &=
  K_{v_r}((u_a,z_a),(u_a,z_a))
  +
  K_{v_r}((u_b,z_b),(u_b,z_b))
  \\
  &\quad
  -K_{v_r}((u_a,z_a),(u_b,z_b))
  -K_{v_r}((u_b,z_b),(u_a,z_a)).
\end{align*}
The additive decomposition of the kernel gives
\begin{equation*}
  \Gamma_{v_r}(a,b)
  =
  \Gamma_{v_r}^{(\mathrm{rest})}(a,b)
  +
  2\varsigma_r^2\bigl(1-\bar k_r(u_a,u_b)\bigr)I_{d_{v_r}},
\end{equation*}
where \(\Gamma_{v_r}^{(\mathrm{rest})}(a,b)\succeq0\). On the event
\(E_{r-1}\), one has \(\|u_a-u_b\|_2>\varepsilon\), and therefore
\[
  \bar k_r(u_a,u_b)\le\bar r_r(\varepsilon).
\]
Consequently, on \(E_{r-1}\),
\begin{equation*}
  \Gamma_{v_r}(a,b)
  \succeq
  2\varsigma_r^2(1-\bar r_r(\varepsilon))I_{d_{v_r}}.
\end{equation*}
Let \(h\in\R^{d_{v_r}}\) be any deterministic unit vector. On \(E_{r-1}\),
the scalar conditional distribution of \(h^\top\Delta_{v_r}\) given
\(\mathscr G_{r-1}\) is Gaussian with mean zero and variance at least
\(2\varsigma_r^2(1-\bar r_r(\varepsilon))\). Since
\[
  E_r
  =
  \{\|\Delta_{v_r}\|_2>\varepsilon\}
  \supseteq
  \{|h^\top\Delta_{v_r}|>\varepsilon\},
\]
we obtain \eqref{eq:route-one-step-retention}.

It remains to iterate the one-step inequalities. Starting from the final step
and using the tower property,
\begin{align}
  \PP(E_L\mid\mathscr G_{L-2})
  &=
  \EE\!\left[
    \PP(E_L\mid\mathscr G_{L-1})
    \middle|\mathscr G_{L-2}
  \right]
  \nonumber\\
  &\ge
  \beta_L(\varepsilon)
  \EE[\mathbf{1}_{E_{L-1}}\mid\mathscr G_{L-2}]
  \nonumber\\
  &=
  \beta_L(\varepsilon)
  \PP(E_{L-1}\mid\mathscr G_{L-2}).
  \label{eq:route-backward-first}
\end{align}
Applying \eqref{eq:route-one-step-retention} to \(E_{L-1}\) gives
\begin{equation}
  \PP(E_{L-1}\mid\mathscr G_{L-2})
  \ge
  \beta_{L-1}(\varepsilon)\mathbf{1}_{E_{L-2}}.
  \label{eq:route-backward-second}
\end{equation}
Combining \eqref{eq:route-backward-first} and
\eqref{eq:route-backward-second},
\[
  \PP(E_L\mid\mathscr G_{L-2})
  \ge
  \beta_L(\varepsilon)\beta_{L-1}(\varepsilon)
  \mathbf{1}_{E_{L-2}}.
\]
Repeating this backward induction along the route yields
\begin{equation*}
  \PP(E_L\mid\mathscr F_{\ell_0})
  =
  \PP(E_L\mid\mathscr G_0)
  \ge
  \left(\prod_{r=1}^{L}\beta_r(\varepsilon)\right)
  \mathbf{1}_{E_0},
\end{equation*}
which is the retaining-route condition with factor
\(\rho_\gamma(\varepsilon)\).
\end{proof}

\subsection{Proof of Theorem~\ref{thm:stochastic-skip-main}}
\label{app:proof-stochastic-skip-main}

\begin{proof}[Proof of Theorem~\ref{thm:stochastic-skip-main}]
For each target node \(v_j\), define
\begin{equation*}
  G_j
  :=
  \{\|F_{v_j}^{(a)}-F_{v_j}^{(b)}\|_2>\varepsilon\}.
\end{equation*}
Since \(M_{\ell_1}\) is the maximum contrast over the whole antichain
\(\cA_{\ell_1}\),
\begin{equation}
  \{M_{\ell_1}>\varepsilon\}
  \supseteq
  \bigcup_{j=1}^{s}G_j.
  \label{eq:M-contains-target-union}
\end{equation}
It is therefore enough to lower-bound the predictive probability of
\(\bigcup_{j=1}^{s}G_j\).

We first work conditionally on \(\mathscr F_{\ell_0}\). By the
\(\varepsilon\)-retaining property of \(\gamma_j\),
\begin{equation}
  \PP(G_j\mid\mathscr F_{\ell_0})
  \ge
  \rho_j\mathbf{1}_{E_j},
  \qquad j=1,\dots,s.
  \label{eq:retaining-lower-Gj}
\end{equation}

We next verify conditional independence of the target events under the same
conditioning. For a route \(\gamma_j\), all non-route parents of interior
nodes are either deterministic roots or belong to \(\cU_{\le\ell_0}\), hence
are \(\mathscr F_{\ell_0}\)-measurable. Thus, under the predictive law, the
variables generated along the interior of \(\gamma_j\) are measurable with
respect to \(\mathscr F_{\ell_0}\) together with the GP modules attached to
\(\operatorname{int}(\gamma_j)\). The route interiors are pairwise disjoint,
so these collections of GP modules are disjoint. Since distinct GP modules are
mutually independent under the DAG-DGP prior, the route-generated random
elements are conditionally independent given \(\mathscr F_{\ell_0}\). The
events \(G_1,\dots,G_s\), being measurable functions of these route-generated
random elements and of \(\mathscr F_{\ell_0}\), are therefore conditionally
independent given \(\mathscr F_{\ell_0}\). Hence
\begin{align}
  \PP\!\left(
    \bigcap_{j=1}^{s}G_j^c
    \middle|\mathscr F_{\ell_0}
  \right)
  &=
  \prod_{j=1}^{s}
  \bigl(1-\PP(G_j\mid\mathscr F_{\ell_0})\bigr)
  \nonumber\\
  &\le
  \prod_{j=1}^{s}
  \bigl(1-\rho_j\mathbf{1}_{E_j}\bigr),
  \label{eq:target-complement-bound-given-F}
\end{align}
where the inequality uses \eqref{eq:retaining-lower-Gj} and
\(\rho_j\mathbf{1}_{E_j}\in[0,1]\).

We now pass from conditioning on \(\mathscr F_{\ell_0}\) to conditioning on
\(\mathscr H_{\ell_0}\). By the tower property,
\begin{align}
  \PP\!\left(
    \bigcap_{j=1}^{s}G_j^c
    \middle|\mathscr H_{\ell_0}
  \right)
  &\quad=
  \EE\!\left[
    \PP\!\left(
      \bigcap_{j=1}^{s}G_j^c
      \middle|\mathscr F_{\ell_0}
    \right)
    \middle|\mathscr H_{\ell_0}
  \right]
  \nonumber\\
  &\quad\le
  \EE_{\Pi_{\ell_0}}\!\left[
    \prod_{j=1}^{s}(1-\rho_j\mathbf{1}_{E_j})
    \middle|\mathscr H_{\ell_0}
  \right].
  \label{eq:target-complement-tower-to-H}
\end{align}
By Lemma~\ref{lem:filtered-factorization}, the source states
\(\bF_{u_1},\dots,\bF_{u_s}\) are conditionally independent given
\(\mathscr H_{\ell_0}\) under \(\Pi_{\ell_0}\). Since each \(E_j\) depends only
on \(\bF_{u_j}\), the source events \(E_1,\dots,E_s\) are conditionally
independent given \(\mathscr H_{\ell_0}\). Thus
\begin{align}
  \EE_{\Pi_{\ell_0}}\!\left[
    \prod_{j=1}^{s}(1-\rho_j\mathbf{1}_{E_j})
    \middle|\mathscr H_{\ell_0}
  \right]
  &\quad=
  \prod_{j=1}^{s}
  \EE_{\Pi_{\ell_0}}\!\left[
    1-\rho_j\mathbf{1}_{E_j}
    \middle|\mathscr H_{\ell_0}
  \right]
  \nonumber\\
  &\quad=
  \prod_{j=1}^{s}
  \left(
    1-\rho_j
    \Pi_{\ell_0}(E_j\mid\mathscr H_{\ell_0})
  \right)
  \nonumber\\
  &\quad=
  \prod_{j=1}^{s}(1-\rho_j q_j),
  \label{eq:source-product-qj}
\end{align}
because the conditional law of \(\bF_{u_j}\) given \(\mathscr H_{\ell_0}\) is
\(Q_{u_j}^{\ell_0}(\cdot\mid\mathscr H_{\ell_0})\), and hence
\(\Pi_{\ell_0}(E_j\mid\mathscr H_{\ell_0})=q_j\). Combining
\eqref{eq:target-complement-tower-to-H} and \eqref{eq:source-product-qj}
gives
\[
  \PP\!\left(
    \bigcap_{j=1}^{s}G_j^c
    \middle|\mathscr H_{\ell_0}
  \right)
  \le
  \prod_{j=1}^{s}(1-\rho_j q_j).
\]
Taking complements gives
\[
  \PP\!\left(
    \bigcup_{j=1}^{s}G_j
    \middle|\mathscr H_{\ell_0}
  \right)
  \ge
  1-
  \prod_{j=1}^{s}(1-\rho_j q_j).
\]
Finally, taking expectation under \(\Pi_{\ell_0}\) and using
\eqref{eq:M-contains-target-union} yields
\[
  \Pi_{\ell_0\to\ell_1}(M_{\ell_1}>\varepsilon)
  \ge
  \EE_{\Pi_{\ell_0}}\!\left[
    1-
    \prod_{j=1}^{s}(1-\rho_j q_j)
  \right].
\]
\end{proof}

\subsection{Source strengths for common conditional distributions}
\label{app:source-strengths}

The theorem is agnostic about how the source probabilities \(q_j\) are
obtained. We now give source-strength calculations for common nodewise
conditional distributions. The main result is a bounded-curvature source
bound for one-parameter exponential-family conditional distributions in
canonical form (e.g., \citep[Sec.~3.2]{wainwright2008}).

For a scalar source node \(u\in\cA_{\ell_0}\) and two cases \(a\neq b\), we
use the local notation for simplicity
\[
  F_a:=F_u^{(a)},
  \qquad
  F_b:=F_u^{(b)},
  \qquad
  Y_a:=Y_u^{(a)},
  \qquad
  Y_b:=Y_u^{(b)}.
\]
When the parent inputs of \(u\) at cases \(a\) and \(b\) are denoted by
\(x_a\) and \(x_b\), and \(K_u(x,x)=\tau_u^2\) for all parent inputs \(x\), we
write
\[
  \Gamma_u(a,b):=2\bigl(\tau_u^2-K_u(x_a,x_b)\bigr),
  \qquad
  \Omega_u(a,b):=2\bigl(\tau_u^2+K_u(x_a,x_b)\bigr).
\]
For a one-parameter exponential-family conditional distribution in canonical
form,
\[
  p(y\mid\theta)=h(y)\exp\{T(y)\theta-A(\theta)\},
\]
define
\(
  d_y:=T(Y_a)-T(Y_b).
\)
For \(\Gamma>0\) and curvature constants \(0\le m_{\!A}\le L_{\!A}<\infty\),
set
\[
  \alpha_L(\Gamma):=\Gamma^{-1}+L_{\!A}/2,
  \qquad
  \alpha_m(\Gamma):=\Gamma^{-1}+m_{\!A}/2.
\]

\begin{proposition}[Bounded-curvature source bound]
\label{prop:ef-curvature-refresh}
Consider a scalar observed node \(u\in\cA_{\ell_0}\) with
\(\cO_u=\{a,b\}\times\{1\}\). Assume the scalar-source notation above, and
suppose that \(\Gamma_u(a,b)>0\) and \(\Omega_u(a,b)>0\). If the nodewise
conditional distribution belongs to a one-parameter exponential family in
canonical form and satisfies
\[
  0\le m_{\!A}\le A''(\theta)\le L_{\!A}<\infty,
  \qquad \theta\in\R,
\]
then, for every \(\varepsilon>0\),
\[
  Q_u^{\ell_0}
  \left(
    |F_u^{(a)}-F_u^{(b)}|>\varepsilon
    \middle|
    \mathscr H_{\ell_0}
  \right)
  \ge
  q_{\mathrm{EF}}(\varepsilon,d_y,\Gamma_u(a,b),m_{\!A},L_{\!A}),
\]
where, for any \(\Gamma>0\),
\begin{align}
  q_{\mathrm{EF}}(\varepsilon,d_y,\Gamma,m_{\!A},L_{\!A})
  &:=
  \sqrt{\frac{\alpha_m(\Gamma)}{\alpha_L(\Gamma)}}
  \exp\!\left\{
    \frac{d_y^2}{8\alpha_L(\Gamma)}
    -
    \frac{d_y^2}{8\alpha_m(\Gamma)}
  \right\}
  \nonumber\\
  &\quad\times
  \Bigg[
    1-
    \Phi_{\mathrm N}\!\left(
      \sqrt{\alpha_L(\Gamma)}
      \left(
        \varepsilon-\frac{|d_y|}{2\alpha_L(\Gamma)}
      \right)
    \right)
    \nonumber\\
  &\qquad\qquad
    +
    \Phi_{\mathrm N}\!\left(
      -\sqrt{\alpha_L(\Gamma)}
      \left(
        \varepsilon+\frac{|d_y|}{2\alpha_L(\Gamma)}
      \right)
    \right)
  \Bigg].
  \label{eq:q-EF-def}
\end{align}
\end{proposition}

\begin{proof}
Set
\[
  \Delta:=F_a-F_b,
  \qquad
  S:=F_a+F_b,
  \qquad
  r_y:=T(Y_a)+T(Y_b).
\]
For readability, also write
\[
  \Gamma:=\Gamma_u(a,b),
  \qquad
  \Omega:=\Omega_u(a,b),
  \qquad
  \bar\alpha_L:=\alpha_L(\Gamma),
  \qquad
  \bar\alpha_m:=\alpha_m(\Gamma).
\]
Under the two-point GP law conditional on \(\mathscr H_{\ell_0}\), \(\Delta\)
and \(S\) are centred Gaussian random variables with variances \(\Gamma\) and
\(\Omega\), respectively. The constant-diagonal assumption gives
\[
  \Cov(\Delta,S\mid\mathscr H_{\ell_0})
  =
  \Var(F_a\mid\mathscr H_{\ell_0})
  -
  \Var(F_b\mid\mathscr H_{\ell_0})
  =0.
\]
Since the pair is jointly Gaussian, \(\Delta\) and \(S\) are conditionally
independent given \(\mathscr H_{\ell_0}\).

The change of variables
\[
  F_a=\frac{S+\Delta}{2},
  \qquad
  F_b=\frac{S-\Delta}{2}
\]
has constant Jacobian. Ignoring the factor \(h(Y_a)h(Y_b)\), which is
constant in \((S,\Delta)\), the joint posterior density of \((S,\Delta)\) is
proportional to
\begin{align*}
  \exp\!\left\{
    -\frac{\delta^2}{2\Gamma}
    -\frac{s^2}{2\Omega}
    +\frac{d_y}{2}\delta
    +\frac{r_y}{2}s
    -A\!\left(\frac{s+\delta}{2}\right)
    -A\!\left(\frac{s-\delta}{2}\right)
  \right\}.
\end{align*}
Thus, after integrating out \(s\), the marginal posterior density of \(\Delta\)
is proportional to
\begin{equation*}
  \exp\!\left\{
    -\frac{\delta^2}{2\Gamma}
    +
    \frac{d_y}{2}\delta
  \right\}
  R_y(\delta),
\end{equation*}
where
\begin{equation*}
  R_y(\delta)
  :=
  \int_{\R}
  \exp\!\left\{
    -\frac{s^2}{2\Omega}
    +\frac{r_y}{2}s
    -A\!\left(\frac{s+\delta}{2}\right)
    -A\!\left(\frac{s-\delta}{2}\right)
  \right\}\,ds.
\end{equation*}

The bounded-curvature assumption controls the symmetric second difference of
\(A\). By Taylor's theorem with Lagrange remainder
\citep[Theorem~5.15]{rudin1976principles}, for all \(x,c\in\R\),
\begin{equation*}
  m_{\!A}c^2
  \le
  A(x+c)+A(x-c)-2A(x)
  \le
  L_{\!A}c^2.
\end{equation*}
Taking \(x=s/2\) and \(c=\delta/2\) gives
\begin{equation}
  \exp\!\left(-\frac{L_{\!A}\delta^2}{4}\right)R_y(0)
  \le
  R_y(\delta)
  \le
  \exp\!\left(-\frac{m_{\!A}\delta^2}{4}\right)R_y(0).
  \label{eq:Ry-upper-lower}
\end{equation}
Let \(B_\varepsilon:=\{|\Delta|>\varepsilon\}\). The lower bound in
\eqref{eq:Ry-upper-lower} gives the lower bound
\begin{equation}
  R_y(0)
  \int_{|\delta|>\varepsilon}
  \exp\!\left\{
    -\frac{\bar\alpha_L}{2}\delta^2
    +
    \frac{d_y}{2}\delta
  \right\}
  d\delta
  \label{eq:ef-numerator-lower}
\end{equation}
for the unnormalised posterior mass of \(B_\varepsilon\). The upper bound in
\eqref{eq:Ry-upper-lower} gives the upper bound
\begin{equation}
  R_y(0)
  \int_{\R}
  \exp\!\left\{
    -\frac{\bar\alpha_m}{2}\delta^2
    +
    \frac{d_y}{2}\delta
  \right\}
  d\delta
  \label{eq:ef-denominator-upper}
\end{equation}
for the full normalising constant. Taking the ratio of
\eqref{eq:ef-numerator-lower} and \eqref{eq:ef-denominator-upper}, and
cancelling \(R_y(0)\), yields
\begin{equation}
  Q_u^{\ell_0}(B_\varepsilon\mid\mathscr H_{\ell_0})
  \ge
  \frac{
  \int_{|\delta|>\varepsilon}
  \exp\!\left\{-\frac{\bar\alpha_L}{2}\delta^2+\frac{d_y}{2}\delta\right\}
  d\delta
  }{
  \int_{\R}
  \exp\!\left\{-\frac{\bar\alpha_m}{2}\delta^2+\frac{d_y}{2}\delta\right\}
  d\delta
  }.
  \label{eq:ef-ratio-before-square}
\end{equation}
Completing the square,
\begin{equation}
  -\frac{\alpha}{2}\delta^2+\frac{d_y}{2}\delta
  =
  -\frac{\alpha}{2}
  \left(\delta-\frac{d_y}{2\alpha}\right)^2
  +
  \frac{d_y^2}{8\alpha}.
  \label{eq:ef-complete-square}
\end{equation}
Using \eqref{eq:ef-complete-square} in the denominator of
\eqref{eq:ef-ratio-before-square} gives
\begin{equation}
  \int_{\R}
  \exp\!\left\{-\frac{\bar\alpha_m}{2}\delta^2+\frac{d_y}{2}\delta\right\}
  d\delta
  =
  \exp\!\left\{\frac{d_y^2}{8\bar\alpha_m}\right\}
  \sqrt{\frac{2\pi}{\bar\alpha_m}}.
  \label{eq:ef-denominator-evaluated}
\end{equation}
Using \eqref{eq:ef-complete-square} in the numerator gives
\begin{align}
  &\int_{|\delta|>\varepsilon}
  \exp\!\left\{-\frac{\bar\alpha_L}{2}\delta^2+\frac{d_y}{2}\delta\right\}
  d\delta
  \nonumber\\
  &\quad=
  \exp\!\left\{\frac{d_y^2}{8\bar\alpha_L}\right\}
  \sqrt{\frac{2\pi}{\bar\alpha_L}}
  \PP(|Z_{\mathrm{EF}}|>\varepsilon),
  \label{eq:ef-numerator-evaluated}
\end{align}
where
\[
  Z_{\mathrm{EF}}
  \sim
  \cN\!\left(
    \frac{d_y}{2\bar\alpha_L},
    \frac{1}{\bar\alpha_L}
  \right).
\]
Substituting \eqref{eq:ef-denominator-evaluated} and
\eqref{eq:ef-numerator-evaluated} into \eqref{eq:ef-ratio-before-square}
gives
\begin{equation}
  Q_u^{\ell_0}(B_\varepsilon\mid\mathscr H_{\ell_0})
  \ge
  \sqrt{\frac{\bar\alpha_m}{\bar\alpha_L}}
  \exp\!\left\{
    \frac{d_y^2}{8\bar\alpha_L}
    -
    \frac{d_y^2}{8\bar\alpha_m}
  \right\}
  \PP(|Z_{\mathrm{EF}}|>\varepsilon).
  \label{eq:ef-before-tail-explicit}
\end{equation}
Writing the two-sided Gaussian tail in \eqref{eq:ef-before-tail-explicit}
explicitly, and using \(|d_y|\) because the event is symmetric, yields
\eqref{eq:q-EF-def}.
\end{proof}

\begin{proof}[Proof of Corollary~\ref{cor:gaussian-refresh-main}]
For Gaussian observations,
\[
  Y=F+\xi,
  \qquad
  \xi\sim\cN(0,\sigma_u^2),
\]
the conditional density has canonical form with
\[
  T(y)=y/\sigma_u^2,
  \qquad
  A(\theta)=\theta^2/(2\sigma_u^2).
\]
Hence \(m_{\!A}=L_{\!A}=\sigma_u^{-2}\), so the upper and lower
bounded-curvature inequalities in Proposition~\ref{prop:ef-curvature-refresh}
are equalities.

Equivalently, and more directly, the prior contrast
\[
  \Delta_u:=F_u^{(a)}-F_u^{(b)}
\]
has conditional distribution
\[
  \Delta_u\mid\mathscr H_{\ell_0}
  \sim
  \cN(0,\Gamma_u(a,b)).
\]
The observation difference satisfies
\[
  Y_u^{(a)}-Y_u^{(b)}
  =
  \Delta_u+\eta_u,
  \qquad
  \eta_u\sim\cN(0,2\sigma_u^2),
\]
with \(\eta_u\) independent of \(\Delta_u\). On \(\{\Gamma_u(a,b)>0\}\), the
one-dimensional Gaussian conditioning formula gives
\[
  \Delta_u
  \mid Y_u^{(a)},Y_u^{(b)},\mathscr H_{\ell_0}
  \sim
  \cN(m_u,v_u),
\]
where
\[
  m_u
  =
  \frac{\Gamma_u(a,b)}{\Gamma_u(a,b)+2\sigma_u^2}
  \bigl(Y_u^{(a)}-Y_u^{(b)}\bigr),
  \qquad
  v_u
  =
  \frac{2\sigma_u^2\Gamma_u(a,b)}
       {\Gamma_u(a,b)+2\sigma_u^2}.
\]
If \(\Gamma_u(a,b)=0\), the conditional prior of the contrast is degenerate
at zero, and so is the filtered contrast.

Finally, if \(Z\sim\cN(m_u,v_u)\), then
\[
  \{|Z|>\varepsilon\}
  \supseteq
  \{\operatorname{sign}(m_u)Z>\varepsilon\},
\]
with either sign used when \(m_u=0\). This gives
\[
  \PP(|Z|>\varepsilon)
  \ge
  1-
  \Phi_{\mathrm N}\!\left(
    \frac{\varepsilon-|m_u|}{\sqrt{v_u}}
  \right),
\]
which is the claimed lower bound.
\end{proof}

\paragraph{Binomial refresh.}
We say that a scalar observed source node has a binomial conditional
distribution with \(N\) trials and canonical parameter \(\theta\) if
\[
  p(y\mid\theta)
  =
  \binom{N}{y}\exp\{y\theta-N\log(1+e^\theta)\},
  \qquad y\in\{0,\dots,N\}.
\]
The Bernoulli conditional distribution corresponds to \(N=1\). This is the
canonical form of the binomial one-parameter exponential family; the Bernoulli
case is one of the standard one-parameter exponential-family examples, e.g., \citep[Table~3.1]{wainwright2008}.

\begin{corollary}[Bernoulli and binomial refresh]
\label{cor:binomial-logistic-refresh}
In the setting of Proposition~\ref{prop:ef-curvature-refresh}, suppose that
the scalar observed source node has a binomial conditional distribution with
\(N\) trials and canonical parameter \(\theta\). Then
\[
  Q_u^{\ell_0}
  \left(
    |F_u^{(a)}-F_u^{(b)}|>\varepsilon
    \middle|
    \mathscr H_{\ell_0}
  \right)
  \ge
  q_{\mathrm{EF}}
  \left(
    \varepsilon,
    Y_a-Y_b,
    \Gamma_u(a,b),
    0,
    \frac{N}{4}
  \right).
\]
\end{corollary}

\begin{proof}
Here \(T(y)=y\) and \(A(\theta)=N\log(1+e^\theta)\). Therefore
\[
  A''(\theta)
  =
  N\frac{e^\theta}{(1+e^\theta)^2}.
\]
The function \(e^\theta/(1+e^\theta)^2\) is nonnegative and bounded above by
\(1/4\), with maximum at \(\theta=0\). Hence
\[
  0\le A''(\theta)\le\frac{N}{4},
  \qquad \theta\in\R.
\]
The result follows from Proposition~\ref{prop:ef-curvature-refresh}.
\end{proof}

\newpage
\usetikzlibrary{shapes.geometric}
\section{Variational inference}

\subsection{ELBO derivation}
\label{app:elbo-derivation}
The ELBO derivation follows the Doubly Stochastic VI line of research \citep{salimbeni2017, ustyuzhaninov2020, lindinger2020}. Here we recall that, since the GP
modules are independent a priori, the augmented joint distribution factorises
as
\begin{equation}
p(\cD,\bF,\bU)
=
\prod_{w\in\cU}
p_0(\bU_w)\,
p_0\!\left(
\bF_w
\mid
\{\bF_p\}_{p\in\PA(w)},\bU_w
\right)
p_w(\bY_w\mid\bF_w,\cO_w),
\label{eq:app-augmented-joint}
\end{equation}
where, as in the main text, \(p_w(\bY_w\mid\bF_w,\cO_w)=1\) whenever
\(\cO_w=\emptyset\). The variational family is
\begin{equation}
q(\bF,\bU)
=
q(\bU)
\prod_{w\in\cU}
p_0\!\left(
\bF_w
\mid
\{\bF_p\}_{p\in\PA(w)},\bU_w
\right),
\label{eq:app-variational-family}
\end{equation}
where we have not yet imposed any restrictions upon $q(\bU)$.
Starting from the marginal likelihood, Jensen's inequality gives
\begin{align}
\log p(\cD)
&=
\log
\int
p(\cD,\bF,\bU)
\,d\bF\,d\bU
\nonumber\\
&=
\log
\int
q(\bF,\bU)
\frac{
p(\cD,\bF,\bU)
}{
q(\bF,\bU)
}
\,d\bF\,d\bU
\nonumber\\
&\ge
\mathbb E_{q(\bF,\bU)}
\left[
\log
\frac{
p(\cD,\bF,\bU)
}{
q(\bF,\bU)
}
\right]
=: \mathcal L(q).
\label{eq:app-general-elbo}
\end{align}
Substituting Eqs.~\eqref{eq:app-augmented-joint} and
\eqref{eq:app-variational-family} into Eq.~\eqref{eq:app-general-elbo}, we
obtain
\begin{align}
\mathcal L(q)
&=
\mathbb E_{q(\bF,\bU)}
\left[
\log
\frac{
\prod_{w\in\cU}
p_0(\bU_w)\,
p_0\!\left(
\bF_w
\mid
\{\bF_p\}_{p\in\PA(w)},\bU_w
\right)
p_w(\bY_w\mid\bF_w,\cO_w)
}{
q(\bU)
\prod_{w\in\cU}
p_0\!\left(
\bF_w
\mid
\{\bF_p\}_{p\in\PA(w)},\bU_w
\right)
}
\right]
\nonumber\\
&=
\mathbb E_{q(\bF,\bU)}
\left[
\sum_{w\in\cU}
\log p_w(\bY_w\mid\bF_w,\cO_w)
+
\sum_{w\in\cU}
\log p_0(\bU_w)
-
\log q(\bU)
\right]
\nonumber\\
&=
\sum_{w\in\cU}
\mathbb E_{q(\bF_w)}
\left[
\log p_w(\bY_w\mid\bF_w,\cO_w)
\right]
-
\operatorname{KL}
\left(
q(\bU)
\,\middle\|\,
\prod_{w\in\cU}p_0(\bU_w)
\right)
\nonumber\\
&=
\sum_{w:\,\cO_w\neq\emptyset}
\mathbb E_{q(\bF_w)}
\left[
\log p_w(\bY_w\mid\bF_w,\cO_w)
\right]
-
\operatorname{KL}
\left(
q(\bU)
\,\middle\|\,
\prod_{w\in\cU}p_0(\bU_w)
\right).
\label{eq:app-elbo-general-q}
\end{align}
The last equality removes the nodes without observations, since their
observational distribution factors are identically one.

Recall that for DAG-SVI we can write our distribution in Eq. \eqref{eq:clique-factorisation} equivalently as:
\begin{equation}
q(\bU)=q_\cH(\bU)
=
\mathcal N(\bm m,\bm\Lambda^{-1}),
\qquad
\bm\Lambda_{vw}=0
\quad\text{if}\quad
\{v,w\}\notin \cH.
\label{eq:app-structured-q}
\end{equation}
We denote by \(q_\cH(\bF_w)\) the marginal distribution induced by propagating
\(q_\cH(\bU)\) through the GP conditionals in the DAG:
\begin{equation}
q_\cH(\bF)
=
\int
q_\cH(\bU)
\prod_{w\in\cU}
p_0\!\left(
\bF_w
\mid
\{\bF_p\}_{p\in\PA(w)},\bU_w
\right)
d\bU .
\label{eq:app-induced-latent-law}
\end{equation}
This induced distribution is not available in closed form, so the expectation terms in
Eq.~\eqref{eq:app-elbo-general-q} are estimated by Monte Carlo, drawing
samples of \(\bF_w\) by following the topological order of the DAG: starting
from the roots and propagating samples through each non-root node's GP
conditional given the (already sampled) parent values.

The mean-field DAG-VI objective is obtained by restricting the inducing posterior to factorise across nodes,
\begin{equation}
q(\bU)
=
\prod_{w\in\cU}q_w(\bU_w),
\label{eq:app-mean-field-q}
\end{equation}
which yields
\begin{equation}
\mathcal L_{\mathrm{VI}}
=
\sum_{w:\,\cO_w\neq\emptyset}
\mathbb E_{q_{\mathrm{VI}}(\bF_w)}
\left[
\log p_w(\bY_w\mid\bF_w,\cO_w)
\right]
-
\sum_{w\in\cU}
\operatorname{KL}
\left(
q_w(\bU_w)
\,\middle\|\,
p_0(\bU_w)
\right),
\label{eq:app-mf-elbo}
\end{equation}
where \(q_{\mathrm{VI}}(\bF_w)\) is induced by
\(\prod_{v\in\cU}q_v(\bU_v)\) through the same GP conditionals.

Finally, since each observational distribution \(p_w(\bY_w\mid\bF_w,\cO_w)\)
factorises over cases, the term in
Eq.~\eqref{eq:app-elbo-general-q} decomposes over observations for any choice of
\(q(\bU)\) (whether the structured \(q_\cH\), the mean-field
\(q_{\mathrm{VI}}\), or any other family of the form
Eq.~\eqref{eq:app-variational-family}) as
\begin{equation}
\sum_{w:\,\cO_w\neq\emptyset}
\sum_{i:\,\cO_w^{(i)}\neq\emptyset}
\mathbb E_{q(F_w^{(i)})}
\left[
\log
p_w\!\left(
\bY_w^{(i)}
\mid
F_w^{(i)},\cO_w^{(i)}
\right)
\right],
\qquad
\cO_w^{(i)}:=\{j:(i,j)\in\cO_w\}.
\label{eq:app-casewise-likelihood}
\end{equation}
Two practical consequences follow. First, evaluating the ELBO requires only
the per-case marginals \(q(F_w^{(i)})\); correlations between latent values
at different cases \(i\neq i'\) never enter, so they need not be tracked
during ancestral sampling. Second, the outer sum over \(i\) admits unbiased
mini-batch estimation, recovering the standard doubly stochastic scheme of
\citep{salimbeni2017} at the case level while the structured \(q(\bU)\)
preserves cross-node coupling at the inducing level.

\subsection{Marginal ancestral sampling}
\label{app:marginal-ancestral-sampling}

The ELBO in Eq.~\eqref{eq:app-elbo-general-q} requires expectations under the
latent law induced by the variational family,
\[
q(\bF)
=
\int
q(\bU)
\prod_{w\in\cU}
p_0\!\left(
\bF_w
\mid
\{\bF_p\}_{p\in\PA(w)},\bU_w
\right)
d\bU .
\]
For chain DGPs, structured Gaussian posteriors over inducing outputs can be
marginalised recursively while retaining dependencies between latent
processes \citep{lindinger2020}. We use the same Gaussian-conditioning
principle along a topological ordering of the DAG.

Fix a topological ordering of the non-root nodes \(\cU\). We write \(v<w\)
whenever \(v\) appears before \(w\), and define
\[
\bF_{<w}:=\{\bF_v:v<w\},
\qquad
\bU_{<w}:=\{\bU_v:v<w\}.
\]
When these collections appear in matrix expressions, they are understood as
the corresponding vectorised concatenations in the chosen topological order.
We consider the global Gaussian inducing posterior
\(q(\bU) = \mathcal N(\bU;\bm m, \bm \Sigma)\), and for any subset
\(S\subseteq\cU\) we write \(\bm m_S\) and \(\bm\Sigma_{S,S}\) for the
sub-vector and sub-block of \(\bm m\) and \(\bm\Sigma\) indexed by the
inducing entries in \(\{\bU_v : v\in S\}\); cross-blocks
\(\bm\Sigma_{S,T}\) are defined analogously. In particular,
\(\bm m_w\) is the marginal mean of \(\bU_w\), \(\bm\Sigma_{ww}\) its
marginal covariance, and \(\bm\Sigma_{w,<w}\) the cross-covariance
between \(\bU_w\) and \(\bU_{<w}\).

Throughout, a hat (e.g.\ \(\widehat{\bF}_{<w}\)) denotes a realised value
of the corresponding random variable.  We denote by \(A_v\) and \(R_v\) the
finite-dimensional GP conditional mean map and residual covariance at node
\(v\), evaluated at these realised parent-state inputs, so that
\[
p_0\!\left(
\bF_v
\mid
\{\widehat{\bF}_p\}_{p\in\PA(v)},\bU_v
\right)
=
\mathcal N(\bF_v;A_v\bU_v,R_v).
\]
Equivalently, in unwhitened inducing coordinates,
\[
A_v
=
K_{v,FZ}K_{v,ZZ}^{-1},
\qquad
R_v
=
K_{v,FF}
-
K_{v,FZ}K_{v,ZZ}^{-1}K_{v,ZF},
\]
where the kernel blocks are computed from the nodewise kernel \(K_v\), the
current parent-state inputs, and the inducing locations \(\bZ_v\).

For a fixed node \(w\), define the stacked prefix matrices
\[
A_{<w}
:=
\operatorname{blockdiag}(A_v:v<w),
\qquad
R_{<w}
:=
\operatorname{blockdiag}(R_v:v<w).
\]
All \(A_v\) and \(R_v\) in these blocks are evaluated along the realised
 \(\widehat{\bF}_{<w}\).

\begin{proposition}[Marginal ancestral factorisation]
\label{prop:app-marginal-ancestral-sampler}
Consider the DAG-DGP variational family in
Eq.~\eqref{eq:app-variational-family} with global Gaussian inducing
posterior \(q(\bU)=\mathcal N(\bU;\bm m,\bm\Sigma)\). Fix a topological
ordering of \(\cU\), and define
\[
\bm C_{<w}
\;:=\;
A_{<w}\bm\Sigma_{<w,<w}A_{<w}^{\top}
+
R_{<w}.
\]
Then the induced latent law factorises along the topological order as
\[
q(\bF)
\;=\;
\prod_{w\in\cU}
q(\bF_w\mid \bF_{<w}),
\]
and each ancestral conditional is Gaussian,
\[
q(\bF_w\mid \bF_{<w})
\;=\;
\mathcal N\!\left(
\bF_w;\;
A_w\,\bm m_{w\mid <w},\;
R_w
+
A_w\,\bm\Sigma_{ww\mid <w}\,A_w^{\top}
\right),
\]
with conditional inducing moments
\[
\bm m_{w\mid <w}
\;=\;
\bm m_w
+
\bm\Sigma_{w,<w}A_{<w}^{\top}
\bm C_{<w}^{-1}
\left(
\bF_{<w}
-
A_{<w}\bm m_{<w}
\right),
\]
\[
\bm\Sigma_{ww\mid <w}
\;=\;
\bm\Sigma_{ww}
-
\bm\Sigma_{w,<w}A_{<w}^{\top}
\bm C_{<w}^{-1}
A_{<w}\bm\Sigma_{<w,w}.
\]
\end{proposition}
\noindent Proposition~\ref{prop:app-marginal-ancestral-sampler} immediately
yields an exact sampler from \(q(\bF)\): traverse \(\cU\) in topological
order and, at each node \(w\), draw \(\bF_w\) from
\(q(\bF_w\mid\widehat{\bF}_{<w})\) using the realised
\(\widehat{\bF}_{<w}\) of previously sampled latents.

\begin{proof}
Start by considering a fixed node \(w\).
By the variational family,
\[
q(\bF,\bU)
=
q(\bU)
\prod_{v\in\cU}
p_0\!\left(
\bF_v
\mid
\{\bF_p\}_{p\in\PA(v)},\bU_v
\right).
\]
If we integrate out all future nodes \(v>w\) in reverse topological order,
their conditional densities integrate to one. Hence
\[
q(\bF_w\mid \widehat{\bF}_{<w})
=
\int
p_0\!\left(
\bF_w
\mid
\{\widehat{\bF}_p\}_{p\in\PA(w)},\bU_w
\right)
q(\bU\mid \widehat{\bF}_{<w})
\,d\bU .
\]
It remains to characterise the marginal conditional law of \(\bU_w\) under
\(q(\bU\mid\widehat{\bF}_{<w})\).

Again by the variational family,
\[
q(\bU\mid\widehat{\bF}_{<w})
\propto
q(\bU)
\prod_{v<w}
p_0\!\left(
\widehat{\bF}_v
\mid
\{\widehat{\bF}_p\}_{p\in\PA(v)},\bU_v
\right).
\]
Using the definition of \(A_v\) and \(R_v\), the product of the already-visited
local conditionals can be written as
\[
\prod_{v<w}
p_0\!\left(
\widehat{\bF}_v
\mid
\{\widehat{\bF}_p\}_{p\in\PA(v)},\bU_v
\right)
=
\mathcal N\!\left(
\widehat{\bF}_{<w};
A_{<w}\bU_{<w},
R_{<w}
\right).
\]
Thus, conditional on the realised prefix, the previously sampled states act as
a linear-Gaussian observation of \(\bU_{<w}\).

With \(A_{<w}\) and \(R_{<w}\) fixed at the realised prefix, \(q(\bU\mid\widehat{\bF}_{<w})\)
coincides with the posterior in the auxiliary linear-Gaussian model where
\(\bU\sim\mathcal N(\bm m,\bm\Sigma)\) and
\(\widehat{\bF}_{<w}\mid\bU\sim\mathcal N(A_{<w}\bU_{<w},R_{<w})\); in this model, the pair
\((\bU_w,\widehat{\bF}_{<w})\) is jointly Gaussian with moments
\[
\mathbb E_q[\bU_w]=\bm m_w,
\qquad
\mathbb E_q[\widehat{\bF}_{<w}]
=
A_{<w}\bm m_{<w},
\]
\[
\operatorname{Cov}_q(\bU_w,\widehat{\bF}_{<w})
=
\bm\Sigma_{w,<w}A_{<w}^{\top},
\]
and
\[
\operatorname{Cov}_q(\widehat{\bF}_{<w})
=
A_{<w}\bm\Sigma_{<w,<w}A_{<w}^{\top}
+
R_{<w}
=
\bm C_{<w}.
\]
Gaussian conditioning \citep[Sec.~2.3, Eqs.~(2.81)--(2.82)]{bishop2006} therefore gives
\[
\mathbb E_q[\bU_w\mid \widehat{\bF}_{<w}]
=
\bm m_w
+
\bm\Sigma_{w,<w}A_{<w}^{\top}
\bm C_{<w}^{-1}
\left(
\widehat{\bF}_{<w}
-
A_{<w}\bm m_{<w}
\right),
\]
and
\[
\operatorname{Cov}_q(\bU_w\mid \widehat{\bF}_{<w})
=
\bm\Sigma_{ww}
-
\bm\Sigma_{w,<w}A_{<w}^{\top}
\bm C_{<w}^{-1}
A_{<w}\bm\Sigma_{<w,w}.
\]
Finally, at node \(w\), the local GP conditional is
\[
p_0\!\left(
\bF_w
\mid
\{\widehat{\bF}_p\}_{p\in\PA(w)},\bU_w
\right)
=
\mathcal N(\bF_w;A_w\bU_w,R_w).
\]
The remaining integration is the linear-Gaussian marginalisation
\citep[Sec.~2.3.3, Eqs.~(2.113)--(2.115)]{bishop2006}
\begin{multline*}
\int
\mathcal N(\bF_w;A_w\bU_w,R_w)
\mathcal N(\bU_w;\bm m_{w\mid <w},\bm\Sigma_{ww\mid <w})
\,d\bU_w
\\=
\mathcal N\!\left(
\bF_w;
A_w\bm m_{w\mid <w},
R_w+A_w\bm\Sigma_{ww\mid <w}A_w^\top
\right).
\end{multline*}
Therefore,
\[
q(\bF_w\mid \widehat{\bF}_{<w})
=
\mathcal N\!\left(
\bF_w;
A_w\bm m_{w\mid <w},
R_w+A_w\bm\Sigma_{ww\mid <w}A_w^\top
\right).
\]

Applying this identity at every node in the chosen topological order gives the
chain-rule factorisation
\[
q(\bF)
=
\prod_{w\in\cU}
q(\bF_w\mid\bF_{<w}).
\]
Therefore, sampling each node from the displayed conditional distribution in
topological order yields a sample from the induced marginal law \(q(\bF)\).
\end{proof}
\subsection{Practical ELBO evaluation}
\label{app:practical-elbo-evaluation}

The marginal ancestral factorisation in
Prop.~\ref{prop:app-marginal-ancestral-sampler} gives the distribution that
must be sampled in order to estimate the part of the ELBO that involves the observational distributions in logarithmic form. In
practice, this observational distribution factorises over data indexes, so the estimator only
requires the marginal law of the latent variables across DAG nodes for each
individual data index and Monte Carlo sample. Cross-index latent
correlations do not enter the likelihood estimator and are therefore not
materialised.

We describe the scalar-output case. Vector-valued nodes are obtained by replacing the interpolation
row vectors below by block interpolation matrices and the residual variances
by residual covariance blocks. We use lowercase letters for pointwise
quantities: \(a_{w,b}\) is a single interpolation vector, \(r_{w,b}\) a single
residual variance, \(\mu_{w,b}\) a scalar mean, and \(v^0_{uv,b}\) a scalar
base covariance. Their uppercase counterparts in
Prop.~\ref{prop:app-marginal-ancestral-sampler} denote stacked or matrix
quantities.

\paragraph{Whitened pointwise GP conditionals.}
As often in sparse DGP inference, we work in whitened inducing
coordinates. We keep the notation \(\bU_w\) for the whitened inducing vector
at node \(w\), so that the prior is \(\bU_w\sim\mathcal N(0,I)\). For a Monte
Carlo sample \(s\in\{1,\ldots,S\}\) and a minibatch data index
\(i\in\mathcal B\), write \(b=(s,i)\). Given the already-sampled parent
values for the same sample-index pair \(b\), the sparse GP conditional at
node \(w\) has the pointwise linear-Gaussian form
\begin{equation}
p_0\!\left(
F_{w,b}
\mid
\{F_{p,b}\}_{p\in\PA(w)},\bU_w
\right)
=
\mathcal N\!\left(
F_{w,b};
a_{w,b}^{\top}\bU_w,\,
r_{w,b}
\right),
\label{eq:app-pointwise-whitened-conditional}
\end{equation}
where \(a_{w,b}\in\mathbb R^{M_w}\) is the whitened interpolation vector and
\(r_{w,b}>0\) is the corresponding diagonal residual variance. Both
\(a_{w,b}\) and \(r_{w,b}\) depend on the current parent-state input to node
\(w\), and hence on the ancestral samples already drawn for the same
sample-index pair.

\subsubsection{Dense implementation}
\label{app:dense-implementation}

The dense implementation first materialises the full covariance
\[
\bm\Sigma=\bm\Lambda^{-1}
\]
of the structured inducing posterior
\(q_{\cH}(\bU)=\mathcal N(\bm m,\bm\Lambda^{-1})\). During the ancestral pass,
for each sample-index pair \(b\), it maintains the already-sampled latent
values \(\widehat F_{<w,b}\), their base means, and their base covariance
matrix under the Gaussian model induced by \(\bm\Sigma\).

For any two nodes \(u\) and \(v\) whose pointwise GP conditionals have already
been constructed for sample-index pair \(b\), define the base pointwise
covariance
\begin{equation}
v^0_{uv,b}
=
a_{u,b}^{\top}\bm\Sigma_{uv}a_{v,b}
+
\mathbf 1\{u=v\}\,r_{u,b},
\label{eq:app-dense-pointwise-covariance}
\end{equation}
and the base mean
\begin{equation}
\mu^0_{u,b}
=
a_{u,b}^{\top}\bm m_u .
\label{eq:app-dense-pointwise-mean}
\end{equation}
Here \(\bm\Sigma_{uv}\) is the inducing covariance block between nodes \(u\)
and \(v\). By \(v^0_{w,<w,b}\), \(v^0_{<w,<w,b}\), and
\(\mu^0_{<w,b}\), we denote the row vector, covariance matrix, and mean vector
obtained by stacking these scalar quantities over the nodes preceding \(w\)
in the chosen topological order.

At node \(w\), Gaussian conditioning gives
\begin{align}
\mu_{w,b}
&=
\mu^0_{w,b}
+
v^0_{w,<w,b}
\left(v^0_{<w,<w,b}\right)^{-1}
\left(
\widehat F_{<w,b}-\mu^0_{<w,b}
\right),
\label{eq:app-dense-pointwise-cond-mean}
\\
\sigma^2_{w,b}
&=
v^0_{ww,b}
-
v^0_{w,<w,b}
\left(v^0_{<w,<w,b}\right)^{-1}
v^0_{<w,w,b}.
\label{eq:app-dense-pointwise-cond-var}
\end{align}
The latent value is then sampled using the pathwise reparametrisation
\[
\widehat F_{w,b}
=
\mu_{w,b}
+
\sqrt{\sigma^2_{w,b}}\,
\varepsilon_{w,b},
\qquad
\varepsilon_{w,b}\sim\mathcal N(0,1).
\]
This is the pointwise implementation of
Prop.~\ref{prop:app-marginal-ancestral-sampler}: for each data index, it
samples the latent variables across DAG nodes in topological order, while
avoiding cross-index covariance terms that do not enter the likelihood
estimator.

\begin{algorithm}[t]
\caption{Dense ancestral sampling for DAG-SVI}
\label{alg:app-dense-sampling}
\begin{algorithmic}[1]
  \REQUIRE Minibatch of data indexes \(\mathcal B\), number of Monte Carlo samples \(S\), structured posterior parameters \((\bm m,\bm\Lambda)\), local GP modules
  \STATE Restrict all root design matrices, observations, and masks to indexes \(i\in\mathcal B\)
  \STATE Form the dense covariance \(\bm\Sigma=\bm\Lambda^{-1}\)
  \STATE Initialise sampled latent values \(\widehat F\gets\emptyset\)
  \STATE Initialise the running pointwise Gaussian state for each sample-index pair \(b=(s,i)\)
  \FOR{each non-root node \(w\in\cU\) in topological order}
      \STATE Build the current node inputs from root values and sampled parent values
      \STATE Compute interpolation vectors and residual variances \((a_{w,b},r_{w,b})\) from Eq.~\eqref{eq:app-pointwise-whitened-conditional}
      \STATE Compute base means and covariances using Eqs.~\eqref{eq:app-dense-pointwise-covariance}--\eqref{eq:app-dense-pointwise-mean}
      \STATE Compute \(\mu_{w,b}\) and \(\sigma^2_{w,b}\) by Gaussian conditioning, Eqs.~\eqref{eq:app-dense-pointwise-cond-mean}--\eqref{eq:app-dense-pointwise-cond-var}
      \STATE Draw \(\widehat F_{w,b}=\mu_{w,b}+\sqrt{\sigma^2_{w,b}}\varepsilon_{w,b}\), with \(\varepsilon_{w,b}\sim\mathcal N(0,1)\)
      \STATE Store \(\widehat F_{w,b}\), \(a_{w,b}\), \(r_{w,b}\), and update the running Gaussian state
  \ENDFOR
  \RETURN sampled latent values \(\widehat F\)
\end{algorithmic}
\end{algorithm}

\subsubsection{Sparse implementation}
\label{app:sparse-implementation}

The sparse implementation represents the same Gaussian posterior in canonical
form,
\begin{equation}
q_{\cH}(\bU)
\propto
\exp\!\left\{
-\frac12\bU^\top\bm\Lambda\bU+\bh^\top\bU
\right\},
\qquad
\bh=\bm\Lambda\bm m,
\label{eq:app-sparse-canonical-practical}
\end{equation}
with \(\bm\Lambda\) sparse on the chordal graph \(\cH\). The chordal
completion ensures that sparse Cholesky elimination can be carried out without
introducing fill-in outside \(\cH\) under a perfect elimination order
\citep{lauritzen1996,rue2005}.

The key observation is that conditioning on an already-sampled latent value
adds a Gaussian site involving only the corresponding inducing block. If
\(\widehat F_{v,b}\) has been sampled and
\[
\widehat F_{v,b}\mid\bU_v
\sim
\mathcal N(a_{v,b}^{\top}\bU_v,r_{v,b}),
\]
then the canonical parameters are updated by
\begin{align}
\Delta\bm\Lambda_{vv,b}
&\leftarrow
\Delta\bm\Lambda_{vv,b}
+
\frac{a_{v,b}a_{v,b}^{\top}}{r_{v,b}},
\label{eq:app-sparse-site-precision}
\\
\Delta\bh_{v,b}
&\leftarrow
\Delta\bh_{v,b}
+
\frac{a_{v,b}\widehat F_{v,b}}{r_{v,b}} .
\label{eq:app-sparse-site-info}
\end{align}
Thus the off-diagonal sparsity pattern is unchanged. For sample-index pair
\(b\), let
\[
\widetilde{\bm\Lambda}_{b}
=
\bm\Lambda+\Delta\bm\Lambda_b .
\]
Rather than solving with the full information vector
\(\bm\Lambda\bm m+\Delta\bh_b\), the implementation uses the centred identity
\begin{equation}
\widetilde{\bm\Lambda}_{b}^{-1}
(\bm\Lambda\bm m+\Delta\bh_b)
=
\bm m
+
\widetilde{\bm\Lambda}_{b}^{-1}
\left(
\Delta\bh_b-\Delta\bm\Lambda_b\bm m
\right).
\label{eq:app-sparse-centred-mean}
\end{equation}
The right-hand side in the second term is nonzero only at blocks that have
already contributed sites.

At node \(w\), define a block vector \(g_{w,b}\) by
\[
[g_{w,b}]_v
=
\begin{cases}
a_{w,b}, & v=w,\\
0, & v\neq w.
\end{cases}
\]
Let
\begin{equation}
y_b
=
\widetilde{\bm\Lambda}_{b}^{-1}
\left(
\Delta\bh_b-\Delta\bm\Lambda_b\bm m
\right),
\qquad
z_b
=
\widetilde{\bm\Lambda}_{b}^{-1}g_{w,b}.
\label{eq:app-sparse-solves}
\end{equation}
The conditional moments needed to sample \(F_{w,b}\) are then
\begin{align}
\mu_{w,b}
&=
a_{w,b}^{\top}
\left(
\bm m_w+[y_b]_w
\right),
\label{eq:app-sparse-cond-mean}
\\
\sigma^2_{w,b}
&=
r_{w,b}
+
a_{w,b}^{\top}[z_b]_w .
\label{eq:app-sparse-cond-var}
\end{align}
Both \(y_b\) and \(z_b\) are obtained by sparse triangular solves using the
current sparse Cholesky factor of \(\widetilde{\bm\Lambda}_b\). After sampling
\(\widehat F_{w,b}\), the site update in
Eqs.~\eqref{eq:app-sparse-site-precision}--\eqref{eq:app-sparse-site-info} is
added. The site is local in the precision matrix; numerically, the Cholesky
factor is updated over the affected part of the elimination tree.

\begin{algorithm}[t]
\caption{Sparse ancestral sampling for DAG-SVI}
\label{alg:app-sparse-sampling}
\begin{algorithmic}[1]
  \REQUIRE Minibatch of data indexes \(\mathcal B\), number of Monte Carlo samples \(S\), structured posterior parameters \((\bm m,\bm\Lambda)\), chordal graph \(\cH\), local GP modules
  \STATE Restrict all root design matrices, observations, and masks to indexes \(i\in\mathcal B\)
  \STATE Compute the sparse Cholesky factor of \(\bm\Lambda\) on \(\cH\)
  \STATE Initialise sampled latent values \(\widehat F\gets\emptyset\)
  \STATE For each sample-index pair \(b=(s,i)\), initialise site terms \(\Delta\bm\Lambda_b=0\), \(\Delta\bh_b=0\), and the corresponding sparse Cholesky state
  \FOR{each non-root node \(w\in\cU\) in topological order}
      \STATE Build the current node inputs from root values and sampled parent values
      \STATE Compute interpolation vectors and residual variances \((a_{w,b},r_{w,b})\) from Eq.~\eqref{eq:app-pointwise-whitened-conditional}
      \STATE Using sparse triangular solves, compute the block readouts in Eqs.~\eqref{eq:app-sparse-solves}--\eqref{eq:app-sparse-cond-var}
      \STATE Draw \(\widehat F_{w,b}=\mu_{w,b}+\sqrt{\sigma^2_{w,b}}\varepsilon_{w,b}\), with \(\varepsilon_{w,b}\sim\mathcal N(0,1)\)
      \STATE Add the site updates
      \[
      \Delta\bm\Lambda_{ww,b}
      \leftarrow
      \Delta\bm\Lambda_{ww,b}
      +
      a_{w,b}a_{w,b}^{\top}/r_{w,b},
      \qquad
      \Delta\bh_{w,b}
      \leftarrow
      \Delta\bh_{w,b}
      +
      a_{w,b}\widehat F_{w,b}/r_{w,b}
      \]
      \STATE Update the sparse Cholesky state over the affected elimination-tree region
  \ENDFOR
  \RETURN sampled latent values \(\widehat F\)
\end{algorithmic}
\end{algorithm}

\paragraph{KL term.}
Both implementations use the same inducing KL. In whitened coordinates, with
\(P=\sum_{w\in\cU}M_w\),
\begin{equation}
\operatorname{KL}\!\left(
q_{\cH}(\bU)
\,\middle\|\,
\prod_{w\in\cU}\mathcal N(0,I)
\right)
=
\frac12
\left\{
\operatorname{tr}(\bm\Sigma)
+
\bm m^\top\bm m
-
P
+
\log|\bm\Lambda|
\right\},
\qquad
\bm\Sigma=\bm\Lambda^{-1}.
\label{eq:app-practical-kl}
\end{equation}
The KL is computed from the variational precision \(\bm\Lambda\), not from the
temporary site-updated precisions \(\widetilde{\bm\Lambda}_b\) used inside the
ancestral sampler. The dense implementation obtains
\(\operatorname{tr}(\bm\Sigma)\) after explicitly forming
\(\bm\Sigma=\bm\Lambda^{-1}\). The sparse implementation obtains
\(\log|\bm\Lambda|\) from the sparse Cholesky factor of \(\bm\Lambda\), and
obtains the diagonal covariance blocks \(\bm\Sigma_{ww}\) needed for
\(\operatorname{tr}(\bm\Sigma)=\sum_w\operatorname{tr}(\bm\Sigma_{ww})\) by
selected-inverse, or Takahashi, recursions
\citep{takahashi1973,erisman1975}. Thus the sparse implementation computes
the KL without materialising the full covariance matrix.

\subsubsection{DAG-VI}
\label{app:dag-vi-practical}

DAG-VI uses the same pointwise GP conditionals and the same topological
ancestral pass, but restricts the inducing posterior to factorise across DAG
nodes,
\[
q_{\mathrm{VI}}(\bU)
=
\prod_{w\in\cU}
\mathcal N(\bU_w;\bm m_w,\bm S_w).
\]
Consequently, no conditioning on previously sampled node values is performed
at the inducing level. At node \(w\), for sample-index pair \(b\),
\begin{equation}
\mu_{w,b}
=
a_{w,b}^{\top}\bm m_w,
\qquad
\sigma^2_{w,b}
=
r_{w,b}
+
a_{w,b}^{\top}\bm S_w a_{w,b}.
\label{eq:app-dag-vi-pointwise}
\end{equation}
The sampled parent values still enter downstream GP inputs, so DAG-VI
propagates marginal uncertainty through the DAG, but it removes posterior
coupling between distinct node mechanisms.

\begin{algorithm}[t]
\caption{Ancestral sampling for DAG-VI}
\label{alg:app-dag-vi-sampling}
\begin{algorithmic}[1]
  \REQUIRE Minibatch of data indexes \(\mathcal B\), number of Monte Carlo samples \(S\), mean-field posterior \(\prod_{w\in\cU}\mathcal N(\bm m_w,\bm S_w)\), local GP modules
  \STATE Restrict all root design matrices, observations, and masks to indexes \(i\in\mathcal B\)
  \STATE Initialise sampled latent values \(\widehat F\gets\emptyset\)
  \FOR{each non-root node \(w\in\cU\) in topological order}
      \STATE Build the current node inputs from root values and sampled parent values
      \STATE Compute interpolation vectors and residual variances \((a_{w,b},r_{w,b})\)
      \STATE Compute \(\mu_{w,b}\) and \(\sigma^2_{w,b}\) from Eq.~\eqref{eq:app-dag-vi-pointwise}
      \STATE Draw \(\widehat F_{w,b}=\mu_{w,b}+\sqrt{\sigma^2_{w,b}}\varepsilon_{w,b}\), with \(\varepsilon_{w,b}\sim\mathcal N(0,1)\)
  \ENDFOR
  \RETURN sampled latent values \(\widehat F\)
\end{algorithmic}
\end{algorithm}

\subsubsection{Stochastic ELBO estimator}
\label{app:stochastic-elbo-estimator}

Let
\[
\ell_w^{(i)}(f)
:=
\log p_w\!\left(
Y_w^{(i)}
\mid
f,\cO_w^{(i)}
\right),
\]
with \(\ell_w^{(i)}(f)=0\) whenever \(\cO_w^{(i)}=\emptyset\). For a
minibatch of data indexes \(\mathcal B\subset[n]\) of size \(B\), sampled
uniformly, the estimator is
\begin{equation}
\widehat{\mathcal L}_{\mathrm{obs}}
=
\frac{n}{B}
\frac1S
\sum_{s=1}^{S}
\sum_{i\in\mathcal B}
\sum_{w\in\cU}
\ell_w^{(i)}
\left(
\widehat F_{w}^{(s,i)}
\right).
\label{eq:app-minibatch-observation-estimator}
\end{equation}
The stochastic ELBO estimate is
\begin{equation}
\widehat{\mathcal L}
=
\widehat{\mathcal L}_{\mathrm{obs}}
-
\mathcal K,
\label{eq:app-stochastic-elbo-final}
\end{equation}
where \(\mathcal K\) is the inducing KL in
Eq.~\eqref{eq:app-practical-kl} for DAG-SVI, or the sum of nodewise KL terms
for DAG-VI. If additional per-index latent variables are used, their KL terms
are added to \(\mathcal K\), with the corresponding minibatch scaling. For
Gaussian distributions, the expectation of \(\ell_w^{(i)}\) under the final
Gaussian conditional can be evaluated analytically; this Rao--Blackwellised
variant reduces Monte Carlo variance but leaves the objective unchanged.

\begin{algorithm}[t]
\caption{One stochastic training step}
\label{alg:app-training-step}
\begin{algorithmic}[1]
  \REQUIRE Dataset \(\cD\), minibatch size \(B\), Monte Carlo samples \(S\), variational family \(\mathsf{method}\in\{\mathrm{DAG\mbox{-}VI},\mathrm{DAG\mbox{-}SVI\ dense},\mathrm{DAG\mbox{-}SVI\ sparse}\}\)
  \STATE Draw a minibatch of data indexes \(\mathcal B\subset[n]\) uniformly, typically without replacement within an epoch
  \IF{\(\mathsf{method}=\mathrm{DAG\mbox{-}SVI\ dense}\)}
      \STATE Draw ancestral samples using Algorithm~\ref{alg:app-dense-sampling}
  \ELSIF{\(\mathsf{method}=\mathrm{DAG\mbox{-}SVI\ sparse}\)}
      \STATE Draw ancestral samples using Algorithm~\ref{alg:app-sparse-sampling}
  \ELSE
      \STATE Draw ancestral samples using Algorithm~\ref{alg:app-dag-vi-sampling}
  \ENDIF
  \STATE Estimate the observation term using Eq.~\eqref{eq:app-minibatch-observation-estimator}
  \STATE Compute the analytic KL term \(\mathcal K\)
  \STATE Form \(\widehat{\mathcal L}=\widehat{\mathcal L}_{\mathrm{obs}}-\mathcal K\)
  \STATE Update variational parameters, inducing locations, and kernel hyperparameters by a stochastic gradient step
\end{algorithmic}
\end{algorithm}

\subsection{Chordal completion and elimination order}
\label{app:chordal-completion-order}

We construct the chordal graph \(\cH\) over the inducing-variable blocks associated with the non-root ancestors of observed nodes. Specifically, we consider
\[
\operatorname{An}
\left(
\{w\in\cU:\cO_w\neq\emptyset\}
\right)
\cap\cU,
\]
where ancestors include the observed nodes themselves. We then moralise this induced DAG, adding undirected edges between each retained parent--child pair and between all retained co-parents. Root nodes are excluded because they do not carry inducing points. Non-root nodes outside this ancestral set are retained as isolated vertices.

When the resulting moral graph is not chordal, we chordally complete it using the MCS-M minimal-triangulation heuristic \citep{berry2004}. The fill edges returned by MCS-M give an inclusion-minimal triangulation, meaning that no added edge can be removed while preserving chordality. This is a local minimality guarantee and does not imply minimum treewidth,
minimum maximum-clique size, or minimum computational cost. Finding an
optimal chordal completion under such criteria is computationally intractable
in general, with minimum fill-in and bounded-treewidth formulations being
classical NP-complete problems \citep{yannakakis1981,arnborg1987}. If the moral graph is already chordal, no fill edges are added. Clique sizes are therefore induced by moralisation and chordal completion. Alternative completion or ordering heuristics designed to control clique size could be incorporated within the same general construction.

After constructing \(\cH\), we compute a deterministic perfect elimination order (PEO) \(\pi=(v_1,\ldots,v_J)\) using maximum cardinality search on \(\cH\) \citep{tarjan1984}, with ties broken by the fixed declaration order of the non-root nodes. This PEO is computed on the completed graph. For each \(v_i\), its later neighbours in \(\pi\) form a clique, and the corresponding elimination front consists of \(v_i\) together with these later neighbours. DAG-SVI-sparse eliminates blocks in ascending PEO and performs selected inversion and triangular sampling in reverse PEO.
\subsection{Computational cost}
\label{app:computational-cost}

We report the leading cost of one stochastic ELBO evaluation. Let $J = |\cU|$ be the number of non-root inducing blocks, let $M = \max_w M_w$, and let $K = SB$, where $S$ is the number of Monte Carlo samples and $B$ is the minibatch size. For DAG-SVI-sparse, let $c$ be the maximum number of inducing blocks in a clique of the chordal graph $\cH$. We assume comparable block sizes and scalar node outputs.

DAG-VI factorises across inducing blocks, giving $\mathcal{O}(JM^3 + KJM^2)$ time and $\mathcal{O}(JM^2)$ memory.

DAG-SVI-dense materialises the full covariance $\bm{\Sigma} = \bm{\Lambda}^{-1}$. Its cost is $\mathcal{O}((JM)^3 + KJ^2 M^2 + KJ^3)$ time and $\mathcal{O}((JM)^2)$ memory. The first term is the global dense Gaussian computation, while the remaining terms come from dense cross-node covariance contractions and Gaussian conditioning during the ancestral pass. When \(J\le M^2\), the last term is dominated by the second, which gives the expression reported in the main text.

DAG-SVI-sparse instead exploits the chordal precision structure and never materialises $\bm{\Lambda}^{-1}$. Sparse Cholesky, log-determinants, and the selected-inverse/Takahashi recursions used for the KL trace term cost $\mathcal{O}(Jc^2 M^3)$ time and $\mathcal{O}(Jc M^2)$ memory. The pointwise GP propagation through the DAG contributes $\mathcal{O}(KJM^2)$. The collapsed sampler also performs local Gaussian conditioning updates along the elimination tree. If $\rho_t$ denotes the number of Cholesky block columns, or supernodes, affected at ancestral step $t$, and $R_{\cH} := \sum_{t=1}^{J} \rho_t$, then these updates add $\mathcal{O}(K c M^2 R_{\cH})$ to the sparse cost. Hence the sparse implementation costs $\mathcal{O}(Jc^2 M^3 + KJM^2 + K c M^2 R_{\cH})$, with persistent memory $\mathcal{O}(Jc M^2)$, up to minibatch-specific temporary storage.

This expression shows how DAG-SVI-sparse benefits from local graph structure. In particular, the base Gaussian computation scales with the maximal clique size $c$, while the sampling overhead is governed by the elimination-tree update profile $R_{\cH}$. In locally sparse DAGs, $c$ is small and $R_{\cH}$ grows slowly with $J$. For example, in the balanced branching-tree setting of Fig.~\ref{fig:branching-elbo-scaling}, each node has at most one parent, so the moralised graph remains a tree and $c = 2$. With a balanced elimination profile, the affected update regions grow with the tree depth, giving $R_{\cH} = \mathcal{O}(J \log J)$, and hence $\mathcal{O}(JM^3 + KJ \log(J) M^2)$, which is substantially below the dense scaling as $J$ grows.

Finally, recall that for a chain DGP, $\cH$ is a path, $\bm\Lambda$ is block-tridiagonal, and our structural family specialises to \citep[Sec.~4.1]{ustyuzhaninov2020}. The corresponding marginalised conditionals can then be computed analytically, saving computation.
\subsection{Explaining away}
\label{app:explaining-away}

Compared with a standard chain DGP, a DAG-DGP can exhibit posterior coupling
between independent mechanisms that share an observed child. This behaviour is
the classical explaining-away effect \citep{pearl1988,lauritzen1996}. The
linear Gaussian case gives a simple illustration.

\paragraph{Linear Gaussian case.}
Consider the scalar collider
\[
    A \sim \mathcal N(0,\sigma_A^2),
    \qquad
    B \sim \mathcal N(0,\sigma_B^2),
    \qquad
    C\mid A,B \sim \mathcal N(\alpha A+\beta B,\sigma_C^2),
\]
with \(A\) and \(B\) independent a priori. Once a value \(c\) of the child
variable is observed, the posterior over \((A,B)\) is Gaussian with precision
\begin{equation}
    Q_{AB\mid c}
    =
    \begin{pmatrix}
        \sigma_A^{-2}+\alpha^2\sigma_C^{-2}
        &
        \alpha\beta\sigma_C^{-2}
        \\
        \alpha\beta\sigma_C^{-2}
        &
        \sigma_B^{-2}+\beta^2\sigma_C^{-2}
    \end{pmatrix}.
    \label{eq:app-linear-gaussian-collider-precision}
\end{equation}
Inverting \eqref{eq:app-linear-gaussian-collider-precision} gives
\begin{equation}
    \operatorname{Cov}(A,B\mid c)
    =
    -
    \frac{
        \alpha\beta\,\sigma_A^2\sigma_B^2
    }{
        \sigma_C^2+\alpha^2\sigma_A^2+\beta^2\sigma_B^2
    }.
    \label{eq:app-linear-gaussian-collider-covariance}
\end{equation}
Thus, whenever \(\alpha\beta>0\), the two parents become negatively correlated
a posteriori. Intuitively, once one branch explains a substantial part of the
observed value \(c\), less support is needed from the other. The observation at
the child therefore induces posterior dependence between parents that are
independent a priori.

\paragraph{Explaining away in DAG-DGPs.}
The same mechanism appears in the DAG-DGP collider
\[
    w_1 \to w_3 \leftarrow w_2 ,
\]
where \(w_1\) and \(w_2\) are two parent mechanisms and observations are
attached only to the child node \(w_3\). For \(i\in[n]\), write the latent
recursion as
\begin{equation}
    \bF_{w_1}^{(i)}
    =
    f_{w_1}\!\left(\bF_{\PA(w_1)}^{(i)}\right),
    \qquad
    \bF_{w_2}^{(i)}
    =
    f_{w_2}\!\left(\bF_{\PA(w_2)}^{(i)}\right),
    \qquad
    \bF_{w_3}^{(i)}
    =
    f_{w_3}\!\left(
        \bF_{w_1}^{(i)},\bF_{w_2}^{(i)}
    \right).
    \label{eq:app-collider-dagdgp-recursion}
\end{equation}
For notational simplicity, suppressing the root inputs and any other possible upstream
variables outside the collider, the prior latent law factorises as
\[
    p_0(\bF_{w_1},\bF_{w_2},\bF_{w_3})
    =
    p_0(\bF_{w_1})\,
    p_0(\bF_{w_2})\,
    p_0(\bF_{w_3}\mid \bF_{w_1},\bF_{w_2}).
\]
Suppose that observations are available only at the child, through a nodewise
conditional distribution \(p_{w_3}(Y_{w_3}\mid \bF_{w_3},\mathcal O_{w_3})\).
The posterior marginal over the two parent branches is then
\begin{equation}
\begin{aligned}
    p(\bF_{w_1},\bF_{w_2}\mid Y_{w_3})
    &\propto
    p_0(\bF_{w_1})\,
    p_0(\bF_{w_2})
    \\
    &\quad \times
    \int
    p_{w_3}(Y_{w_3}\mid \bF_{w_3},\mathcal O_{w_3})\,
    p_0(\bF_{w_3}\mid \bF_{w_1},\bF_{w_2})
    \,d\bF_{w_3}.
\end{aligned}
\label{eq:app-collider-parent-posterior}
\end{equation}
The integral in \eqref{eq:app-collider-parent-posterior} depends jointly on
\(\bF_{w_1}\) and \(\bF_{w_2}\), and therefore does not factorise in general.
This is the DAG-DGP analogue of explaining away: once the observed child is
partly accounted for by one branch, the posterior mass over the other branch
shifts accordingly.

The effect is especially transparent when the child kernel contains separate
contributions from the two parents, for instance under additive fusion.
In this case the child receives two distinct nonlinear contributions whose
combined effect is constrained by the observations at \(w_3\). The posterior
therefore induces dependence between the two parent branches even though the
GP modules \(f_{w_1}\) and \(f_{w_2}\) are independent under the prior.

\paragraph{DAG-VI cannot retain explaining away.}\label{par:MFcannot_ea}
We now show that DAG-VI cannot represent the posterior coupling required by
explaining away. Consider again the collider discussed in the main paper
\[
    w_1 \to w_3 \leftarrow w_2 ,
\]
and assume that observations are available only at the child node \(w_3\).
Under DAG-VI, the inducing posterior factorises across nodes,
\[
    q_{\mathrm{DAG\text{-}VI}}(\bU)
    =
    \prod_{w\in\mathcal U} q_w(\bU_w).
\]
Together with the DAG-DGP conditionals, this gives
\begin{equation}
\begin{aligned}
    &q_{\mathrm{DAG\text{-}VI}}
    (\bF_{w_1},\bU_{w_1},
     \bF_{w_2},\bU_{w_2},
     \bF_{w_3},\bU_{w_3})
    \\
    &\quad =
    p(\bF_{w_1}\mid \bU_{w_1})q_{w_1}(\bU_{w_1})\,
    p(\bF_{w_2}\mid \bU_{w_2})q_{w_2}(\bU_{w_2})
    \\
    &\qquad \times
    p(\bF_{w_3}\mid \bU_{w_3},\bF_{w_1},\bF_{w_2})
    q_{w_3}(\bU_{w_3}).
\end{aligned}
\label{eq:app-dagvi-collider-family}
\end{equation}
Marginalising the child variables yields
\begin{equation}
\begin{aligned}
    &q_{\mathrm{DAG\text{-}VI}}
    (\bF_{w_1},\bU_{w_1},\bF_{w_2},\bU_{w_2})
    \\
    &\quad =
    p(\bF_{w_1}\mid \bU_{w_1})q_{w_1}(\bU_{w_1})\,
    p(\bF_{w_2}\mid \bU_{w_2})q_{w_2}(\bU_{w_2})
    \\
    &\qquad \times
    \int q_{w_3}(\bU_{w_3})\,d\bU_{w_3}
    \int
    p(\bF_{w_3}\mid \bU_{w_3},\bF_{w_1},\bF_{w_2})
    \,d\bF_{w_3}.
\end{aligned}
\end{equation}
Both integrals are equal to one. Hence
\begin{equation}
\begin{aligned}
    &q_{\mathrm{DAG\text{-}VI}}
    (\bF_{w_1},\bU_{w_1},\bF_{w_2},\bU_{w_2})
    \\
    &\quad =
    q_{\mathrm{DAG\text{-}VI}}(\bF_{w_1},\bU_{w_1})\,
    q_{\mathrm{DAG\text{-}VI}}(\bF_{w_2},\bU_{w_2}).
\end{aligned}
\end{equation}
In particular,
\begin{equation}
    q_{\mathrm{DAG\text{-}VI}}(\bF_{w_1},\bF_{w_2})
    =
    q_{\mathrm{DAG\text{-}VI}}(\bF_{w_1})\,
    q_{\mathrm{DAG\text{-}VI}}(\bF_{w_2}).
    \label{eq:app-dagvi-no-parent-coupling}
\end{equation}
Thus DAG-VI preserves marginal uncertainty along each branch, but assigns no
posterior dependence between the two parents. It therefore cannot represent the
explaining-away dependence induced by observations at \(w_3\).

\paragraph{A topological directed factorisation is also insufficient.}
One might instead consider a directed variational posterior that follows the
topological order of the DAG,
\begin{equation}
    q_{\mathrm{top}}(\bU)
    =
    \prod_{w\in\mathcal U}
    q_w\!\left(\bU_w\mid \{\bU_p\}_{p\in\PA(w)}\right),
    \label{eq:app-topological-variational-family}
\end{equation}
which is the direct DAG analogue of autoregressive or layer-wise posterior
factorisations used in chain settings \citep{ustyuzhaninov2020, ober2021}. This approach is also unable to capture explaining-away. For the same collider \(w_1\to w_3\leftarrow w_2\),
\eqref{eq:app-topological-variational-family} gives
\[
    q_{\mathrm{top}}(\bU_{w_1},\bU_{w_2},\bU_{w_3})
    =
    q_{w_1}(\bU_{w_1})\,
    q_{w_2}(\bU_{w_2})\,
    q_{w_3}(\bU_{w_3}\mid \bU_{w_1},\bU_{w_2}).
\]
Marginalising the child gives
\begin{equation}
\begin{aligned}
    q_{\mathrm{top}}(\bU_{w_1},\bU_{w_2})
    &=
    q_{w_1}(\bU_{w_1})q_{w_2}(\bU_{w_2})
    \int
    q_{w_3}(\bU_{w_3}\mid \bU_{w_1},\bU_{w_2})
    \,d\bU_{w_3}
    \\
    &=
    q_{w_1}(\bU_{w_1})q_{w_2}(\bU_{w_2}).
\end{aligned}
\label{eq:app-topological-no-parent-coupling}
\end{equation}
The child conditional can model how \(\bU_{w_3}\) depends on its parents, but
it does not create marginal posterior dependence between the co-parents once
the child is integrated out. Explaining away requires exactly such dependence:
conditioning on observations at \(w_3\) couples the plausible contributions of
\(w_1\) and \(w_2\).

This is why the structured approximation is based instead on the moralised
ancestral graph. For the collider \(w_1\to w_3\leftarrow w_2\), moralisation
adds the co-parent edge \(\{w_1,w_2\}\), allowing the approximate posterior over
inducing variables to retain the posterior dependence induced by the observed
child.

\section{Stochastic Deep Gaussian Processes over Graphs as a Special Case of DAG-DGP}
\label{app:DGPG_app}

Stochastic Deep Gaussian Processes over Graphs (DGPG) \citep{li2020} were
introduced for a modelling task different from ours, namely learning maps
between input and output signals defined on the vertices of a fixed graph.
In that setting, the graph indexes the components of each signal and
specifies which neighbouring components are used by each graph-indexed GP
module.

We show that this construction is nevertheless contained in the DAG-DGP
framework. After unrolling the base graph across depth, any DGPG model can
be represented as a DAG-DGP on a layered DAG, with deterministic roots
\(F^0=X\), concatenation fusion rule at each non-root node, and observations
restricted to the terminal layer. Under this identification, the DGPG ELBO corresponds to the DAG-VI
objective. This observation positions DGPG as one particular
graph-structured architecture within the broader DAG-DGP class. Thus, the
DAG-DGP framework is strictly more general at the modelling level, since it
allows arbitrary DAG structures, heterogeneous fusion rules, and
observations at arbitrary nodes. It is also more general at the inferential
level, since the same embedded DGPG model can be equipped with our DAG-SVI
objective, which applies directly to this model.

\subsection{The DGPG model}
\label{app:dgpg-model}
\paragraph{Setup.}
Following \citep{li2020}, the dataset is
\(\mathcal D=\{G,\Psi,\Phi\}\), where \(G=\langle V,E\rangle\) is a graph
with vertices \(V=\{v_1,\dots,v_K\}\) and edges \(E\subseteq V\times V\).
For \(v_k\in V\), we write
\(
  \PA_G(k):=\{j:(v_j,v_k)\in E\}
\)
for the parent indices of \(v_k\) in the base graph, with self-loops allowed; we assume every vertex has at least one parent, i.e.\ \(\PA_G(k)\neq\emptyset\).
Each input \(\psi\in\Psi\) is a graph signal
\(\psi:V\to\mathbb R^{d_{\mathrm{in}}}\), and each output
\(\phi\in\Phi\) is a graph signal
\(\phi:V\to\mathbb R^{d_{\mathrm{out}}}\). The learning task is to infer a
map \(h:\Psi\to\Phi\), taking an input graph signal to an output graph
signal.
With \(N\) training signals, the inputs and outputs are stacked row-wise as
\(X=(x_1,\dots,x_N)^\top\) and \(Y=(y_1,\dots,y_N)^\top\), where
\(x_i\in\mathbb R^{K d_{\mathrm{in}}}\) and
\(y_i\in\mathbb R^{K d_{\mathrm{out}}}\) concatenate the vertex-wise
features of the \(i\)-th input and output signals.
\paragraph{Generative model.}
DGPG stacks $L$ layers of graph-indexed GP mappings. For layer $l$ the latent matrix is $F^{l}\in\mathbb{R}^{N\times Kd_l}$ (with per-node dimension $d_l$, $d_0=d_{\mathrm{in}}$, $d_L=d_{\mathrm{out}}$) and $F^{0}=X$; each layer is augmented with $M$ inducing inputs $Z^{l}$ and inducing outputs $U^{l}$. For a matrix \(M\in\{X,Y,F,Z,U\}\), \(M^{l,k}\) denotes the sub-block of
layer \(l\) associated with vertex \(v_k\), \(M^{l,\PA_G(k)}\) the
concatenated sub-block at the base-graph parents of \(v_k\), and
\(M^{l,k}_i\) the \(i\)-th row of \(M^{l,k}\). Assuming Gaussian GP priors and inducing outputs $U^{l,k}$ that are independent across layers and vertices, the joint density factorises over observations, layers, and vertices as
\begin{equation}
\label{eq:dgpg-joint}
\begin{aligned}
p\big(Y,\{F^{l,k},U^{l,k}\}_{l,k}\big)
&=
\prod_{n=1}^{N}\prod_{k=1}^{K}
p\big(y_n^{k}\mid F^{L,k}_n\big)
\\
&\quad\times
\prod_{l=1}^{L}\prod_{k=1}^{K}
p\big(F^{l,k}\mid U^{l,k};\,F^{l-1,\PA(k)},Z^{l-1,\PA(k)}\big)\,
p\big(U^{l,k};\,Z^{l-1,\PA(k)}\big) .
\end{aligned}
\end{equation}
with $F^{0}=X$. Crucially, the GP module at node $k$ in layer $l$ acts on the \emph{concatenation} $F^{l-1,\PA(k)}$ of its graph-parent signals from the previous layer. The notation of \citep{li2020} follows the standing convention of
\citep{salimbeni2017}, according to which the semicolon separates fixed
inputs and kernel-design quantities from random quantities being conditioned
on. In our DAG-DGP notation, this fixed dependence is kept implicit.

\paragraph{Variational family and recursive sampling.}
DGPG uses the doubly-stochastic family of \citep{salimbeni2017}, retaining the GP conditionals and a factorised Gaussian inducing posterior,
\begin{equation}
\label{eq:dgpg-q}
q\big(\{F^{l,k},U^{l,k}\}_{l,k}\big)
=\prod_{l=1}^{L}\prod_{k=1}^{K}
 p\big(F^{l,k}\mid U^{l,k};\,F^{l-1,\PA(k)},Z^{l-1,\PA(k)}\big)\,q\big(U^{l,k}\big),
\end{equation}
where \(q(U^{l,k})\!:=\!\mathcal{N}\big(U^{l,k}\mid m^{l,k},S^{l,k}\big)\).
Marginalising each $U^{l,k}$ yields per-node Gaussian conditionals $q(F^{l,k}\mid F^{l-1,\PA(k)})=\mathcal{N}(F^{l,k}\mid\tilde\mu^{l,k},\tilde\Sigma^{l,k})$ whose moments depend only on the parent states. Consequently the terminal-layer marginal $q(F^{L,k}_i)$ depends only on the ancestors of $\langle L,k\rangle$ and can be drawn recursively across depth by the reparameterisation,
\begin{equation}
\label{eq:dgpg-reparam}
\widehat{F}^{l,k}_i=\mu_{m^{l,k},Z^{l-1,\PA(k)}}\big(\widehat{F}^{l-1,\PA(k)}_i\big)
 +\epsilon^{l,k}_i\odot\sqrt{\Sigma_{S^{l,k},Z^{l-1,\PA(k)}}\big(\widehat{F}^{l-1,\PA(k)}_i,\widehat{F}^{l-1,\PA(k)}_i\big)},
\end{equation}
where \(\epsilon^{l,k}_i\sim\mathcal{N}(0,I_{d_l})\)
and $\mu$ and $\Sigma$ are the sparse variational predictive mean and covariance.

\paragraph{Evidence lower bound.}
With \eqref{eq:dgpg-joint}--\eqref{eq:dgpg-q}, the DGPG ELBO is
\begin{equation}
\label{eq:dgpg-elbo}
\mathcal{L}_{\mathrm{DGPG}}
=\sum_{n=1}^{N}\sum_{k=1}^{K}\mathbb{E}_{q(F^{L,k}_n)}\!\big[\log p(y_n^{k}\mid F^{L,k}_n)\big]
-\sum_{l=1}^{L}\sum_{k=1}^{K}\mathrm{KL}\big[\,q(U^{l,k})\,\big\|\,p(U^{l,k};Z^{l-1,\PA(k)})\,\big].
\end{equation}

\subsection{DGPG as a particular case of DAG-DGP}
\label{app:dgpg-embedding}

Proposition~\ref{prop:dgpg-dagdgp} shows that the depth-unrolled graph $\widetilde{G}$ is a layered DAG and that DGPG coincides with the DAG-DGP supported on it, under a concatenation fusion rule and observations available only at the terminal layer. We first define the depth-unrolling of the base graph and the corresponding notion of layered DAG. We then prove that the unrolling of any base graph is layered. Finally, we show that, under this specialization, the DGPG augmented joint and ELBO are recovered exactly as the corresponding DAG-DGP joint and DAG-VI objective.
Figure~\ref{fig:dgpg-three-panels} shows an example of this. Throughout this subsection, parent sets in the original DGPG base graph
\(G\) are denoted by \(\PA_G\), whereas parent sets in the depth-unrolled
DAG \(\widetilde G\) are denoted by \(\PA_{\widetilde G}\).

\begin{definition}[Depth-unrolling of the base graph]
\label{def:unrolled}
Given a DGPG model with base graph $G=\langle V,E\rangle$ and depth $L$, the \emph{depth-unrolling} of $G$ is the directed graph $\widetilde{G}=(\widetilde{V},\widetilde{E})$ with
\[
\widetilde{V}=\big\{(l,k):0\le l\le L,\ k\in V\big\},
\qquad
\widetilde{E}=\big\{\,(l-1,j)\to(l,k):\ 1\le l\le L,\ j\in\PA_G(k)\,\big\},
\]
so that
\[
\PA_{\widetilde G}\big((l,k)\big)=\{(l-1,j):j\in\PA_G(k)\},
\qquad l\ge 1.
\]
We call $\widetilde{\mathcal{R}}=\{(0,k):k\in V\}$ the roots, $\widetilde{\mathcal{U}}=\{(l,k):1\le l\le L,\ k\in V\}$ the non-roots, and $\mathcal{A}_l=\{(l,k):k\in V\}$ the $l$-th depth slice.
\end{definition}

\begin{definition}[Layered DAG]
\label{def:layered-dag}
A DAG $G=(V,E)$ is \emph{layered} if there exists a partition into antichains
\[
V=A_0\sqcup A_1\sqcup\cdots\sqcup A_L
\]
such that every edge connects consecutive antichains:
\[
E\subseteq \bigcup_{\ell=0}^{L-1} A_\ell\times A_{\ell+1}.
\]
\end{definition}

\begin{lemma}[The depth-unrolling of any DGPG base graph is a layered DAG]
\label{lem:layered}
For any DGPG base graph $G$ and depth $L$, the depth-unrolling $\widetilde{G}$ of Definition~\ref{def:unrolled} is a layered DAG.
\end{lemma}
\begin{proof}
By construction, every edge of $\widetilde{G}$ takes the form
\[
(l-1,j)\to(l,k),
\qquad 1\le l\le L,\quad j\in\PA(k),
\]
and therefore increases the depth index by one. Hence every directed path strictly increases the depth index, so $\widetilde{G}$ is acyclic and no two vertices in the same slice $\mathcal{A}_l$ are comparable. Thus each $\mathcal{A}_l$ is an antichain.

Moreover, the antichains satisfy
\[
\widetilde{V}
=
\mathcal{A}_0\sqcup\mathcal{A}_1\sqcup\cdots\sqcup\mathcal{A}_L,
\]
and, again by construction,
\[
\widetilde{E}
\subseteq
\bigcup_{l=1}^{L}\mathcal{A}_{l-1}\times\mathcal{A}_l.
\]
Therefore $\mathcal{A}_0,\ldots,\mathcal{A}_L$ form a partition of $\widetilde V$ into antichains satisfying the consecutive-layer condition. Hence $\widetilde{G}$ is a layered DAG.
\end{proof}

\begin{proposition}[DGPG is a DAG-DGP and DAG-VI recovers its ELBO]
\label{prop:dgpg-dagdgp}
For any DGPG base graph $G$, the DGPG model of \citep{li2020} is the DAG-DGP on the layered DAG $\widetilde{G}$ of Definition~\ref{def:unrolled}, with deterministic roots $F^{0}=X$, observations supported only on the terminal antichain-layer $\mathcal{A}_L$, and concatenation fusion at each non-root node $w=(l,k)$, meaning that the parent tuple is treated as a single concatenated input and the local kernel is obtained by applying a standard kernel to this concatenated parent state.
Under this identification, the DGPG joint \eqref{eq:dgpg-joint} is the augmented DAG-DGP joint, and its objective equals the corresponding DAG-VI bound, that is,
\(
\mathcal{L}_{\mathrm{DGPG}}=\mathcal{L}_{\mathrm{VI}}.
\)
\end{proposition}

\begin{proof}
By Lemma~\ref{lem:layered}, $\widetilde{G}$ is a layered DAG, so the DAG-DGP construction on it is well defined. We identify each non-root node $w=(l,k)$ of $\widetilde{G}$ with the DGPG GP module at layer $l$ and graph vertex $v_k$. Under this identification,
\[
U_w=U^{l,k},
\qquad
Z_w=Z^{l-1,\PA_G(k)},
\qquad
\{F_p:p\in\PA_{\widetilde G}(w)\}=F^{l-1,\PA_G(k)}.
\]
The root nodes are deterministic and given by $F_{(0,k)}=X^{0,k}$, so that $F^0=X$. For every non-root node $(l,k)$, the concatenation fusion rule makes the DAG-DGP parent input the DGPG input $F^{l-1,\PA_G(k)}$.

The augmented DAG-DGP joint on $\widetilde{G}$ is
\[
p(\mathcal{D},F,U)
=
\prod_{w\in\widetilde{\mathcal{U}}}
p_0(U_w)\,
p_0\big(F_w\mid\{F_p\}_{p\in\PA_{\widetilde G}(w)},U_w\big)\,
p_w(Y_w\mid F_w,O_w),
\]
where $p_w\equiv 1$ whenever $O_w=\emptyset$. We set $O_{(l,k)}=\emptyset$ for $l<L$ and
\[
p_{(L,k)}(Y_{(L,k)}\mid F_{(L,k)})
=
\prod_{n=1}^{N}p(y_n^k\mid F^{L,k}_n)
\]
at the terminal layer. Hence, substituting $w=(l,k)$ in the preceding augmented DAG-DGP joint factorisation gives
\[
p_0(U_{(l,k)})
=
p\big(U^{l,k};Z^{l-1,\PA_G(k)}\big),
\]
and
\[
p_0\big(F_{(l,k)}
\mid
\{F_p\}_{p\in\PA_{\widetilde G}((l,k))},U_{(l,k)}
\big)
=
p\big(
F^{l,k}
\mid
U^{l,k};
F^{l-1,\PA_G(k)},Z^{l-1,\PA_G(k)}
\big).
\]
Therefore
\[
\begin{aligned}
p(\mathcal{D},F,U)
={}&
\prod_{l=1}^{L}\prod_{k=1}^{K}
p\big(U^{l,k};Z^{l-1,\PA_G(k)}\big)
p\big(
F^{l,k}
\mid
U^{l,k};
F^{l-1,\PA_G(k)},Z^{l-1,\PA_G(k)}
\big)
\\
&\times\prod_{k=1}^{K}\prod_{n=1}^{N}
p(y_n^k\mid F^{L,k}_n),
\end{aligned}
\]
which is the DGPG joint in \eqref{eq:dgpg-joint}, up to reordering of factors.

For the variational family, the DAG-VI construction (here on $\widetilde{G}$) gives
\[
q(F,U)
=
q(U)
\prod_{w\in\widetilde{\mathcal{U}}}
p_0\big(F_w\mid\{F_p\}_{p\in\PA_{\widetilde G}(w)},U_w\big).
\]
Choosing the mean-field inducing posterior (DAG-VI)
\[
q(U)
=
\prod_{w\in\widetilde{\mathcal{U}}}q_w(U_w),
\qquad
q_w(U_w)
=
\mathcal{N}(U^{l,k}\mid m^{l,k},S^{l,k}),
\]
recovers the DGPG variational family \eqref{eq:dgpg-q}. The ancestral sampling recursion is also the same, since the DAG-DGP parent input at $(l,k)$ is precisely the concatenated DGPG state $F^{l-1,\PA_G(k)}$.

It remains to match the objectives.
Starting from the DGPG ELBO, we have
\[
\begin{aligned}
\mathcal{L}_{\mathrm{DGPG}}
&=
\sum_{k=1}^{K}\sum_{n=1}^{N}
\mathbb{E}_{q(F^{L,k}_n)}
\big[
\log p(y_n^k\mid F^{L,k}_n)
\big]
-
\sum_{l=1}^{L}\sum_{k=1}^{K}
\mathrm{KL}
\big[
q(U^{l,k})
\,\|\,
p(U^{l,k};Z^{l-1,\PA_G(k)})
\big]
\\[2pt]
&{=}
\sum_{k=1}^{K}\sum_{n=1}^{N}
\mathbb{E}_{q(F^{L,k}_n)}
\big[
\log p(y_n^k\mid F^{L,k}_n)
\big]
-
\sum_{w\in\widetilde{\mathcal{U}}}
\mathrm{KL}
\big(
q_w(U_w)
\,\|\,
p_0(U_w)
\big)
\\[2pt]
&{=}
\sum_{k=1}^{K}
\mathbb{E}_{q_{\mathrm{VI}}(F_{(L,k)})}
\big[
\log p_{(L,k)}(Y_{(L,k)}\mid F_{(L,k)},O_{(L,k)})
\big]
-
\sum_{w\in\widetilde{\mathcal{U}}}
\mathrm{KL}
\big(
q_w(U_w)
\,\|\,
p_0(U_w)
\big)
\\[2pt]
&{=}
\sum_{w:\,O_w\neq\emptyset}
\mathbb{E}_{q_{\mathrm{VI}}(F_w)}
\big[
\log p_w(Y_w\mid F_w,O_w)
\big]
-
\sum_{w\in\widetilde{\mathcal{U}}}
\mathrm{KL}
\big(
q_w(U_w)
\,\|\,
p_0(U_w)
\big)
\\[2pt]
&=
\mathcal{L}_{\mathrm{VI}}.
\end{aligned}
\]
\end{proof}

\begin{remark}[DGPG is a strict subclass of DAG-DGP]
\label{rmk:dgpg-strict-subclass}
Let $\mathfrak{M}_{\mathrm{DGPG}}$ denote the DGPG model class. Let $\mathfrak{M}_{\mathrm{DAG\text{-}DGP}}$ denote the general DAG-DGP model class as presented in this paper. Proposition~\ref{prop:dgpg-dagdgp} shows that every DGPG specification is a DAG-DGP specification, i.e.
$\mathfrak{M}_{\mathrm{DGPG}}\subseteq\mathfrak{M}_{\mathrm{DAG\text{-}DGP}}$.
The inclusion is strict because the embedding fixes choices that are free in the general DAG-DGP formulation. In particular, DGPG fixes the observation pattern to the terminal antichain, whereas DAG-DGPs allow observations at internal nodes, as used in the multi-fidelity and protein-signalling models of the main text.

The simplest separation, however, relates to the graphical structure. By Lemma~\ref{lem:layered}, every depth-unrolling produces a layered DAG. Hence any non-layered DAG is outside the DGPG domain. For example, consider the three-node DAG $H$ with edges $a\to b$, $b\to c$, and the skip edge $a\to c$.
This DAG is not layered. Indeed, recall that in a layered DAG every edge must connect two consecutive antichains. The path $a\to b\to c$ would require $a,b,c$ to lie in three consecutive antichains, whereas the edge $a\to c$ would require $a$ and $c$ to lie in consecutive antichains. These two requirements are incompatible.
Nevertheless, a DAG-DGP is well defined directly on $H$ for any admissible kernels and fusion rule.
\end{remark}

\begin{figure}[t]
\centering
\begin{tikzpicture}[
  pn/.style   ={draw, circle, minimum size=5.2mm, inner sep=0pt, fill=white, line width=0.7pt},
  gnb/.style  ={draw, circle, minimum size=5.2mm, inner sep=0pt, fill=white, line width=0.7pt},
  gns/.style  ={draw, circle, minimum size=1.8mm, inner sep=0pt, fill=white, line width=0.35pt},
  root/.style ={draw, diamond, aspect=1.25, minimum size=8.5mm, inner sep=0.5pt, line width=0.7pt},
  lat/.style  ={draw, circle, minimum size=7.6mm, inner sep=0pt, line width=0.7pt},
  obs/.style  ={draw, rectangle, minimum size=7.6mm, inner sep=0pt, line width=0.7pt},
  broot/.style={draw, diamond, aspect=1.25, minimum size=10.2mm, inner sep=0.5pt, line width=0.7pt},
  blat/.style ={draw, circle, minimum size=10.2mm, inner sep=0pt, line width=0.7pt},
  bobs/.style ={draw, rectangle, minimum size=10.2mm, inner sep=0pt, line width=0.7pt},
  ge/.style   ={-{Stealth[length=1.5mm]}, line width=0.7pt},
  sl/.style   ={-{Stealth[length=1.5mm]}, line width=0.9pt, color=orange!88!black},
  fl/.style   ={-{Stealth[length=2.4mm]}, line width=1.2pt},
  inlab/.style={font=\scriptsize, inner sep=0pt},
  Gsmall/.pic={
    \node[gns] (a) at (0,0.16){};
    \node[gns] (b) at (-0.22,-0.09){};
    \node[gns] (c) at (0.22,-0.09){};
    \draw[-{Stealth[length=0.8mm]}, line width=0.35pt] (a)--(b);
    \draw[{Stealth[length=0.8mm]}-{Stealth[length=0.8mm]}, line width=0.35pt] (b)--(c);
    \draw[-{Stealth[length=0.8mm]}, line width=0.35pt, color=orange!88!black]
      (a) to[out=120,in=60,looseness=14] (a);
  },
]

  \node[gnb,label={[font=\scriptsize]above:$v_1$}] (Ga) at (2.1,0.35){};
  \node[gnb,label={[font=\scriptsize]below:$v_2$}] (Gb) at (0.8,-0.95){};
  \node[gnb,label={[font=\scriptsize]below:$v_3$}] (Gc) at (3.4,-0.95){};

  \draw[ge] (Ga)--(Gb);
  \draw[{Stealth[length=1.5mm]}-{Stealth[length=1.5mm]}, line width=0.7pt] (Gb)--(Gc);
  \draw[sl] (Ga) to[out=120,in=60,looseness=8] (Ga);

  \node[font=\small\bfseries] at (2.1,-2.95){(a)\ $G$};

  \node[broot] (F0) at (5.2,1.20){};
  \node[blat]  (F1) at (5.2,-0.50){};
  \node[bobs]  (F2) at (5.2,-2.20){};

  \pic (s0) at (5.2,1.20)  {Gsmall};
  \pic (s1) at (5.2,-0.50) {Gsmall};
  \pic (s2) at (5.2,-2.20) {Gsmall};

  \draw[fl] (F0)--(F1) node[midway,left,font=\footnotesize]{$f^1$};
  \draw[fl] (F1)--(F2) node[midway,left,font=\footnotesize]{$f^2$};

  \node[font=\small\bfseries] at (5.2,-2.95){(b)\ DGPG};

  \node[root, inlab] (R1) at (8.3,1.0) {$v_1^0$};
  \node[root, inlab] (R2) at (9.7,1.0) {$v_2^0$};
  \node[root, inlab] (R3) at (11.1,1.0) {$v_3^0$};

  \node[lat, inlab]  (L1) at (8.3,-0.5) {$v_1^1$};
  \node[lat, inlab]  (L2) at (9.7,-0.5) {$v_2^1$};
  \node[lat, inlab]  (L3) at (11.1,-0.5) {$v_3^1$};

  \node[obs, inlab]  (O1) at (8.3,-2.0) {$v_1^2$};
  \node[obs, inlab]  (O2) at (9.7,-2.0) {$v_2^2$};
  \node[obs, inlab]  (O3) at (11.1,-2.0) {$v_3^2$};

  \draw[sl] (R1)--(L1);
  \draw[sl] (L1)--(O1);

  \draw[ge] (R1)--(L2);
  \draw[ge] (R3)--(L2);
  \draw[ge] (R2)--(L3);

  \draw[ge] (L1)--(O2);
  \draw[ge] (L3)--(O2);
  \draw[ge] (L2)--(O3);

  \node[font=\small\bfseries] at (9.7,-2.95){(c)\ DAG-DGP on $\widetilde{G}$};

\end{tikzpicture}
\caption{Three views built from one base graph. \textbf{(a)} The base graph $G$ on $K=3$ vertices, with a single self-loop on $v_1$ (highlighted) and a feedback pair $v_2\!\leftrightarrow\!v_3$, hence cyclic. \textbf{(b)} DGPG: a standard chain DGP whose state at every layer is a signal on $G$---shown as an identical small copy of $G$ inside each layer node---and whose layer map $f^\ell$ is wired by $G$. \textbf{(c)} The DAG-DGP defined on the depth-unrolled graph $\widetilde{G}$: each node $v_k^\ell$ denotes the copy of the original node $v_k$ at unrolled layer $\ell$. In both \textbf{(b)} and \textbf{(c)}, diamonds denote root/input nodes or layers, circles denote latent nodes or layers, and squares denote observed nodes or layers.}
\label{fig:dgpg-three-panels}
\end{figure}

\newpage
\section{Experiments}\label{app:experiments}
Experiments were run in \texttt{float64} on NVIDIA TITAN RTX, GeForce RTX 2060 SUPER, and NVIDIA A10 GPUs. To indicate the computational scale of the real-data
experiments, we report an example of wall-clock training times. Training on the published heavy-ion split took 6.3  minutes for
DAG-VI and 8.8 minutes for DAG-SVI on an NVIDIA TITAN RTX. A thorough study of scalability with respect to the number of observations  is conducted in the Sachs experiment, along with training times. 
We now provide more details on our experiments.

\subsection{Data Access Statement}\label{app:data-access}
Code for all experiments is available at \coderepo. Synthetic data can be
regenerated using this code. The Sachs data \citep{sachs2005} are publicly
available on Zenodo
(\url{https://doi.org/10.5281/zenodo.7681811}), and the consensus network
is available through \textit{bnlearn} \citep{bnlearn:2010}.
The heavy-ion collision data were kindly provided by the authors of
\cite{ji2024} and are not redistributed by us. Code to reproduce this
experiment is included, but reproduction requires obtaining access to
the data directly from those authors.
\subsection{Illustrative multi-fidelity DAG-DGP of Figure~\ref{fig:heavy_ion_dag}}\label{app:fig1}

Figure~\ref{fig:heavy_ion_dag} shows a DAG-DGP fitted to a synthetic multi-fidelity
problem with a scalar input \(X\in[0,1]\), two low-fidelity nodes \(L_1, L_2\) and a
high-fidelity node \(H\), with edges \(X\to L_1\), \(X\to L_2\) and
\(\{X, L_1, L_2\}\to H\).

\paragraph{Data-generating process.}
As is standard in multi-fidelity benchmarks \citep{kennedy2000,perdikaris2017}, the data are generated from analytic test functions. The low fidelities \(L_1\) and \(L_2\) are sums of two sinusoids of different
frequencies, which is a classical type of oscillatory low fidelity used in the nonlinear benchmarks of
\citep{perdikaris2017}. For our purpose, we consider the high fidelity to depend on two lower fidelities and on a discrepancy term in the input, that is,
\[
  H(x) = g\bigl(L_1(x), L_2(x)\bigr) + \beta x,
\]
where \(g\) is a weighted sum of three Gaussian peaks \citep{genz1984} on the
\((L_1, L_2)\) plane and \(\beta x\) is a linear discrepancy. This is a graphical
multi-fidelity structure \citep{ji2024}. Graphically, it induces the same DAG as the heavy-ion problem
of Section~\ref{sec:heavy-ion}, in which two low fidelities jointly inform the high fidelity, rather
than forming the chain of fidelities of \citep{kennedy2000,perdikaris2017, cutajar2019}. The transfer
\(g\) is nonlinear in the low fidelities, as in the nonlinear autoregressive scheme of
\citep{perdikaris2017}, while the discrepancy enters additively, as in
\citep{kennedy2000}. Figure~\ref{fig:fig1-fidelities} shows the three functions.

\begin{figure}[t]
  \centering
  \includegraphics[width=\linewidth]{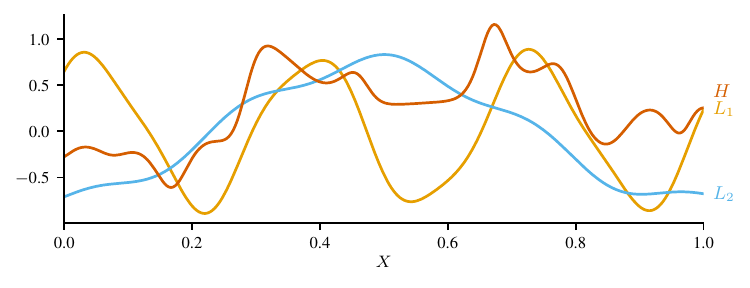}
  \caption{The low fidelities \(L_1\), \(L_2\) and the high fidelity \(H\) of the example in
  Figure~\ref{fig:heavy_ion_dag}, against \(X\).}
  \label{fig:fig1-fidelities}
\end{figure}

\paragraph{Observation design.}
All nodes are observed with Gaussian noise of standard deviation \(0.02\) on 25
equispaced inputs in \([0.02, 0.98]\). \(L_1\) is not observed at
\(x\in\{0.74, 0.78, 0.82\}\) and \(L_2\) at \(x\in\{0.14, 0.18, \dots, 0.30\}\), so that
their posterior uncertainty is visible where their data are missing. \(H\) is observed at
12 of the inputs where both low fidelities are observed, a nested design
\citep{LeGratietGarnier2014}, chosen to be maximally spread in the standardised
observed values of \((X, L_1, L_2)\).

\paragraph{Model.}
We consider a DAG-DGP model for this dataset. The low fidelities have priors \(f_{L_i}\sim\GP(0, \sigma_i^2 k_{\mathrm{SE}})\),
\(i = 1, 2\), and the high fidelity \(f_H\sim\GP(0, K_H)\) with an additive fusion of the
direct input and the low-fidelity outputs,
\[
  K_H\bigl((x, \ell), (x', \ell')\bigr) = \sigma_x^2\,k_{\mathrm{SE}}(x, x')
  + \sigma_\ell^2\,k_{\mathrm{SE}}^{\mathrm{ARD}}(\ell, \ell') + \sigma_b^2,
  \qquad \ell = (\ell_1, \ell_2),
\]
so that \(f_H(x, \ell) = g_X(x) + g_L(\ell) + b\) separates the direct effect of \(X\)
from that of the low fidelities. All likelihoods are Gaussian.

\paragraph{Inference.}
For this small example, we proceed as follows. We first estimate the hyperparameters by type-II maximum likelihood. Given the hyperparameters, the posterior over the latent functions
can be computed by importance sampling. Integrating \(f_H\) out analytically,
\[
  p(L_1, L_2 \mid \mathcal{D}) \propto p(L_1 \mid Y_{L_1})\, p(L_2 \mid Y_{L_2})\,
  \mathcal{N}\bigl(Y_H;\, 0,\, K_H(L) + \sigma_H^2 I\bigr),
\]
which we approximate by self-normalised importance sampling with \(4000\) joint draws of
\((L_1, L_2)\) from their GP posteriors given their own data (effective sample size
\(\approx 3400\)); posterior summaries of \(H\) are the corresponding mixtures of Gaussian
conditionals. Note that this can be done due to the small dimensionality of the dataset, and the illustrative nature of the experiment.

\paragraph{Accuracy.} We compute root-mean-squared errors against the ground truth, obtaining $0.06$, $0.03$, and $0.08$, respectively for \(L_1\), \(L_2\), and \(H\). The estimated direct effect of \(X\) over the input range, \(\hat g_X(1) - \hat g_X(0) = 0.52\), is close to the true $\beta = 0.55$.
Figure~\ref{fig:fig1-uncertainty} shows the marginal posteriors of \(L_1\), \(L_2\) and
\(H\), together with a GP on \(X\) alone fitted to the high-fidelity data. As expected, the DAG-DGP is conservative about uncertainty in regions without high-fidelity observations, in particular in the intervals \(x\in(0.10, 0.34)\) and \(x\in(0.70, 0.86)\), where its band still covers the true function. In these intervals the transfer from the low
fidelities can also be appreciated, as the root-mean-squared error for \(H\) is \(0.10\), against \(0.36\)
for the GP on \(X\), whose band is wider in the first interval and fails to recover the true
function in part of the second.

\begin{figure}[t]
  \centering
  \includegraphics[width=\linewidth]{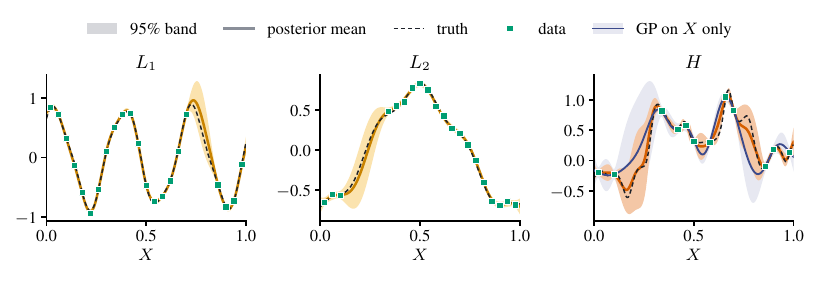}
  \caption{Marginal posteriors of \(L_1\), \(L_2\) and \(H\) against \(X\) for the example in
  Figure~\ref{fig:heavy_ion_dag}: posterior mean, 95\% band, ground truth and observations.
  In the \(H\) panel, the indigo line and band are the posterior of a GP on \(X\) alone
  fitted to the high-fidelity data.}
  \label{fig:fig1-uncertainty}
\end{figure}

\paragraph{Visualization.}
The \(L_1\) and \(L_2\) panels show the posterior mean, the 95\% band, six joint
posterior draws and the observations. The \(H\) panel shows the posterior mean of the
high-fidelity function over its whole input space,
\(\hat H(x, \ell) = \mathbb{E}[f_H(x, \ell) \mid \mathcal{D}]\), on the box spanned by
\(X\) and the range of the low-fidelity posterior means, by direct volume rendering:
colour is from low (blue) to high (orange) values of \(\hat H\), and opacity increases
towards both ends of the scale. The tube is
the trajectory \(x\mapsto(\mathbb{E}[L_1(x)\mid\mathcal{D}],
\mathbb{E}[L_2(x)\mid\mathcal{D}], x)\), along which the high fidelity is evaluated,
coloured by \(\hat H\), and its projection onto the \((L_1, L_2)\) plane is drawn on the
floor; the green cubes are the high-fidelity observations.

\newpage

\subsection{Fusion kernels}\label{app:fusion-kernels}

\paragraph{Experimental design.}
We consider the collider DAG
\(X_1\to A\to C\leftarrow B\leftarrow X_2\), with \(A\), \(B\), and \(C\)
partially observed. We define two complementary data-generating mechanisms
using components of the Ishigami function, a standard nonlinear benchmark in
global sensitivity analysis \citep{ishigami1990}. To generate the child
responses, we independently rescale the latent parent trajectories to
\([-\pi,\pi]\). The context-dependent mechanism is
\[
    C(A,B)=0.1B^4\sin(A),
\]
so that \(B\) modulates the contribution of \(A\), whereas the independent
mechanism is
\[
    C(A,B)=\sin(A)+7\sin^2(B),
\]
where the two parent-specific contributions are additive.

The experiment uses a \(10\times10\) Cartesian root grid. We observe 10 noisy
values from each parent and, respectively, 80 and 36 child values under the
context-dependent and independent mechanisms. We intentionally use a sparser
observation regime for the independent mechanism to examine the effect of a
mismatched product-fusion inductive bias under partial observability.
Gaussian observation noise has standard deviation \(0.025\) at the parents
and \(0.05\) at the child. Since our aim is to isolate the effect of the fusion rule, we keep the observation-noise variances fixed at their data-generating values and the parent kernels fixed (RBF, lengthscale 0.42) as well, and learn only the child-kernel hyperparameters, together with the inducing locations and variational parameters. This ensures that additive and product fusion share identical parent models and differ only in the child kernel.

\paragraph{Models and evaluation.}
We compare additive fusion,
\[
    K_{\mathrm{add}}
    =s_A k_A(A,A')+s_B k_B(B,B'),
\]
and product fusion,
\[
    K_{\mathrm{prod}}
    =s\,k_A(A,A')k_B(B,B'),
\]
using RBF component kernels. The parent GPs use 10 inducing points and the
child GP an \(8\times8\) inducing grid. Models are trained for \(10{,}000\)
full-batch Adam steps using five Monte Carlo samples per ELBO estimate.
We report results for DAG-SVI, but DAG-VI gave the same qualitative conclusions.
CRPS is evaluated on noisy held-out observations and \(R^2\) against the
corresponding latent test values.

\begin{table}[h]
    \centering
    \small
    \setlength{\tabcolsep}{4pt}
    \caption{DAG-SVI performance under additive and product fusion.
    Results are means \(\pm\) standard deviations across three independently
    generated datasets, after averaging two optimization seeds per dataset.}
    \label{tab:fusion-kernels}
    \begin{tabular}{lcccc}
        \toprule
        & \multicolumn{2}{c}{Additive fusion}
        & \multicolumn{2}{c}{Product fusion} \\
        \cmidrule(lr){2-3}\cmidrule(lr){4-5}
        Mechanism
        & CRPS \(\downarrow\) & \(R^2\) \(\uparrow\)
        & CRPS \(\downarrow\) & \(R^2\) \(\uparrow\) \\
        \midrule
        Context-dependent
        & \(2.121\pm0.535\) & \(0.2595\pm0.0709\)
        & \(\mathbf{0.291\pm0.133}\) & \(\mathbf{0.9520\pm0.0637}\) \\
        Independent
        & \(\mathbf{0.101\pm0.032}\) & \(\mathbf{0.9961\pm0.0023}\)
        & \(1.036\pm0.298\) & \(0.6171\pm0.1110\) \\
        \bottomrule
    \end{tabular}
\end{table}
The results in Table~\ref{tab:fusion-kernels} confirm our expectation. In particular, product fusion is strongly preferable when one
parent's contribution depends on the other parent, whereas additive fusion
is preferable when the two parent-specific contributions are additive and act independently.
Additive fusion is simpler and more interpretable but cannot represent
interactions; product fusion with RBF components is more flexible and can potentially represent additive mechanism with enough data, but
underperforms on the additive mechanism under partial observability.

\newpage

\subsection{ELBO scaling}\label{app:branching-tree}

\paragraph{Experimental design.}
Figure~\ref{fig:branching-elbo-scaling} benchmarks the cost of ELBO evaluation
on a synthetic binary branching-tree DAG, used to compare DAG-VI with the dense
and sparse structured backends of DAG-SVI. The DAG has one two-dimensional root
input \(X\) and depth \(d\in\{2,\dots,9\}\). Every non-root node has exactly one
parent and splits into two children, so the graph contains
\(2^{d+1}-2\) latent non-root nodes in total, with the \(2^d\) leaves observed.
For each depth, we generate \(20\) training cases. Root inputs are sampled
uniformly from \([-2,2]^2\). Each latent node is then generated from a
one-parent nonlinear GP draw using random Fourier features with variance \(1\),
lengthscale \(1.2\) at the first latent layer and \(0.9\) thereafter, followed
by a mild Gaussian perturbation and a \(\tanh\) squashing nonlinearity. The
observed leaves are obtained by adding independent Gaussian noise with standard
deviation \(0.05\).

We additionally consider DAGs with fixed depth and increasing width, consisting
of \(W\in\{2,4,8,16,32,64,128,256\}\) parallel three-node chains sharing the
same root, \(X\to C_{k,0}\to C_{k,1}\to C_{k,2}\), with the terminal nodes partially
observed. Conditioning on \(X\) yields \(W\) disconnected three-node
components, allowing us to study ELBO scaling with graph width under a given moderate depth (Figure~\ref{fig:kwide3deep-elbo-scaling}). This is another setting where the sparsity structure can be exploited directly, bringing greater scaling.

\paragraph{Models and timing protocol.}
All methods use the same nodewise GP architecture, with an ARD RBF kernel at
each latent node and \(20\) inducing points per latent node, initialized from
training parent-state inputs. We compare three ELBO evaluators: DAG-VI,
DAG-SVI with dense marginalization, and DAG-SVI with sparse marginalization.
To isolate the marginalization cost only, the dense and sparse structured
models are created from the same initialized parameter state; they differ only
in the backend used to evaluate the ELBO. For each depth, we evaluate the
full-batch ELBO on all \(20\) training cases using a single Monte Carlo sample,
and record wall-clock time on GPU.

\paragraph{Relation to the main implementation claims.}
These benchmarks are intended to validate that the sparse backend can substantially reduce ELBO-evaluation cost relative to
dense structured inference as the inducing dimension grows, under suitable sparsity structures, while DAG-VI, whose variational family is more restrictive, remains the most efficient of our methods.

\begin{figure}[h]
    \centering
    \includegraphics[width=0.82\linewidth]
        {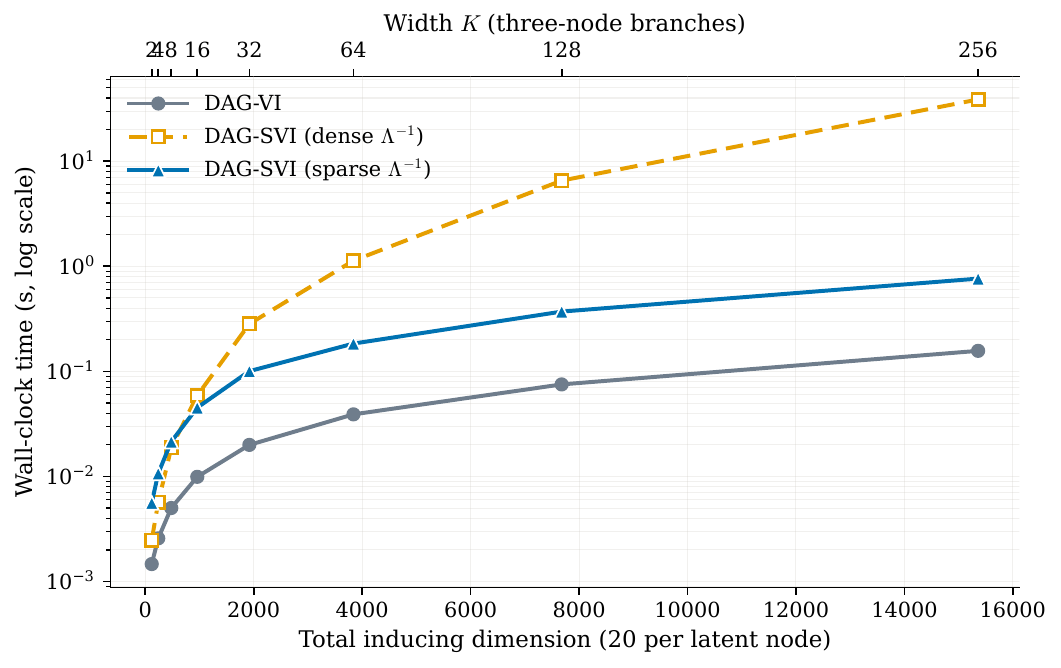}
    \caption{ELBO-evaluation time on \(W\)-wide, three-deep DAGs as the number
    of parallel branches increases.}
    \label{fig:kwide3deep-elbo-scaling}
\end{figure}

\clearpage
\subsection{Theory Validation}\label{app:theory-validation}

\paragraph{Simulation design.}
Figure~\ref{fig:main-theory-validation} reports finite-depth Monte Carlo
checks of the prior and posterior non-collapse statements in
Section~\ref{sec:theory}; Figure~\ref{fig:prior-theory-illustration} sketches the
simulated DAG structures. Panels~a--c are generated under the DAG-DGP prior
of Eq.~\eqref{eq:dagdgp-recursion}, by ancestral sampling on the simulated
DAG in topological order. Rather than sampling whole GP paths, we sample the
finite two-case Gaussian conditionals induced at each node by its realised
parent states, which is the finite-dimensional prior construction studied in
Section~\ref{sec:theory} and formalised in
Appendix~\ref{app:standing-probabilistic-notation}. Throughout, we refer to
the quantities defined there. Panel~d is generated from the filtering and
forward-predictive laws used in Section~\ref{sec:posterior-refresh} and
Appendix~\ref{app:posterior-refresh}. For simplicity, all experiments use
scalar node outputs and squared-exponential kernels, with the hyperparameters
stated below.

\paragraph{Panel a: repeated separating nodes.}
Panel~a validates Theorem~\ref{thm:separating-noncollapse} on a layered
DAG with depth $L=60$ and antichain width $W=16$. We induce separation through
root connections. The two cases have a single deterministic root coordinate,
fixed at $r^{(a)}=0$ and $r^{(b)}=0.3$. For each Monte Carlo realisation, every
non-root node selects one parent uniformly from the preceding antichain.
Non-separating nodes use a squared-exponential kernel on this parent
coordinate, with variance $1$ and lengthscale $1$. Separating nodes use the
same parent component, additively fused with a squared-exponential root
component with the same hyperparameters. This means that the simulated
separating nodes are precisely instances of the additive root-retention
mechanism in Remark~\ref{rem:additive-root-separation}, with the associated
separation constant obtained from
Eq.~\eqref{eq:additive-root-gamma0-general}. The one-step antichain probability
used in the theorem curve is therefore the $p_\varepsilon$ of
Eq.~\eqref{eq:p-epsilon-def}.

For each $s\in\{0,\ldots,10\}$, exactly $s$ separating nodes are chosen in each
antichain. The case $s=0$ is included only as a no-separation baseline. For
each $s$, we simulate $1500$ independent prior realisations and compute the
finite-window version of the frequency appearing in
Theorem~\ref{thm:separating-noncollapse},
\begin{equation}
\hat\rho_\varepsilon
=
\frac{1}{L-L_0}
\sum_{\ell=L_0}^{L-1}
\mathbf{1}\{M_\ell>\varepsilon\},
\qquad
L_0=30,
\qquad
\varepsilon=0.30 .
\end{equation}
The plotted points and error bars are the empirical mean and standard deviation
of $\hat\rho_\varepsilon$ across realisations. The curve is the lower bound of
Theorem~\ref{thm:separating-noncollapse}, evaluated with the above
root-retaining separation constant.

\paragraph{Panels b and c: topology effects.}
Panel~b validates the indegree recursion of
Proposition~\ref{prop:indegree} in the layered radial block of
Assumption~\ref{ass:laplace-radial-degree}. We use depth $50$, width $128$,
kernel variance $1$, parent lengthscale $2.0$, root lengthscale $0.8$, and
root gap $1.0$. The first antichain is initialised by applying the same
two-case Gaussian prior rule to the two fixed root inputs; hence the initial
contrast variance is determined by the root gap and the root lengthscale. All
later node values are propagated using the two-point Gaussian law in
Lemma~\ref{lem:single-radial-two-point}. For a fixed indegree
$k\in\{1,2,4,8\}$, every node receives $k$ parents from the previous antichain.
To keep all nodes statistically symmetric and avoid boundary effects, we assign
these parents by taking $k$ consecutive nodes and wrapping around at the edge
of the layer. This is only a convenient regular layered DAG used to isolate the
effect of indegree. Under product fusion, the product of squared-exponential
parent correlations is the radial squared-exponential kernel on the
concatenated parent state, so the simulation is the squared-exponential special
case of Assumption~\ref{ass:laplace-radial-degree}. We estimate $C_\ell$ by
averaging the simulated squared contrasts over nodes and over $2400$ Monte
Carlo draws. The curves illustrate the monotonicity in $k$ in
Proposition~\ref{prop:indegree} and the corresponding contraction discussion
in Corollary~\ref{cor:indegree-local}: larger indegree broadens the recurrence
relative to the chain case. Note that, for readability, we display depths 1–9; later depths show the same qualitative behaviour.

Panel~c validates the outdegree mechanism in
Proposition~\ref{prop:outdegree-branching}. We simulate the designated
$b$-ary branching subgraph of Appendix~\ref{app:outdegree-branching}, with
$b\in\{1,2,3,5\}$, maximum depth $9$, threshold $t=0.45$, variance $1$,
lengthscale $1.0$, and root gap $1.0$. Children are conditionally independent
given their parent contrast and are sampled using
Lemma~\ref{lem:single-radial-two-point}. For each branching factor and each
$L\in\{3,6,9\}$, the plotted value estimates the finite-depth survival event
corresponding to Proposition~\ref{prop:outdegree-branching}, namely that the
antichain maximum remains above threshold at every depth up to $L$. The
estimates are based on $700$ Monte Carlo draws. This isolates the claim that
larger outdegree creates more parallel opportunities for a threshold-size
contrast to survive.

\paragraph{Panel d: intermediate observations as stochastic skip connections.}
Panel~d validates Theorem~\ref{thm:stochastic-skip-main} using the Gaussian
refresh setting of Eq.~\eqref{eq:main-gaussian-refresh} and
Corollary~\ref{cor:gaussian-refresh-main}. At each observed source, the
two-case source contrast is sampled from the Gaussian filtering law determined
by source contrast variance $1$, observation-noise variance $0.20$, observed
gap $1.0$, and threshold $\varepsilon=0.35$. The corresponding source strength
is the $q_j$ appearing in Theorem~\ref{thm:stochastic-skip-main}. Downstream
routes use squared-exponential route kernels with variance $0.60$ and
lengthscale $0.55$; their retention factors are computed from
Definition~\ref{def:additive-route-retention-coefficients} and
Proposition~\ref{prop:additive-route-retention}.

We compare two geometries. The first is a chain, giving the $s=1$ case of
Theorem~\ref{thm:stochastic-skip-main}. The second is a disjoint-route DAG, as
in the geometry of Fig.~\ref{fig:dag-precision-overview}, with two observed
sources and two pairwise interior-disjoint admissible routes, in the sense of
Definitions~\ref{def:admissible-route} and~\ref{def:disjoint-routes}. In this
case, the theorem bound is computed as one minus the probability that both
routes fail. Thus, each refreshed route contributes its own success
probability, and the two contributions combine through the ``at least one route
succeeds'' term in Eq.~\eqref{eq:main-stochastic-skip}. The empirical curves
estimate the forward-predictive probability
$\Pi_{\ell_0\to\ell_0+h}(M_{\ell_0+h}>\varepsilon)$ for $h=0,\ldots,6$, using
$50000$ posterior-predictive draws.

In the disjoint-route simulation, each target receives its main parent along
the designated route and a weak additive parent from the opposite refreshed
source, with variance $0.02$ and lengthscale $0.55$. This makes the simulated
target a simple multi-parent DAG node, while the theorem bound is evaluated
using only the designated disjoint routes. Since the extra parent enters
additively, it only contributes additional non-negative contrast variance and
is not needed to certify route retention. The plotted theorem curve is
therefore the lower bound for the designated routes. The panel validates both
parts of the stochastic-skip statement: noisy intermediate observations refresh
source contrasts, and multiple disjoint refreshed routes increase the
downstream probability that at least one contrast survives.
\begin{figure}[ht]
    \centering
    \includegraphics[width=\linewidth]{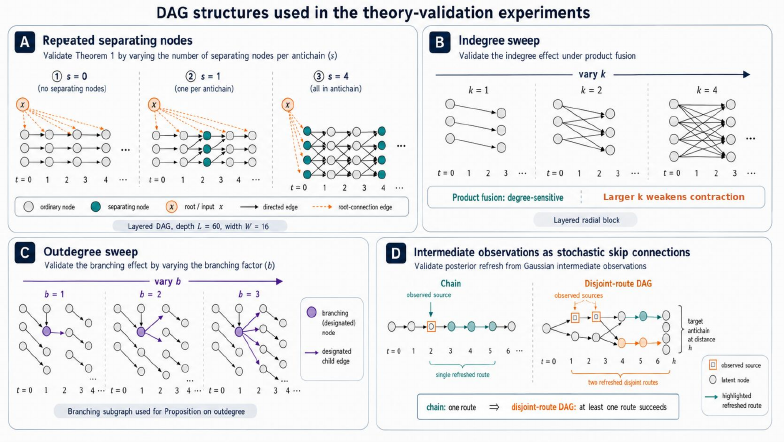}
    \caption{DAG structures used in the theory-validation experiments of Figure~\ref{fig:main-theory-validation}.}
    \label{fig:prior-theory-illustration}
\end{figure}
\clearpage

\subsection{Latent-Collider experiment}\label{cu_ea}

\paragraph{Data-generating process.}
We consider the collider DAG
\[
x_1 \to w_1 \to w_3 \leftarrow w_2 \leftarrow x_2,
\]
where only the child node \(w_3\) is observed. The root inputs \(x_1\) and
\(x_2\) are placed on fixed asymmetric one-dimensional grids with
\(n_1=15\) and \(n_2=11\) points, respectively, yielding
\(15 \times 11 = 165\) observed locations for \(w_3\). The latent parent
functions \(w_1(x_1)\) and \(w_2(x_2)\) are sampled independently from
zero-mean Gaussian processes with RBF kernels of lengthscale \(0.4\) and
variance \(1\). The child latent surface is then sampled from a GP defined
on the parent pair \((w_1,w_2)\), with fusion kernel
\[
K_{w_3}\big((a,b),(a',b')\big)
=
k_{\mathrm{RBF}}(a+b,a'+b'),
\]
using child lengthscale \(0.60\) and variance \(1\). Observations are
generated as
\[
y(x_1,x_2) = w_3(x_1,x_2) + \varepsilon, \qquad
\varepsilon \sim \mathcal N(0,0.015^2).
\]
We deliberately use a small observation-noise level because the experiment is
designed to test explaining-away behaviour. Indeed, as suggested by the
posterior parent covariance of the linear-Gaussian collider in Eq.~\eqref{eq:app-linear-gaussian-collider-covariance},
lower observation noise is expected to induce stronger posterior dependence
between the latent parents. This construction induces a genuine
explaining-away geometry because multiple parent configurations \((w_1,w_2)\)
can produce nearly the same child value through the fused direction \(w_1+w_2\).

\paragraph{Model training and evaluation.}
We compare DAG-VI and DAG-SVI on the same DAG, kernels, inducing locations,
and optimisation setup; only the variational family differs. Both methods use
\(15\) inducing points for each parent node and a \(7 \times 7\) inducing grid
for the child node in the latent \((w_1,w_2)\) plane. Kernel hyperparameters,
the observation-noise variance, and inducing locations are fixed throughout
training, such that the comparison isolates the effect of the variational family at a fixed training budget.
Training is full-batch for \(3000\) gradient steps with learning rate \(0.01\)
and \(4\) Monte Carlo samples per ELBO estimate. Summary statistics are
aggregated over \(3\) random seeds.

\paragraph{Visualisation of explaining-away geometry.}
The left panel of Fig.~\ref{fig:explaining-away-cu} is designed to isolate
\emph{posterior geometry} from posterior scale, already represented in Fig.~\ref{fig:explaining-away-collider} of the main paper. We therefore select a
small set of representative interior observation locations on the child surface,
chosen to span the input domain and to cover a range of local \(w_3\)-level-set
orientations while avoiding boundary cases where the geometry is less
informative. For each selected case \((x_{1,i},x_{2,j})\), we extract posterior
draws of the corresponding parent states \((w_1,w_2)\). To display these local
posterior clouds in the original input space, we map latent perturbations back
to input space using the local sensitivities of the true parent
functions,
\[
\Delta x_1^{(s)}
=
\frac{w_{1,ij}^{(s)}-\bar w_{1,ij}}
{\partial w_1(x_{1,i})/\partial x_1},
\qquad
\Delta x_2^{(s)}
=
\frac{w_{2,ij}^{(s)}-\bar w_{2,ij}}
{\partial w_2(x_{2,j})/\partial x_2},
\]
where \(\bar w_{1,ij}\) and \(\bar w_{2,ij}\) denote the posterior means at the
selected case. Each cloud is then recentered at its corresponding observation
location and rescaled isotropically for display. This removes absolute scale
differences between DAG-VI and DAG-SVI, so the panel emphasises the
\emph{orientation} of posterior uncertainty relative to the true \(w_3\)-level
set passing through that observation. The goal is to visualise whether the posterior
captures the explaining-away direction: DAG-SVI aligns local uncertainty with
the child level set, whereas DAG-VI remains close to isotropic.

\paragraph{Visualisation of compositional uncertainty.}
The right panel of Fig.~\ref{fig:explaining-away-cu} complements the
geometry plot by restoring posterior scale on the natural domains of the
parent functions. For each method, we draw posterior samples of the latent
parent functions \(w_1(x_1)\) and \(w_2(x_2)\), and summarise them by pointwise
credible bands and posterior means. Unlike the explaining-away panel, no
recentering or display rescaling is applied here, so the figure reflects the
actual posterior spread learned by each variational family. This panel is
therefore intended to visualise \emph{compositional uncertainty}: not having observations, the latent functions for \(w_1\) and \(w_2\) should preserve posterior uncertainty, as explained by \cite{ustyuzhaninov2020} for the standard DGPs. DAG-SVI preserves substantially
more posterior variance over \(w_1\) and \(w_2\), whereas DAG-VI collapses toward
an almost deterministic decomposition.

\begin{figure}[t]
    \centering

    \begin{minipage}[t]{0.48\linewidth}
        \centering
        \includegraphics[width=\linewidth]{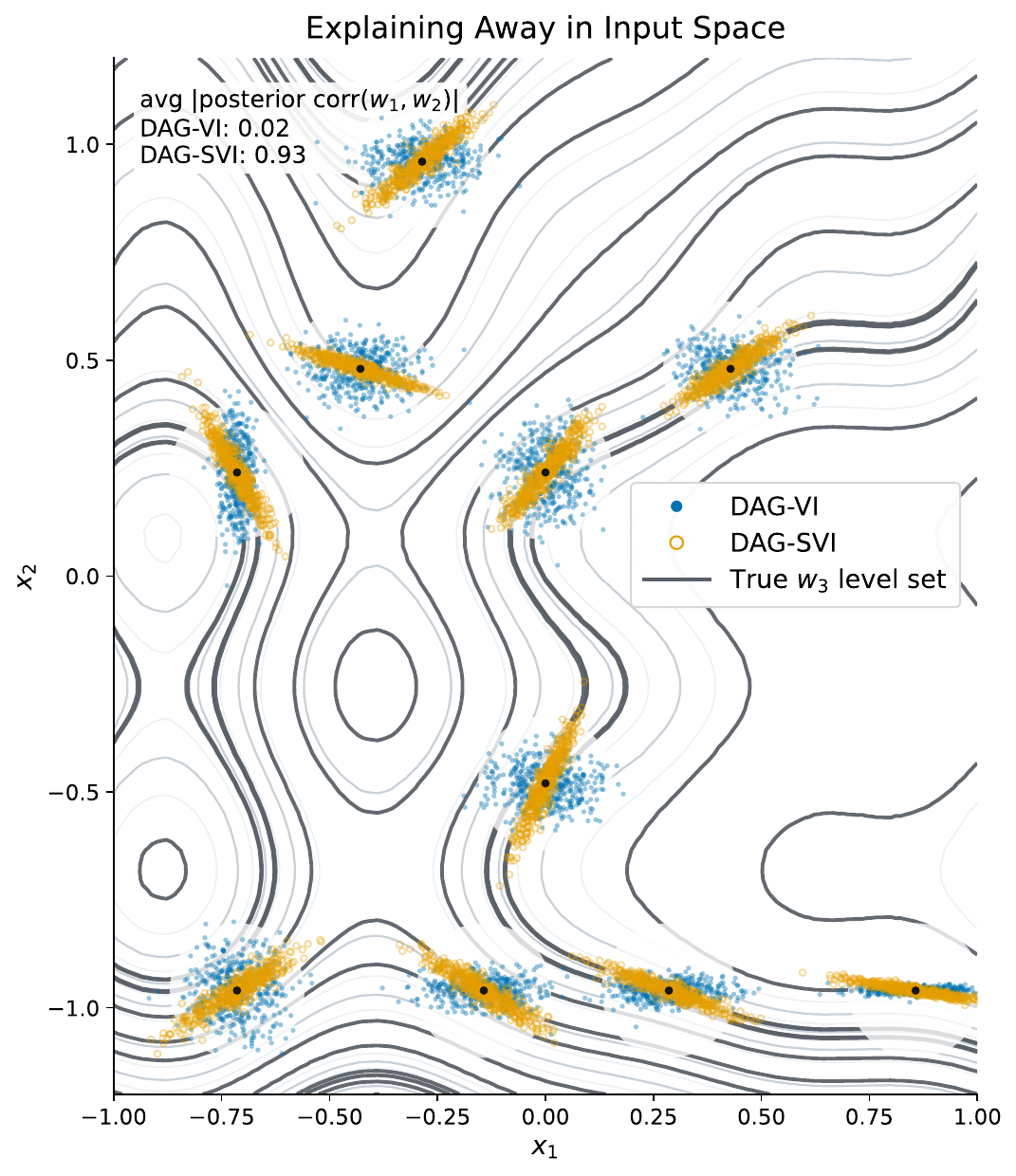}
    \end{minipage}
    \hfill
    \begin{minipage}[t]{0.48\linewidth}
        \centering
        \includegraphics[width=\linewidth]{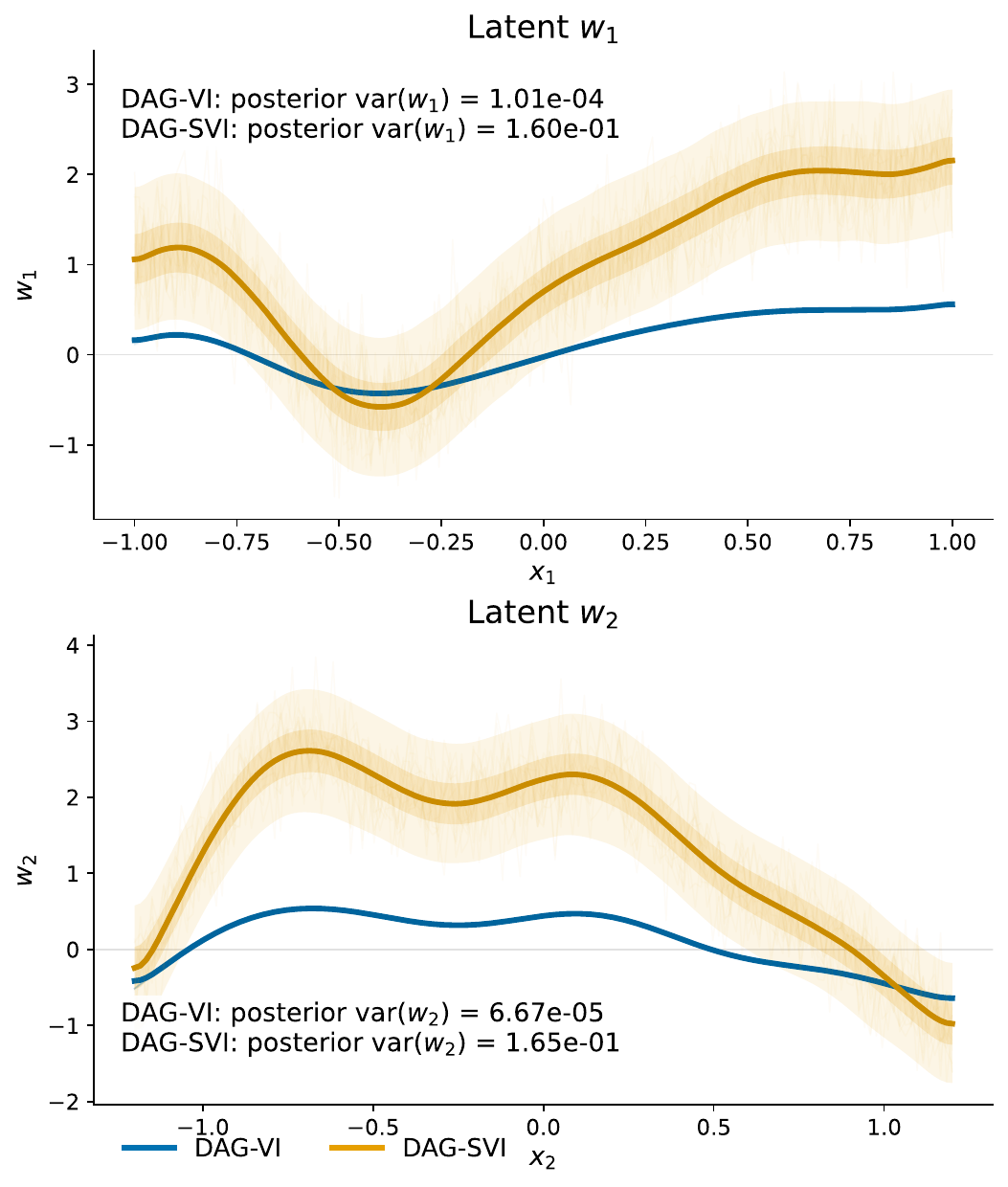}
    \end{minipage}

    \caption{Explaining-away and compositional uncertainty in the synthetic
latent collider (Section~\ref{cu_ea}). \textbf{Left:} projected and rescaled
local posterior samples of $(w_1,w_2)$ in input space. Under DAG-SVI, samples
concentrate along the true $w_3$-level set, recovering the explaining-away
geometry; under DAG-VI, they remain nearly isotropic. Rescaling is used
only to expose the geometric behaviour of the two posteriors on a
comparable scale. The raw variability of DAG-VI is in fact very small, as
shown on the right. \textbf{Right:} DAG-SVI retains uncertainty over the
latent factorisation of $w_3$, whereas DAG-VI collapses to a single
point-like representation.}
    \label{fig:explaining-away-cu}
\end{figure}
\clearpage

\subsection{DAG misspecification}\label{app:misspecification}

\paragraph{Data-generating process.}
We consider \(120\) root inputs in \([-2,2]^4\) and four GP branches
\(x_k\to w_k\), \(k=1,\dots,4\), as candidate parents of the child \(w_5\). The parent
functions are sampled independently from zero-mean GPs with RBF kernels of
lengthscale \(0.7\) and unit variance. In the true DAG, we have a weighted mixture, that is,
\[
    w_5=\sum_{k=1}^{3}a_k f_k(w_k),
    \qquad
    (a_1,a_2,a_3)=(0.75,0.35,0.12),
\]
where the \(f_k\) are independent zero-mean GPs with RBF lengthscale \(0.8\)
and unit variance; equivalently, the child covariance is
\[
    K_{w_5}=\sum_{k=1}^{3}a_k^2
    k_{\mathrm{RBF}}(w_k,w_k').
\]
Thus, \(w_4\) is irrelevant. Each parent is observed in \(66\) cases and the
child in \(90\), with Gaussian noise of standard deviation \(0.05\) and
\(0.08\), respectively. The remaining \(30\) child observations form the test
set, where only the root inputs are provided.

\paragraph{Model training and evaluation.}
All models use RBF kernels at \(w_1,\dots,w_4\) and an additive fusion kernel
\[
    K_{w_5}
    =\sum_{p\in\PA(w_5)}s_p
      k_{\mathrm{RBF}}(w_p,w_p')
\]
at the child, with learned positive component scales \(s_p\). Models differ
only in the specified parent set of \(w_5\). Besides the true set
\(\{w_1,w_2,w_3\}\), we add the irrelevant parent \(w_4\), remove the weakest
parent \(w_3\), remove the two weakest parents \(w_2,w_3\), or remove the
strongest parent \(w_1\). DAG-VI and DAG-SVI use \(10\) inducing points at
each parent and \(24\) at the child. They are trained for \(5{,}000\)
full-batch gradient steps with learning rate \(10^{-2}\), halved for the
likelihood parameters, and four Monte Carlo samples per ELBO estimate.

In order to ensure the same latent process, we fix one draw of the root inputs and latent DAG-DGP functions and generate
five observational replicates by resampling observations,
observation masks, and child train--test split. Figure~\ref{fig:misspecification-rmse} reports the test
RMSE on \(w_5\), normalized by the RMSE of the true DAG within each replicate.

\begin{figure}[h]
    \centering
    \includegraphics[width=0.68\linewidth]
        {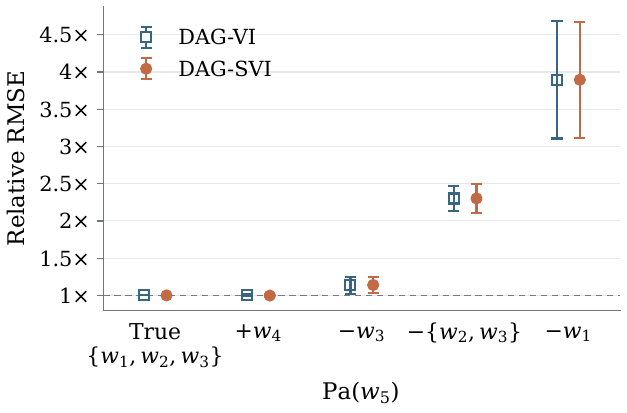}
    \caption{Test RMSE relative to the true DAG as parents of \(w_5\) are
    added or removed. Results are paired within each observational replicate;
    error bars show the sample standard deviation across five replicates.}
    \label{fig:misspecification-rmse}
\end{figure}

\newpage

\subsection{Sachs flow cytometry experiment}\label{app:sachs_exp}

\paragraph{Model specification and training.}
We use the Sachs DAG described in Sec.~\ref{sec:sachs}, with PKC and
Plcg as root inputs and PKA and Raf observed during both training and
testing.
Each non-root GP uses an ARD RBF kernel, equivalent to product fusion
over its input coordinates, and a Gaussian conditional distribution
for the observations.
We include two latent coordinates, one associated with PKA and one
with Raf, with standard Gaussian priors. Each non-root GP uses
32 inducing locations for $N=902$ and 64 for $N=11{,}672$.

For the random-split experiment with $N=902$, both DAG methods share
an initialization obtained from 1,000 node-wise pretraining steps
per node with DAG-VI. This is followed
by 10,000 full-batch joint optimization steps using Adam
\citep{kingma2015adam}, a fixed learning rate of $10^{-3}$, and five
Monte Carlo samples per ELBO estimate. During joint training,
inducing locations and likelihood noise are held fixed for the
first 4,000 steps and optimized during the remaining 6,000.
Kernel hyperparameters and variational parameters are optimized
throughout joint training.

The MO-GP baseline uses a sparse linear model of coregionalisation
with nine latent GPs, ARD RBF kernels, zero mean functions, and
32 inducing locations per latent GP. It is trained using Adam
with learning rate $10^{-2}$ for 19,000 full-batch steps, matching
the number of joint and node-wise updates allocated to the DAG
methods.

\paragraph{Scalability experiment.}
We pool all 14 experimental conditions, giving $N=11{,}672$, and
use five random 80/20 splits stratified by condition. Before fitting,
we subtract means specific to each intervention, estimated from the
training partition, restoring the corresponding output offsets before
scoring. This removes condition-specific mean shifts without
modelling intervention-dependent mechanisms.

We use a mini-batch size of 1,024 for both DAG methods, with
64 inducing locations per node and a budget of 30,000 joint
updates.
Both DAG methods start from a shared mean-field initialization
obtained with 1,000 pretraining steps per node and batches of 512.
The joint learning rate remains $10^{-3}$; inducing locations and
likelihood noise are unfrozen after 1,000 joint updates.
The MO-GP baseline uses 64 inducing locations per latent GP and
30,000 full-batch updates at learning rate $10^{-2}$.
All training times in Table~\ref{table:mf_svi_results} were measured on NVIDIA A10 GPUs (Ampere architecture).

\paragraph{Additional results.}
Table~\ref{table:mf_svi_results} reports metrics over all nine non-root nodes, including PKA and Raf,  because the model targets the joint predictive distribution of the full signalling network, in which these partially observed nodes are propagated by the DAG to their descendants.
Table~\ref{tab:sachs_seven_nodes} restricts the same metrics to the seven nodes that are partially observable at training and test time (Mek, PIP2, PIP3, Erk, Akt, P38 and Jnk). Both DAG methods still outperform MO-GP on every metric in both settings.

\begin{table}[ht]
    \centering
    \caption{Sachs results restricted to the seven nodes unobserved at test time. We report mean $\pm$ std over five splits for both interpolation (top) and scalability (bottom).}
    \label{tab:sachs_seven_nodes}
    \small
    \setlength{\tabcolsep}{6pt}
    \begin{tabular}{@{}clccc@{}}
        \toprule
        $N$ & Model
        & RMSE$\downarrow$
        & CRPS$\downarrow$
        & ES$\downarrow$ \\
        \midrule
        \multirow{3}{*}{902}
        & DAG-VI (full batch)
        & $0.711\pm0.020$
        & $0.380\pm0.010$
        & $1.428\pm0.051$ \\
        & DAG-SVI (full batch)
        & $\mathbf{0.709}\pm0.017$
        & $\mathbf{0.376}\pm0.008$
        & $\mathbf{1.421}\pm0.043$ \\
        & MO-GP (full batch)
        & $0.759\pm0.016$
        & $0.409\pm0.009$
        & $1.558\pm0.036$ \\
        \midrule
        \multirow{3}{*}{$11{,}672$}
        & DAG-VI (1024 mini-batch)
        & $\mathbf{0.783}\pm0.007$
        & $\mathbf{0.413}\pm0.003$
        & $\mathbf{1.123}\pm0.011$ \\
        & DAG-SVI (1024 mini-batch)
        & $0.784\pm0.004$
        & $0.414\pm0.002$
        & $1.127\pm0.006$ \\
        & MO-GP (full batch)
        & $0.836\pm0.005$
        & $0.452\pm0.003$
        & $1.207\pm0.008$ \\
        \bottomrule
    \end{tabular}
\end{table}

\paragraph{Extrapolation experiment.}
We additionally evaluate extrapolation
using a projection-based split adapted from \cite{lindinger2020}.
To construct the split, we standardize the root inputs
$x_i=(\mathrm{PKC}_i,\mathrm{Plcg}_i)$ (on the original scale)
over the full dataset, sample a random unit direction
$w\in\mathbb{R}^2$, and sort observations
by $z_i=\widetilde{x}_i^\top w$. We train on the central 80\%
and test on the lower and upper 10\% tails, yielding 721 training
observations and 181 test observations. While the standardization
is used to define the split, the models continue to use
log-transformed measurements.
PKA and Raf remain observed in both sets.

We retain the same DAG model and joint optimization settings as
in the $N=902$ random-split experiment. All three models are trained
for 10,000 full-batch steps, using learning rates of $10^{-3}$
for DAG-VI/SVI and $10^{-2}$ for MO-GP.
Table~\ref{tab:sachs-extrapolation} reports means and standard
deviations across five random projection splits.
DAG-SVI shows the best overall performance.

\begin{table}[ht]
    \centering
    \small
    \caption{Sachs extrapolation results, mean $\pm$ standard
    deviation over five random projection splits.}
    \label{tab:sachs-extrapolation}
    \begin{tabular}{lccc}
        \toprule
        Model
        & RMSE $\downarrow$
        & CRPS $\downarrow$
        & ES $\downarrow$ \\
        \midrule
        DAG-VI
        & $0.892\pm0.112$
        & $0.409\pm0.036$
        & $2.032\pm0.165$ \\
        DAG-SVI
        & $\mathbf{0.806}\pm0.051$
        & $\mathbf{0.383}\pm0.011$
        & $\mathbf{1.857}\pm0.069$ \\
        MO-GP
        & $0.828\pm0.046$
        & $0.441\pm0.027$
        & $2.094\pm0.173$ \\
        \bottomrule
    \end{tabular}
\end{table}
\newpage
\subsection{Heavy-ion collision}\label{app:heavyion}
\paragraph{Model specification and training.}
We use the elicited multi-fidelity DAG of Figure~\ref{fig:heavy_ion_dag} on a
shared nine-dimensional input space, with two lower-fidelity latent nodes
\(L_1\) and \(L_2\) feeding the high-fidelity node \(H\). In all protocols,
\(L_1\) and \(L_2\) use \(200\) inducing points each, while \(H\) uses all
available high-fidelity training locations in the current split or fold:
\(25\) on the published split, \(20\) in each fold of the repeated 5-fold
protocol, and \(90\) in each fold of the pooled 10-fold protocol. Inducing
locations are fixed after initialization.
We use an RBF kernel on the lower-fidelities and the multi-fidelity fusion rule \citep{perdikaris2017, cutajar2019, ji2024} for the high fidelity, that is,
\[
K_H([x,z],[x',z'])
=
K_{\mathrm{RBF},\rho}(x,x')
\left[
K_{\mathrm{LIN}}(z,z')
+
K_{\mathrm{RBF}}(z,z')
\right]
+
K_{\mathrm{RBF},\delta}(x,x'),
\]
where \(z=(L_1,L_2)\).
For \(H\), we consider two initialization schemes. In the \texttt{predmean}
variant, inducing inputs are initialized at sampled high-fidelity input
locations augmented with the predictive means of pre-trained lower-fidelity
models, following the multifidelity inducing-input construction described by
\citep{cutajar2019} (Section~4.4), which is motivated by the practical
difficulty of freely optimizing such augmented inducing representations. In the
\texttt{random} variant, the high-fidelity inducing inputs are initialized at
sampled high-fidelity input locations, with the lower-fidelity coordinates
sampled independently, with replacement, from the empirical distributions of
the corresponding training-fold observations. All
runs use full-batch Adam in double precision with learning rate \(0.01\),
likelihood learning-rate multiplier \(0.1\), and \(10\) Monte Carlo samples
per ELBO estimate.

DAG-VI first pre-trains each node independently as an SVGP and then optimizes
the joint DAG objective. DAG-SVI follows the same nodewise warmup, then trains
an auxiliary mean-field model for \(2000\) joint steps to initialize the
structured posterior, and finally optimizes the structured objective. On the
published split used in Table~\ref{tab:heavy_ion_results}, both methods use
\(500\) pretraining steps per node; DAG-VI then runs \(12000\) joint steps,
whereas DAG-SVI runs \(10000\) structured steps after the \(2000\)-step
mean-field warm start. Results are reported over four seeds, the training seed being the only source of variation. 

In the high-fidelity-scarce protocol of
Table~\ref{tab:heavy_ion_results_cv}, we perform \(10\) repeated 5-fold
cross-validation runs over the \(25\) high-fidelity observations only, using
\(2000\) pretraining steps for \(L_1\) and \(L_2\), and \(12000\) joint steps for DAG-VI;
\(2000+10000\) steps for DAG-SVI, as above. Fold-wise metrics are averaged within each repetition and
then summarized across the \(10\) split seeds. In the pooled joint 10-fold
protocol of Tables~\ref{tab:joint_cv_energy_h_metrics}
and~\ref{tab:joint_cv_energy_h_metrics_predmean}, we use \(200\), \(200\), and
\(100\) observations at \(L_1\), \(L_2\), and \(H\), respectively, with
held-out observations across all fidelities, \(3000\) pretraining steps for
\(L_1\) and \(L_2\), \(1000\) pretraining steps for \(H\), and the same
\(12000\)-step DAG-VI / \(2000+10000\)-step DAG-SVI budgets as above.
Tables~\ref{tab:joint_cv_energy_h_metrics_predmean}
and~\ref{tab:joint_cv_energy_h_metrics} report the \texttt{predmean} and
\texttt{random} initializations, respectively. We use the \texttt{random}
initialization only in the pooled 10-fold protocol, as a sensitivity check.

\paragraph{Predictive metrics.}
RMSE and CRPS are reported with their standard definitions (see \citep{ji2024}); it is worth noting
that, following \citep{ji2024}, the normalized RMSE (N-RMSE) is defined as
\(1-\mathrm{RMSE}/\mathrm{RMSE}_{\mathrm{base}}\), where
\(\mathrm{RMSE}_{\mathrm{base}}\) is the RMSE of a constant sample-mean
baseline predictor, so that larger values indicate better performance. Additionally, the authors compute the CRPS using a Gaussian approximation. For consistency and comparability, we adopt the same approximation for this experiment only. 
In the published-split experiment, the baseline mean pools the training targets
of all three fidelities; in the cross-validation protocols, it uses the training
fold of the fidelity being scored.

\paragraph{Joint predictive metrics.}
Let \(F\) denote the joint posterior predictive distribution of the concatenated held-out vector
\[
y_{\mathrm{val}}
=
\big[
\{L_1(x_i)\}_{i\in\mathcal{I}_{L_1}},
\{L_2(x_i)\}_{i\in\mathcal{I}_{L_2}},
\{H(x_i)\}_{i\in\mathcal{I}_{H}}
\big].
\]
We evaluate the multivariate weighted energy score
\[
\mathrm{ES}_W(F,y_{\mathrm{val}})
=
\mathbb{E}\,\|X-y_{\mathrm{val}}\|_W
-
\frac{1}{2}\mathbb{E}\,\|X-X'\|_W,
\qquad
X,X' \stackrel{iid}{\sim} F,
\]
with
\[
\|z\|_W
=
\left(
\frac{1}{n_{L_1}\sigma_{L_1,\mathrm{tr}}^2}\|z_{L_1}\|_2^2
+
\frac{1}{n_{L_2}\sigma_{L_2,\mathrm{tr}}^2}\|z_{L_2}\|_2^2
+
\frac{1}{n_H\sigma_{H,\mathrm{tr}}^2}\|z_H\|_2^2
\right)^{1/2},
\]
where \(n_{L_1}, n_{L_2}, n_H\) are the held-out block sizes in the current fold, and
\(\sigma_{L_1,\mathrm{tr}}, \sigma_{L_2,\mathrm{tr}}, \sigma_{H,\mathrm{tr}}\) are the empirical standard deviations computed from the observed training data of that fold only.

\paragraph{Additional baseline comparison for the published split.}
Table~\ref{tab:heavy_ion_additional_baselines} reports additional baselines on
the same split: a chain multi-fidelity DGP \citep{cutajar2019} with $L_1$
only, another with $L_2$ only, and a two-layer DGP that takes only the input
$X$. DAG-VI and DAG-SVI outperform all of them. The better performance of
MF-DGP ($L_2$) compared to MF-DGP ($L_1$) is consistent with the Shapley
analysis in Fig.~\ref{fig:heavy_ion}.

\begin{table}[ht]
    \centering
    \small
    \setlength{\tabcolsep}{6pt}
    \caption{Additional baselines on the published heavy-ion split
    (75 test points). Results are reported as mean\,$\pm$\,std
    over four seeds.}
    \label{tab:heavy_ion_additional_baselines}
    \begin{tabular}{lccc}
        \toprule
        Model
        & RMSE ($\times10^{-2}$) $\downarrow$
        & N-RMSE $\uparrow$
        & CRPS ($\times10^{-2}$) $\downarrow$ \\
        \midrule
        MF-DGP ($L_1$)
        & $3.95\pm0.79$ & $0.619\pm0.076$ & $2.29\pm0.58$ \\
        MF-DGP ($L_2$)
        & $2.63\pm0.09$ & $0.746\pm0.009$ & $1.43\pm0.04$ \\
        2-DGP ($X$)
        & $4.70\pm0.05$ & $0.547\pm0.005$ & $2.45\pm0.03$ \\
        \bottomrule
    \end{tabular}
\end{table}

\paragraph{Comparison of CRPS computations.}
For full comparability with the results in \citep{ji2024},
Tables~\ref{tab:heavy_ion_results} and
\ref{tab:heavy_ion_additional_baselines} report the CRPS for the observed,
noisy measurements of $H$. We therefore include the observational noise
in the Gaussian approximation used to compute the CRPS.
For completeness, we also report the CRPS for the latent $H$ (without
observational noise), computed both through the Gaussian approximation
and through Monte Carlo samples, without assuming Gaussianity, as in
all our other experiments. The two approaches yield very similar results,
as shown in Table~\ref{tab:heavy_ion_crps_comparison}.

\begin{table}[ht]
    \centering
    \small
    \setlength{\tabcolsep}{8pt}
    \caption{CRPS for the latent $H$ ($\times10^{-2}$, lower is better),
    computed using the Gaussian formula and Monte Carlo samples.
    Results are reported as mean\,$\pm$\,std over four seeds.
    For MO-GP, the two computations are equivalent, yielding
    $1.53 \pm 0.13$.}
    \label{tab:heavy_ion_crps_comparison}
    \begin{tabular}{lcc}
        \toprule
        Model & Gaussian approximation & Monte Carlo \\
        \midrule
        MF-DGP ($L_1$)
        & $2.32 \pm 0.53$ & $2.36 \pm 0.59$ \\
        MF-DGP ($L_2$)
        & $1.50 \pm 0.05$ & $1.49 \pm 0.05$ \\
        2-DGP ($X$)
        & $2.49 \pm 0.03$ & $2.49 \pm 0.04$ \\
        2-DGP ($X,L_1,L_2$)
        & $1.30 \pm 0.03$ & $1.30 \pm 0.03$ \\
        DAG-VI
        & $1.18 \pm 0.03$ & $1.18 \pm 0.03$ \\
        DAG-SVI
        & $\mathbf{1.16} \pm 0.02$
        & $\mathbf{1.16} \pm 0.01$ \\
        \bottomrule
    \end{tabular}
\end{table}

\paragraph{Results on the other two experiments.}
We also report results from two additional evaluation settings:
cross-validation on the high-fidelity observations only, and joint
cross-validation across all fidelity levels.

\medskip
\noindent\textbf{Cross-validation on the high-fidelity observations.}
Table~\ref{tab:heavy_ion_results_cv} reports the results obtained
using five validation folds on the high-fidelity observations,
averaged over 10 repetitions.

\begin{table}[ht]
    \centering
    \small
    \setlength{\tabcolsep}{6pt}
    \renewcommand{\arraystretch}{1.05}
    \caption{Predictive performance on the heavy-ion collision emulation
    task using 200 observations at each lower fidelity and 25
    high-fidelity observations, with validation folds formed on the
    high-fidelity set only. Metrics are averaged across the five folds
    within each repetition and reported as mean\,$\pm$\,std across
    10 repetitions. RMSE and CRPS are reported in units of
    $\times10^{-2}$; lower is better for both, while higher is better
    for N-RMSE.}
    \label{tab:heavy_ion_results_cv}
    \begin{tabularx}{\columnwidth}{@{}>{\raggedright\arraybackslash}Xccc@{}}
        \toprule
        Model
        & RMSE ($\times10^{-2}$) $\downarrow$
        & N-RMSE $\uparrow$
        & CRPS ($\times10^{-2}$) $\downarrow$ \\
        \midrule
        d-GMGP
        & $2.39 \pm 0.36$ & $0.761 \pm 0.036$ & $1.62 \pm 0.31$ \\
        MO-GP
        & $2.90 \pm 0.21$ & $0.710 \pm 0.023$ & $1.89 \pm 0.18$ \\
        DAG-VI
        & $1.90 \pm 0.18$ & $0.810 \pm 0.017$ & $1.15 \pm 0.11$ \\
        DAG-SVI
        & $\mathbf{1.74} \pm 0.18$
        & $\mathbf{0.826} \pm 0.016$
        & $\mathbf{1.07} \pm 0.14$ \\
        \bottomrule
    \end{tabularx}
\end{table}

\medskip
\noindent\textbf{Joint cross-validation across fidelity levels.}
Tables~\ref{tab:joint_cv_energy_h_metrics_predmean} and
\ref{tab:joint_cv_energy_h_metrics} report the results of joint
10-fold cross-validation, using respectively the \texttt{predmean}
and \texttt{random} initialization strategies for the high-fidelity
inducing inputs in DAG-VI and DAG-SVI.

\begin{table}[H]
    \centering
    \small
    \setlength{\tabcolsep}{6pt}
    \renewcommand{\arraystretch}{1.05}
    \caption{Joint 10-fold cross-validation results on the multi-fidelity
    heavy-ion task using the \texttt{predmean} initialization for the
    high-fidelity inducing inputs in both DAG-VI and DAG-SVI.
    Results are reported as mean\,$\pm$\,std across folds.
    Lower is better for weighted energy score, RMSE, and CRPS;
    higher is better for N-RMSE.}
    \label{tab:joint_cv_energy_h_metrics_predmean}
    \begin{tabular}{lcccc}
        \toprule
        Model
        & Weighted ES $\downarrow$
        & H RMSE ($\times 10^{-2}$) $\downarrow$
        & H CRPS ($\times 10^{-2}$) $\downarrow$
        & H N-RMSE $\uparrow$ \\
        \midrule
        MO-GP
        & $0.2169 \pm 0.0459$
        & $1.71 \pm 0.54$
        & $0.99 \pm 0.34$
        & $0.819 \pm 0.058$ \\
        DAG-VI
        & $0.1610 \pm 0.0358$
        & $1.10 \pm 0.54$
        & $0.59 \pm 0.25$
        & $0.888 \pm 0.048$ \\
        DAG-SVI
        & $\mathbf{0.1590} \pm 0.0347$
        & $\mathbf{1.04} \pm 0.52$
        & $\mathbf{0.56} \pm 0.25$
        & $\mathbf{0.894} \pm 0.048$ \\
        \bottomrule
    \end{tabular}
\end{table}

\begin{table}[H]
    \centering
    \small
    \setlength{\tabcolsep}{6pt}
    \renewcommand{\arraystretch}{1.05}
    \caption{Joint 10-fold cross-validation results on the multi-fidelity
    heavy-ion task using the \texttt{random} initialization for the
    high-fidelity inducing inputs in both DAG-VI and DAG-SVI.
    Results are reported as mean\,$\pm$\,std across folds.
    Lower is better for weighted energy score, RMSE, and CRPS;
    higher is better for N-RMSE.}
    \label{tab:joint_cv_energy_h_metrics}
    \begin{tabular}{lcccc}
        \toprule
        Model
        & Weighted ES $\downarrow$
        & H RMSE ($\times 10^{-2}$) $\downarrow$
        & H CRPS ($\times 10^{-2}$) $\downarrow$
        & H N-RMSE $\uparrow$ \\
        \midrule
        MO-GP
        & $0.2169 \pm 0.0459$
        & $1.71 \pm 0.54$
        & $0.99 \pm 0.34$
        & $0.819 \pm 0.058$ \\
        DAG-VI
        & $0.1623 \pm 0.0337$
        & $1.12 \pm 0.52$
        & $0.61 \pm 0.25$
        & $0.884 \pm 0.049$ \\
        DAG-SVI
        & $\mathbf{0.1571} \pm 0.0305$
        & $\mathbf{1.04} \pm 0.47$
        & $\mathbf{0.57} \pm 0.23$
        & $\mathbf{0.892} \pm 0.048$ \\
        \bottomrule
    \end{tabular}
\end{table}

\paragraph{Effect of joint training.}
Figure~\ref{fig:pretraining-effect} shows a simple counterfactual experiment
for DAG-VI under the pooled joint 10-fold (init \texttt{predmean}) protocol.
After the same \(H\)-pretraining phase, we continue training \(H\)
while keeping its parent models frozen. Compared with this
continuation, joint training yields lower validation RMSE and CRPS. This supports the benefit of joint learning.

\begin{figure}[ht]
    \centering
    \includegraphics[width=\linewidth]{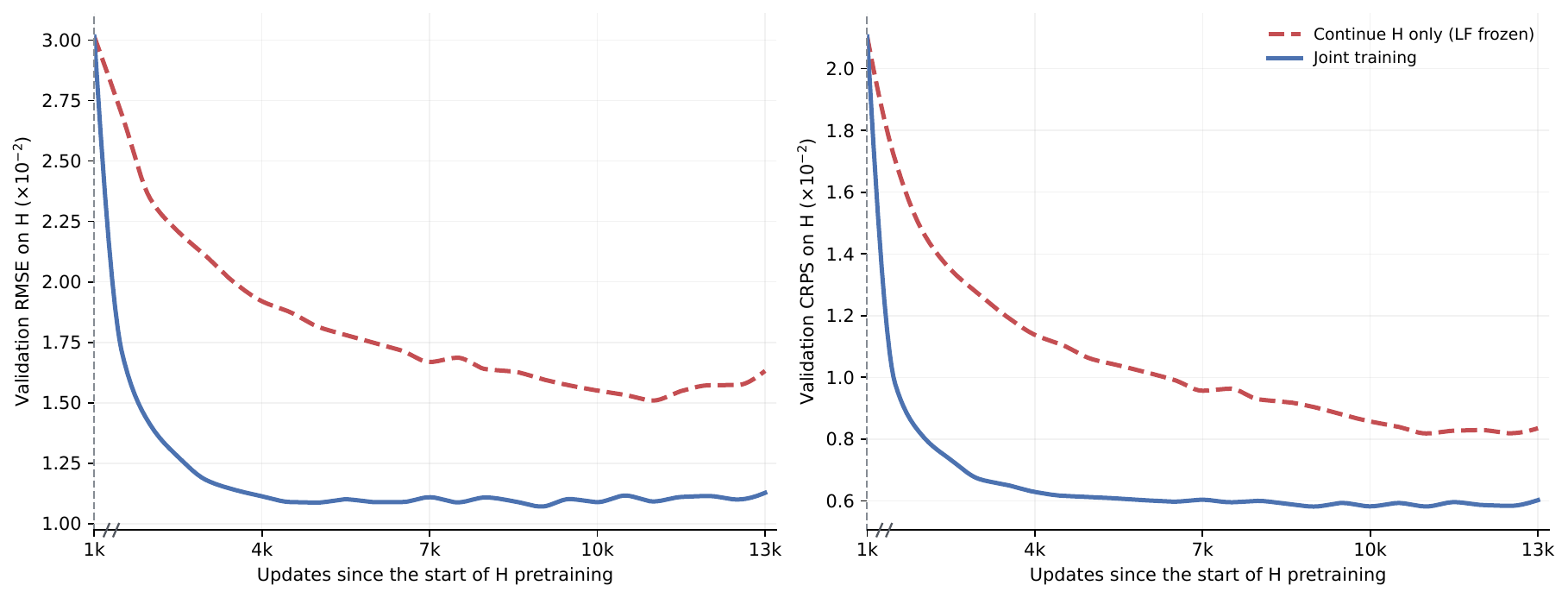}
    \caption{DAG-VI validation RMSE and CRPS, averaged over ten folds
    of the pooled \texttt{predmean} protocol. After 1,000 shared
    \(H\)-pretraining updates, we compare continued \(H\)-only
    training with fixed lower-fidelity models (red dashed)
    against joint training (blue), each for 12,000 additional updates.}
    \label{fig:pretraining-effect}
\end{figure}

\paragraph{Shapley attributions.}
Fig.~\ref{fig:heavy_ion} shows pointwise Shapley attributions
for $L_1$ and $L_2$ at the 75 test points, computed for both
the predictive mean and posterior variance. Absolute contributions
are normalised per point as $|\phi_i|/\sum_j|\phi_j|$ and then
averaged over four DAG-SVI fits. Boxplots summarise their
distribution across test points. For the predictive mean, coalition values are
expectations with the absent fidelity marginalised over the test points; for
the posterior variance, they are variances of conditional expectations across
posterior samples. Both are estimated using $k$-nearest-neighbour regression.

\clearpage

\end{document}